\documentclass[10pt,letterpaper]{article}

\usepackage[utf8]{inputenc}
\usepackage[T1]{fontenc}
\usepackage[english]{babel}
\usepackage[hyphens]{url}
\usepackage{authblk}
\usepackage{booktabs}
\usepackage{hyperref}
\usepackage{amsfonts}
\usepackage{nicefrac}
\usepackage{microtype}
\usepackage{amsmath, amsthm, amssymb}
\usepackage{fancybox}
\usepackage{graphicx}
\usepackage{float}
\usepackage{algorithm}
\usepackage{algpseudocode}
\usepackage{color}
\usepackage{tikz}
\usepackage{array}
\usepackage{subcaption}
\usepackage[textsize=scriptsize]{todonotes}
\usepackage{enumitem}
\usepackage{wrapfig}
\usepackage{enumitem}
\usepackage{etoolbox}
\usepackage{diagbox}
\usepackage{comment}
\usepackage{makecell}
\usepackage{multirow}
\usepackage{bbm}
\usepackage[in]{fullpage}
\usepackage{multicol}
\usepackage{mathpazo}
\usepackage[numbers, compress]{natbib}

\usepackage[framemethod=TikZ]{mdframed}
\mdfsetup{skipabove=5pt,
innertopmargin=-5pt}

\usepackage{adjustbox}

\usepackage{thmtools}
\usepackage{thm-restate}
\usepackage{caption}

\usepackage{algorithm}

\newcommand{\R}{\mathbb{R}}

\newcommand{\Prob}{\mathbb{P}}

\renewcommand\bar\overline

\DeclareMathOperator*{\st}{subject\,\, to \quad}

\newcommand{\Etrain}{\mathcal{E}_{\text{train}}}
\newcommand{\Eall}{\mathcal{E}_{\text{all}}}

\newcommand{\iid}{i.i.d.\ }

\newtheoremstyle{slplain}% name
  {.4\baselineskip\@plus.1\baselineskip\@minus.1\baselineskip}% Space above
  {.3\baselineskip\@plus.1\baselineskip\@minus.1\baselineskip}% Space below
  {}% Body font
  {}%Indent amount (empty = no indent, \parindent = para indent)
  {\bfseries}%  Thm head font
  {.}%       Punctuation after thm head
  { }%      Space after thm head: " " = normal interword space;
  {}%       Thm head spec
\theoremstyle{slplain} % italics

\newtheorem*{definition*}{Definition}
\newtheorem*{theorem*}{Theorem}
\newtheorem{theorem}{Theorem}[section]
\newtheorem{lemma}[theorem]{Lemma}
\newtheorem{proposition}[theorem]{Proposition}

\newtheorem{definition}[theorem]{Definition}

\newtheorem{problem}[theorem]{Problem}
\newtheorem{assumption}[theorem]{Assumption}

\newtheorem{remark}[theorem]{Remark}

\newcommand{\vcdim}{d_{\text{VC}}}

\newlist{todolist}{itemize}{2}
\setlist[todolist]{label=$\square$}
\usepackage{pifont}
\newcommand{\imgintable}[1]{
    \begin{minipage}{2.1cm}
        \centering 
        \vspace{5pt}
        \includegraphics[width=2cm]{#1}
        \vspace{5pt}
    \end{minipage}
}

\newcommand{\augimage}[3]{
\begin{figure}[t]
    \centering
    \begin{subfigure}{0.48\textwidth}
        \includegraphics[width=\textwidth]{#1/img.png}
        \caption{Training images.}
    \end{subfigure} \hfill
    \begin{subfigure}{0.48\textwidth}
        \includegraphics[width=\textwidth]{#1/aug.png}
        \caption{Corresponding images after augmentations.}
    \end{subfigure}
    \caption{#2}
    \label{fig:#3}
\end{figure}}

\newcommand{\multiimage}[3]{
\begin{figure}
    \centering
    \includegraphics[width=\textwidth]{#1/img-1.png}
    \includegraphics[width=\textwidth]{#1/img-2.png}
    \includegraphics[width=\textwidth]{#1/img-3.png}
    \caption{#2}
    \label{fig:#3}
\end{figure}
}

\newcommand{\indep}{\perp \!\!\! \perp}

\DeclareMathOperator{\E}{\mathbb{E}}

\newcommand{\norm}[1]{\ensuremath{\left\| #1 \right\|}}

\DeclareMathOperator*{\argmin}{argmin}
\DeclareMathOperator*{\argmax}{argmax}
\DeclareMathOperator*{\minimize}{minimize}
\DeclareMathOperator*{\maximize}{maximize}

\newcommand{\calA}{\ensuremath{\mathcal{A}}}
\newcommand{\calB}{\ensuremath{\mathcal{B}}}

\newcommand{\calD}{\ensuremath{\mathcal{D}}}

\newcommand{\calF}{\ensuremath{\mathcal{F}}}

\newcommand{\calH}{\ensuremath{\mathcal{H}}}

\newcommand{\calL}{\ensuremath{\mathcal{L}}}

\newcommand{\calP}{\ensuremath{\mathcal{P}}}

\newcommand{\calX}{\ensuremath{\mathcal{X}}}
\newcommand{\calY}{\ensuremath{\mathcal{Y}}}
\newcommand{\calZ}{\ensuremath{\mathcal{Z}}}

\def\nd/{\textsuperscript{nd}}
\def\rd/{\textsuperscript{rd}}
\def\th/{\textsuperscript{th}}

\makeatletter
\def\nnil{\nil}
\newcounter{prob}

\newcounter{dual}

\makeatother

\newenvironment{prob*}{%
	\csname equation*\endcsname%
	\aligned%
}{%
	\endaligned%
	\csname endequation*\endcsname%
}

\title{Model-Based Domain Generalization}
\author{Alexander Robey}
\author{George J. Pappas}
\author{Hamed Hassani}

\affil{Department of Electrical and Systems Engineering \\ University of Pennsylvania}
\date{}

\begin{document}

\maketitle

\begin{abstract}
    Despite remarkable success in a variety of applications, it is well-known that deep learning can fail catastrophically when presented with out-of-distribution data.  Toward addressing this challenge, we consider the \emph{domain generalization} problem, wherein predictors are trained using data drawn from a family of related training domains and then evaluated on a distinct and unseen test domain.  We show that under a natural model of data generation and a concomitant invariance condition, the domain generalization problem is equivalent to an infinite-dimensional constrained statistical learning problem; this problem forms the basis of our approach, which we call Model-Based Domain Generalization.  Due to the inherent challenges in solving constrained optimization problems in deep learning, we exploit nonconvex duality theory to develop unconstrained relaxations of this statistical problem with tight bounds on the duality gap.  Based on this theoretical motivation, we propose a novel domain generalization algorithm with convergence guarantees.   In our experiments, we report improvements of up to 30 percentage points over state-of-the-art domain generalization baselines on several benchmarks including ColoredMNIST, Camelyon17-WILDS, FMoW-WILDS, and PACS.
\end{abstract}
\section{Introduction}

Despite well-documented success in numerous applications \cite{lecun2015deep,esteves2017polar,esteves2018learning,jaderberg2015spatial}, the complex prediction rules learned by modern machine learning methods can fail catastrophically when presented with out-of-distribution (OOD) data \cite{hendrycks2019benchmarking,djolonga2020robustness,taori2020measuring,hendrycks2020many,torralba2011unbiased}.  Indeed, rapidly growing bodies of work conclusively show that state-of-the-art methods are vulnerable to distributional shifts arising from spurious correlations \cite{arjovsky2019invariant,ahuja2020invariant,lu2021nonlinear}, adversarial attacks \cite{biggio2013evasion,goodfellow2014explaining,madry2017towards,wong2017provable,dobriban2020provable}, sub-populations \cite{santurkar2020breeds,sohoni2020no,koh2020wilds,xiao2020noise}, and naturally-occurring variation \cite{robey2020model,wong2020learning,gowal2020achieving,laidlaw2020perceptual}.  This failure mode is particularly pernicious in \emph{safety-critical applications}, wherein the shifts that arise in fields such as medical imaging \cite{esteva2019guide,yao2019strong,li2020domain,bashyam2020medical}, autonomous driving \cite{zhang2020learning,yang2018real,zhang2017curriculum}, and robotics \cite{julian2020never,sonar2020invariant,vinitsky2020robust} are known to lead to unsafe behavior.  And while some progress has been made toward addressing these vulnerabilities, the inability of modern machine learning methods to generalize to OOD data is one of the most significant barriers to deployment in safety-critical applications~\cite{ribeiro2016should,biggio2018wild}.

In the last decade, the \emph{domain generalization} community has emerged in an effort to improve the OOD performance of machine learning methods \cite{blanchard2011generalizing,muandet2013domain,blanchard2017domain,huang2020self}.  In this field, predictors are trained on data drawn from a family of related training domains and then evaluated on a distinct and unseen test domain.   Although a variety of approaches have been proposed for this setting \cite{sun2016deep,li2018learning}, it was recently shown that that no existing domain generalization algorithm can significantly outperform empirical risk minimization (ERM) \cite{vapnik1999overview} over the training domains when ERM is properly tuned and equipped with state-of-the-art architectures \cite{he2016deep,huang2017densely} and data augmentation techniques \cite{gulrajani2020search}.  Therefore, due to the prevalence of OOD data in safety critical applications, it is of the utmost importance that new algorithms be proposed which can improve the OOD performance of machine learning methods.

In this paper, we introduce a new framework for domain generalization which we call \emph{Model-Based Domain Generalization} (MBDG).  The key idea in our framework is to first learn transformations that map data between domains and then to subsequently enforce invariance to these transformations.  Under a general model of covariate shift and a novel notion of invariance to learned transformations, we use this framework to rigorously re-formulate the domain generalization problem as a semi-infinite constrained optimization problem.  We then use this re-formulation to prove that a tight approximation of the domain generalization problem can be obtained by solving the empirical, parameterized dual for this semi-infinite problem.  Finally, motivated by these theoretical insights, we propose a new algorithm for domain generalization; extensive experimental evidence shows that our algorithm advances the state-of-the-art on a range of benchmarks by up to thirty percentage points.  

\paragraph{Contributions.}  Our contributions can be summarized as follows:

\begin{itemize}[nolistsep,leftmargin=3em]
    \item We propose a new framework for domain generalization in which invariance is enforced to underlying transformations of data which capture inter-domain variation.
    \item Under a general model of covariate shift, we rigorously prove the equivalence of the domain generalization problem to a novel semi-infinite constrained statistical learning problem.
    \item We derive \emph{data-dependent} duality gap bounds for the empirical parameterized dual of this semi-infinite problem, proving that tight approximations of the domain generalization problem can be obtained by solving this dual problem under the covariate shift assumption.
    \item We introduce a primal-dual style algorithm for domain generalization in which invariance is enforced over unsupervised generative models trained on data from the training domains.  
    \item We empirically show that our algorithm significantly outperforms state-of-the-art baselines on several standard benchmarks, including \texttt{ColoredMNIST}, \texttt{Camelyon17-WILDS}, and \texttt{PACS}.
\end{itemize}

\section{Related work}

\paragraph{Domain generalization.}  The rapid acceleration of domain generalization research has led to an abundance of principled algorithms, many of which distill knowledge from an array of disparate fields toward resolving OOD failure modes \cite{zhou2021domain,wang2021generalizing,shi2021gradient,bellot2020accounting}.  Among such works, one prominent thrust has been to learn predictors which have internal feature representations that are consistent across domains \cite{ganin2016domain,albuquerque2019adversarial,li2018domain,motiian2017unified,ghifary2016scatter,hu2020domain,ilse2020diva,akuzawa2019adversarial,chattopadhyay2020learning,piratla2020efficient,shankar2018generalizing,li2018deep}.  This approach is also popular in the field of unsupervised domain adaptation \cite{ben2007analysis,daume2009frustratingly,pan2010domain,tzeng2017adversarial,fu2020learning}, wherein it is assumed that unlabeled data from the test domain is available during training~\cite{patel2015visual,csurka2017domain,wang2018deep}.  Also related are works that seek to learn a kernel-based embedding of each domain in an underlying feature space \cite{dubey2021adaptive,deshmukh2019generalization}, and those that employ Model-Agnostic Meta Learning \cite{finn2017model} to adapt to unseen domains \cite{li2018learning,balaji2018metareg,dou2019domain,li2019episodic,shu2021open,li2019feature,wang2020heterogeneous,qiao2020learning,zhang2020adaptive}.  Recently, another prominent direction has been to design weight-sharing \cite{mancini2018robust,mancini2018best,li2017deeper,ding2017deep} and instance re-weighting schemes~\cite{sagawa2019distributionally,hu2018does,johansson2018learning}.  Unlike any of these approaches, we explicitly enforce hard invariance-based constraints on the underlying statistical domain generalization problem.

% ^this last sentence is too vague.

\paragraph{Data augmentation.}  Another approach toward improving OOD performance is to modify the available training data.  Among such methods, perhaps the most common is to leverage various forms of data augmentation \cite{krizhevsky2012imagenet,hendrycks2019augmix,chen2019invariance,zhang2019unseen,volpi2018generalizing,xu2020adversarial,yan2020improve,zhang2017mixup}.  Recently, several approaches have been proposed which use style-transfer techniques and image-to-image translation networks~\cite{goodfellow2014generative,karras2019style,brock2018large,gatys2016image,zhu2017unpaired,huang2018multimodal,almahairi2018augmented,russo2018source} to augment the training domains with artificially-generated data \cite{zhou2020deep,carlucci2019domain,vandenhende2019three,arruda2019cross,rahman2019multi,yue2019domain,murez2018image,li2021semantic}.  Alternatively, rather than generating new data, \cite{wang2019learning,nam2019reducing,asadi2019towards} all seek to remove textural features in the data to encourage domain invariance.   Unlike the majority of these works, we do not perform data augmentation directly on the training objective; rather, we derive a principled primal-dual style algorithm which enforces invariance-based constraints on data generated by unsupervised generative models.
\section{Domain generalization}

\begin{figure}
    \centering
    \begin{subfigure}[b]{0.31\textwidth}
        \includegraphics[width=\columnwidth]{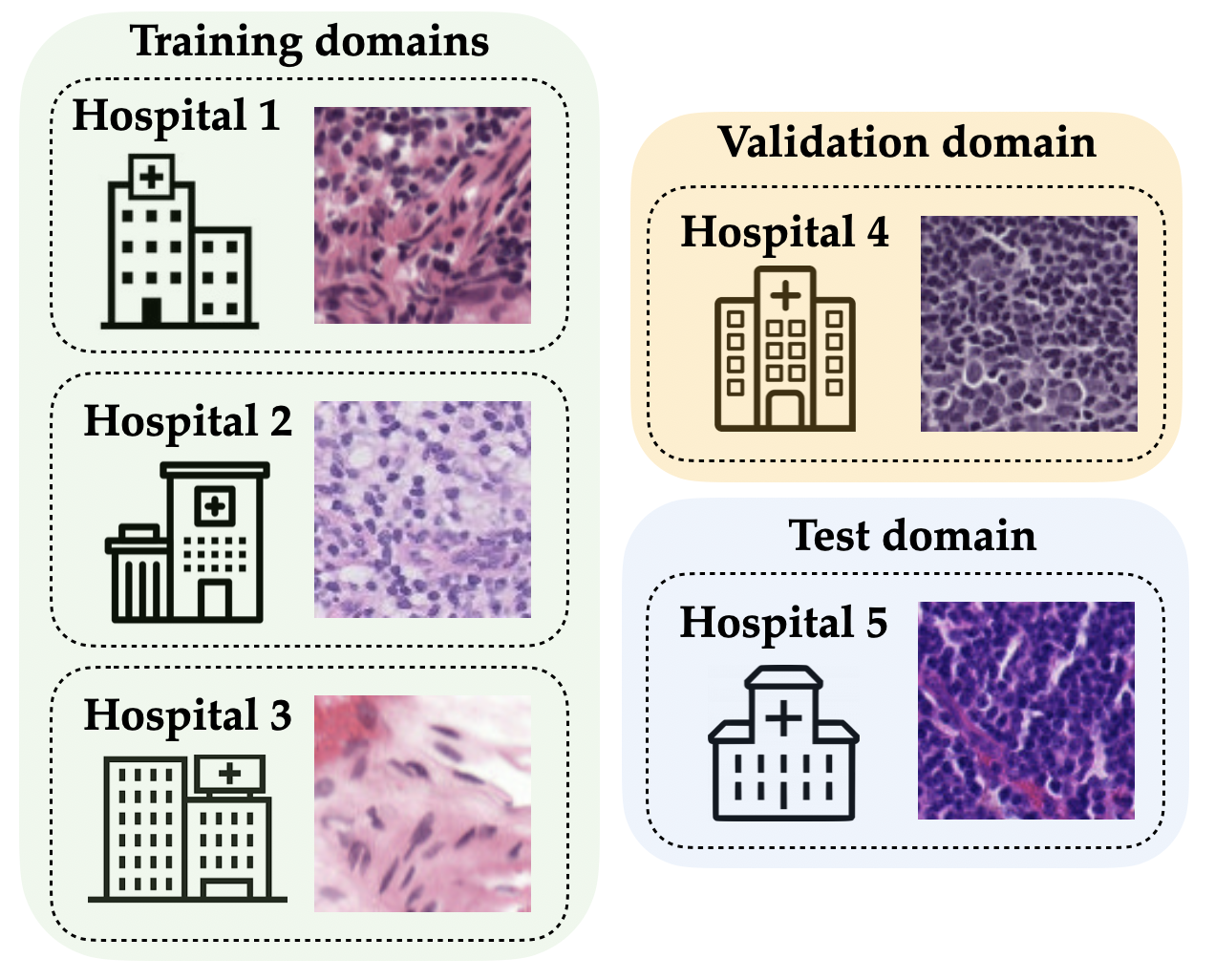}
        \caption{In domain generalization, the data are drawn from a family of related domains.  For example, in the \texttt{Camelyon17-WILDS} dataset \cite{koh2020wilds}, which contains images of cells, the domains correspond to different hospitals where these images were captured.}
        \label{fig:domain-gen}
    \end{subfigure}\quad
    \begin{subfigure}[b]{0.31\textwidth}
        \includegraphics[width=\textwidth]{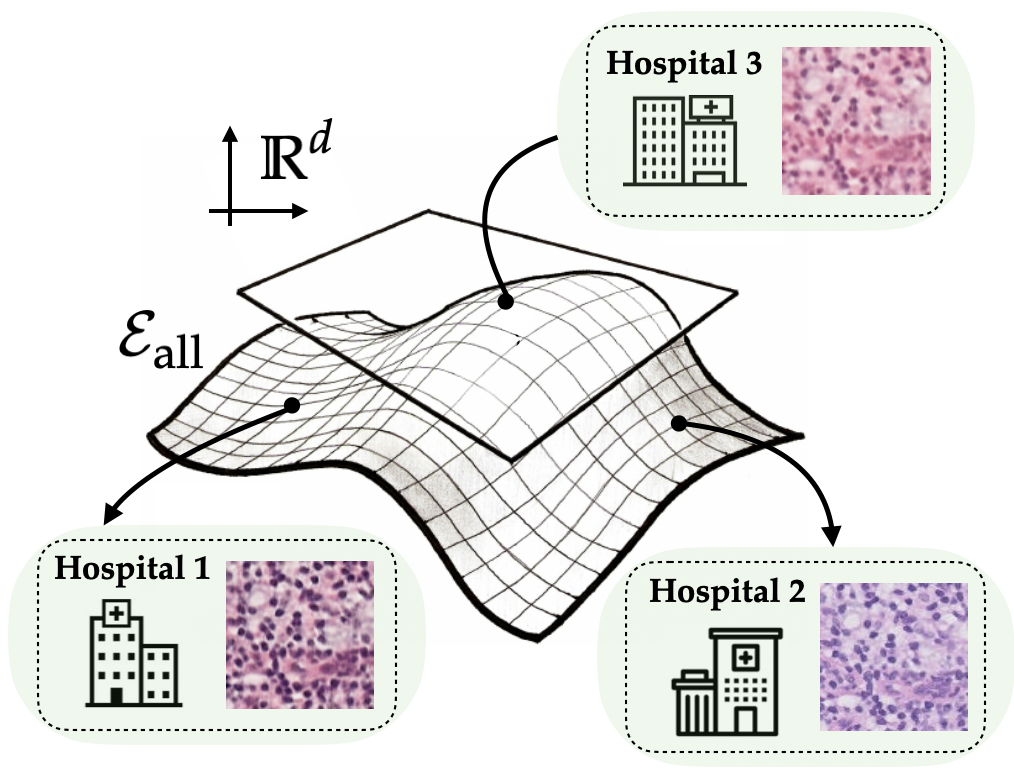}
        \caption{Each data point in a domain generalization task is observed in a particular domain $e\in\Eall$.  The set of all domains $\Eall$ can be thought of as an abstract space lying in $\R^p$.  In \texttt{Camelyon17-WILDS}, this space $\Eall$ corresponds to the set of all possible hospitals.}
        \label{fig:domain-envs}
    \end{subfigure} \quad
    \begin{subfigure}[b]{0.31\textwidth}
        \includegraphics[width=\textwidth]{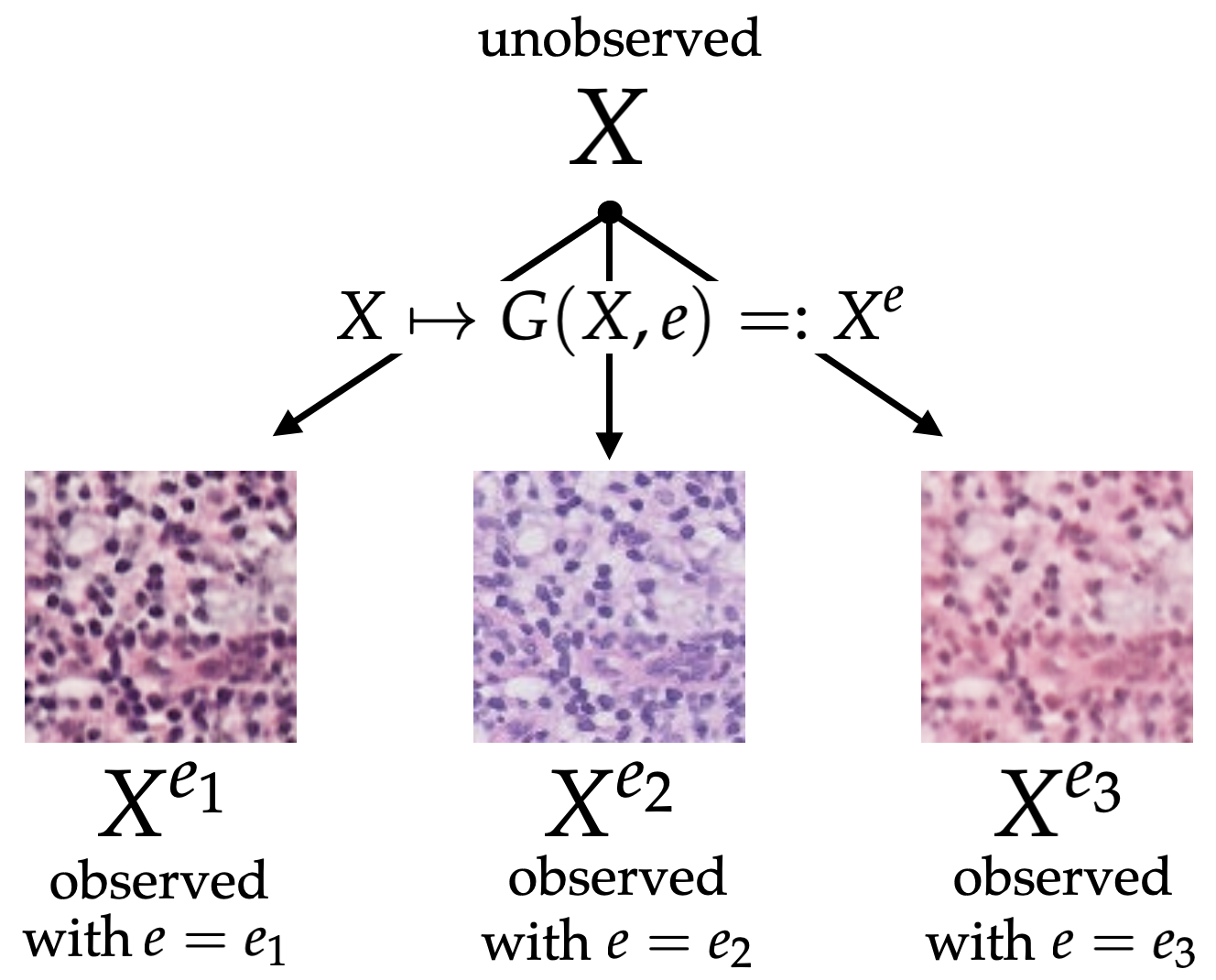}
        \caption{We assume that the variation from domain to domain is characterized by an underlying generative model $G(x,e)$, which transforms the unobserved random variable $X \mapsto G(X,e) := X^e$, where $X^e$ represents $X$ observed in any domain $e\in\Eall$.}
        \label{fig:model-diagram}
    \end{subfigure}
    \caption{An overview of the domain generalization problem setting used in this paper.}
    \label{fig:domain-gen-outline}
\end{figure}

The domain generalization setting is characterized by a pair of random variables $(X,Y)$ over instances $x\in\mathcal{X}\subseteq \R^d$ and corresponding labels $y\in\cal Y$, where $(X,Y)$ is jointly distributed according to an unknown probability distribution $\Prob(X,Y)$.  Ultimately, as in all of supervised learning tasks, the  objective in this setting is to learn a predictor $f$ such that $f(X) \approx Y$, meaning that $f$ should be able to predict the labels $y$ of corresponding instances $x$ for each $(x,y)\sim \Prob(X,Y)$.  However, unlike in standard supervised learning tasks, the domain generalization problem is complicated by the assumption that one cannot sample directly from $\Prob(X,Y)$.  Rather, it is assumed that we can only measure $(X,Y)$ under different environmental conditions, each of which corrupts or varies the data in a different way.  For example, in medical imaging tasks, these environmental conditions might correspond to the imaging techniques and stain patterns used at different hospitals; this is illustrated in Figure~\ref{fig:domain-gen}.  

To formalize this notion of environmental variation, we assume that data is drawn from a set of \emph{environments} or \emph{domains} $\Eall$ (see Figure \ref{fig:domain-envs}).  Concretely, each domain $e\in\Eall$ can be identified with a pair of random variables $(X^e, Y^e)$, which together denote the observation of the random variable pair $(X,Y)$ in environment $e$.  Given samples from a finite subset $\Etrain \subsetneq\Eall$ of domains, the goal of the domain generalization problem is to learn a predictor $f$ that generalizes across all possible environments, implying that $f(X) \approx Y$.  This can be summarized as follows:

\begin{problem}[Domain generalization] \label{prob:domain-gen}
Let $\Etrain \subsetneq \Eall$ be a finite subset of training domains, and assume that for each $e\in\Etrain$, we have access to a dataset $\mathcal{D}^e := \{(x_j^e, y_j^e)\}_{j=1}^{n_e}$ sampled \iid from $\Prob(X^e,Y^e)$.  Given a function class $\mathcal{F}$ and a loss function $\ell:\mathcal{Y}\times\mathcal{Y}\to\R_{\geq 0}$, our goal is to learn a predictor $f\in\calF$ using the data from the datasets $\calD^e$ that minimizes the worst-case risk over the entire family of domains $\Eall$.  That is, we want to solve the following optimization problem:
\begin{align}
    \minimize_{f\in\mathcal{F}} \: \max_{e\in\Eall} \: \E_{\Prob(X^e,Y^e)} \ell(f(X^e), Y^e). \tag{DG} \label{eq:domain-gen}
\end{align}
\end{problem}

In essence, in Problem \ref{prob:domain-gen} we seek a predictor $f\in\mathcal{F}$ that generalizes from the finite set of training domains $\Etrain$ to perform well on the set of all domains $\Eall$.  However, note that while the inner maximization in \eqref{eq:domain-gen} is over the set of all training domains $\Eall$, by assumption we do not have access to data from any of the domains $e\in\Eall \backslash \Etrain$, making this problem challenging to solve.  Indeed, as generalizing to arbitrary test domains is impossible~\cite{krueger2020out}, further structure is often assumed on the topology of $\Eall$ and on the corresponding distributions $\Prob(X^e,Y^e)$.

\subsection{Disentangling the sources of variation across environments}  The difficulty of a particular domain generalization task can be characterized by the extent to which the distribution of data in the unseen test domains $\Eall\backslash\Etrain$ resembles the distribution of data in the training domains $\Etrain$.  For instance, if the domains are assumed to be convex combinations of the training domains, as is often the case in multi-source domain generalization \cite{gan2016learning,matsuura2020domain,niu2015visual}, Problem \ref{prob:domain-gen} can be seen as an instance of distributionally robust optimization \cite{ben2009robust}.  More generally, in a similar spirit to~\cite{krueger2020out}, we identify two forms of variation across domains: \emph{covariate shift} and \emph{concept shift}.  These shifts characterize the extent to which the marginal distributions over instances $\Prob(X^e)$ and the instance-conditional distributions $\Prob(Y^e|X^e)$ differ between domains.  We capture these shifts in the following definition:

\begin{definition}[Covariate shift \& concept shift] \label{def:cov-and-concept-shift}  Problem \ref{prob:domain-gen} is said to experience \textbf{covariate shift} if environmental variation is due to differences between the set of marginal distributions over instances $\{\Prob(X^e)\}_{e\in\Eall}$.  On the other hand, Problem \ref{prob:domain-gen} is said to experience \textbf{concept shift} if environmental variation is due to changes amongst the instance-conditional distributions $\{\Prob(Y^e|X^e)\}_{e\in\Eall}$.
\end{definition}

The growing domain generalization literature encompasses a great deal of past work, wherein both of these shifts have been studied in various contexts~\cite{ben2010theory,david2010impossibility,bagnell2005robust,scholkopf2012causal,lipton2018detecting}, resulting in numerous algorithms designed to solve Problem~\ref{prob:domain-gen}.  Indeed, as this body of work has grown, new benchmark datasets have been developed which span the gamut between covariate and concept shift (see e.g.\ Figure~3 in~\cite{ye2021ood} and the discussion therein).  However, a large-scale empirical study recently showed that no existing algorithm can significantly outperform ERM across these standard domain generalization benchmarks when ERM is carefully implemented~\cite{gulrajani2020search}.  As ERM is known to fail in the presence natural distribution shifts~\cite{rosenfeld2018elephant}, this result highlights the critical need for new algorithms that can go beyond ERM toward solving Problem~\ref{prob:domain-gen}.

\section{Model-based domain generalization} \label{sect:mbdg}

In what follows, we introduce a new framework for domain generalization that we call \emph{Model-Based Domain Generalization} (MBDG).  In particular, we prove that when Problem \ref{prob:domain-gen} is characterized solely by covariate shift, then under a natural invariance-based condition, Problem \ref{prob:domain-gen} is equivalent to an infinite-dimensional constrained statistical learning problem, which forms the basis of MBDG.

\subsection{Formal assumptions for MBDG}  

In general, domain generalization tasks can be characterized by both covariate and concept shift.  However, in this paper, we restrict the scope of our theoretical analysis to focus on problems in which inter-domain variation is due solely to covariate shift through an underlying model of data generation.  Formally, we assume that the data in each domain $e\in\Eall$ is generated from the underlying random variable pair $(X,Y)$ via an unknown function $G$.  

\begin{assumption}[Domain transformation model] \label{assume:gen-model}
Let $\delta_e$ denote a Dirac distribution for $e\in\Eall$.  We assume that there exists\footnote{Crucially, although we assume the existence of a domain transformation model $G$, we emphasize that for many problems, it may be impossible to obtain or derive a simple analytic expression for $G$.  This topic will be discussed at length in Section~\ref{sect:alg} and in Appendix \ref{sect:dtms}.} a measurable function $G:\mathcal{X}\times \Eall \to \mathcal{X}$, which we refer to as a \emph{domain transformation model}, that parameterizes the inter-domain covariate shift via 
\begin{align}
    \Prob(X^e) =^d G \: \# \: (\Prob(X) \times \delta_e) \quad \forall e\in\Eall,
\end{align}
where $\#$ denotes the push-forward measure and $=^d$ denotes equality in distribution.
\end{assumption}

\noindent Informally, this assumption specifies that there should exist a function $G$ that relates the random variables $X$ and $X^e$ via the mapping $X \mapsto G(X,e) = X^e$.  In past work, this particular setting in which the instances $X^e$ measured in an environment $e$ are related to the underlying random variable $X$ by a generative model has been referred to as \emph{domain shift}~\cite[\S 1.8]{quinonero2009dataset}.  In our medical imaging example, the domain shift captured by a domain transformation model would characterize the mapping from the underlying distribution $\Prob(X)$ over different cells to the distribution $\Prob(X^e)$ of images of these cells observed at a particular hospital; this is illustrated in Figure \ref{fig:model-diagram}, wherein inter-domain variation is due to varying colors and stain patterns encountered at different hospitals. 

On the other hand, in our running medical imaging example, the label $y\sim Y$ describing whether a given cell contains a cancerous tumor should not depend on the lighting and stain patterns used at different hospitals.  In this sense, while in other applications, e.g.\ the datasets introduced in~\cite{arjovsky2019invariant,kattakinda2021focus}, the instance-conditional distributions can vary across domains, in this paper we assume that inter-domain variation is \emph{solely} characterized by the domain shift parameterized by $G$.

\begin{assumption}[Domain shift] \label{assume:cov-shift}
We assume that inter-domain variation is solely characterized by domain shift in the marginal distributions $\Prob(X^e)$, as described in Assumption~\ref{assume:gen-model}.  As a consequence, we assume that the instance-conditional distributions $\Prob(Y^e|X^e)$ are stable across domains, meaning that $Y^e$ and $Y$ are equivalent in distribution and that for each $x\in\calX$ and $y\in\calY$, it holds that
\begin{align}
    \Prob(Y = y|X = x) = \Prob(Y^e = y|X^e = G(x,e)) \quad \forall e\in\Eall.
\end{align}
\end{assumption}

\begin{figure}
    \centering
    \begin{subfigure}{0.48\textwidth}
        \centering
        \includegraphics[width=0.6\textwidth]{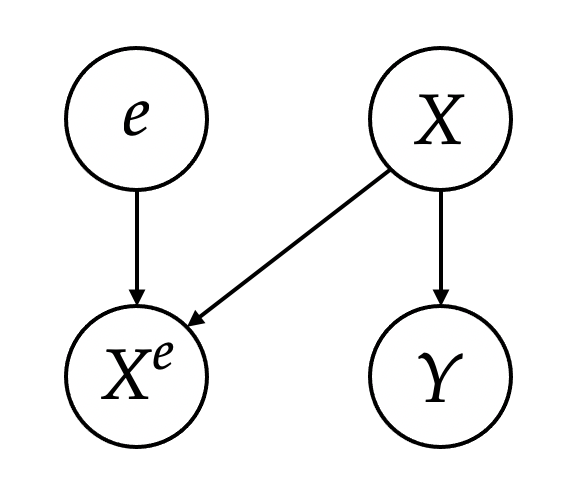}
        \caption{\textbf{Domain shift.}  In this paper, we assume that the instances $X^e$ from a domain $e\in\Eall$ are generated by a domain transformation model $G(X,e)$, resulting in domain shift.  Thus, in the above SCM, $X$ and $e$ are the sole causal ancestors of $X^e$.  Further, we assume that $e$ is not a causally related to $Y=Y^e$.}
        \label{fig:covar-shift-causal}
    \end{subfigure}\hfill
    \begin{subfigure}{0.48\textwidth}
        \centering
        \includegraphics[width=0.6\textwidth]{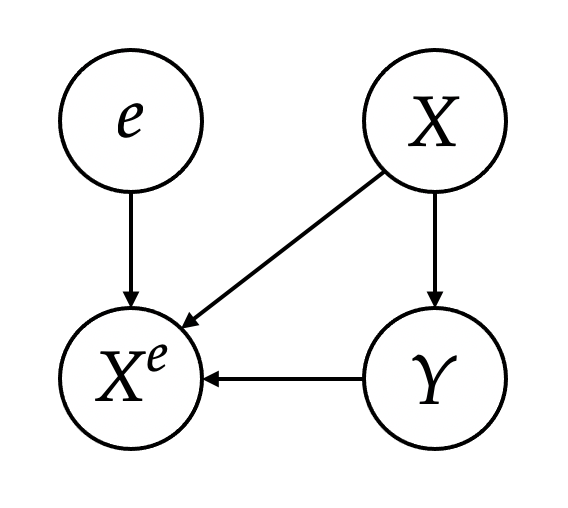}
        \caption{\textbf{Concept shift.}  In this figure, we illustrate a causal data generating model in which the instances $X^e$ can be (spuriously) correlated with the label $Y$, leading to concept shift.  Note that unlike in the SCM shown in (a), in this SCM, $Y$ is also a causal parent of $X^e$.}
        \label{fig:spur-corr-causal}
    \end{subfigure}
    \caption{\textbf{Causal interpretations of domain generalization tasks.}  We compare structural causal models (SCMs) for covariate shift and concept shift.  Throughout, the environment $e\in\Eall$ is assumed to be independent of $(X,Y)$, i.e.\ $e\indep (X,Y)$.}
    \label{fig:causal}
\end{figure}

\subsection{A causal interpretation}  The language of causal inference provides further intuition for the structure imposed on Problem~\ref{prob:domain-gen} by Assumptions~\ref{assume:gen-model} and~\ref{assume:cov-shift}.  In particular, the structural causal model (SCM) for problems in which data is generated according to the mechanism described in Assumptions~\ref{assume:gen-model} and~\ref{assume:cov-shift} is shown in Figure~\ref{fig:covar-shift-causal}.  Recall that in Assumption~\ref{assume:gen-model} imposes that $X$ and $e$ are \emph{causes} of the random variable $X^e$ via the mechanism $X^e = G(X,e)$.  This results in the causal links $e\longrightarrow X^e \longleftarrow X$. Further, in Assumption~\ref{assume:cov-shift}, we assume that $\Prob(Y^e|X^e)$ is fixed across environments, meaning that the label $Y$ is independent of the environment $e$.  In Figure~\ref{fig:covar-shift-causal}, this translates to there being no causal link between $e$ and $Y$.  

To offer a point of comparison, in Figure~\ref{fig:spur-corr-causal}, we show a different SCM that does not fulfill our assumptions.  Notice that in this SCM, $Y$ and $e$ are both causes of $X^e$, meaning that the distributions $\Prob(Y^e|X^e)$ can vary in domain dependent ways.  This gives rise to concept shift, which has also been referred to as \emph{spurious correlation}~\cite{arjovsky2019invariant,singla2021causal}.  Notably, the SCM shown in Figure~\ref{fig:spur-corr-causal} corresponds to the data generating procedure used to construct the \texttt{ColoredMNIST} dataset~\cite{arjovsky2019invariant}, wherein the MNIST digits in various domains $X^e$ are (spuriously) colorized according to the label~$Y$.\footnote{While the data-generating mechanism for \texttt{ColoredMNIST} does not fulfill our assumptions, the algorithm we propose in Section~\ref{sect:alg} still empirically achieves state-of-the-art results on \texttt{ColoredMNIST}.}

\subsection{Pulling back Problem~\ref{prob:domain-gen}}  

The structure imposed on Problem~\ref{prob:domain-gen} by Assumptions~\ref{assume:gen-model} and~\ref{assume:cov-shift} provides a concrete way of parameterizing large families of distributional shifts in domain generalization problems.  Indeed, the utility of these assumptions is that when taken together, they provide the basis for pulling-back Problem \ref{prob:domain-gen} onto the underlying distribution $\Prob(X,Y)$ via the domain transformation model $G$.  This insight is captured in the following proposition:

\begin{proposition} \label{prop:pull-back-domain-gen}
Under Assumptions~\ref{assume:gen-model} and~\ref{assume:cov-shift}, Problem \ref{prob:domain-gen} is equivalent to
\begin{align}
    \minimize_{f\in\cal F} \: \max_{e\in\Eall} \: \E_{\Prob(X,Y)}  \ell(f(G(X,e)), Y). \label{eq:pull-back-dg}
\end{align}
\end{proposition}

\noindent The proof of this fact is a straightforward consequence of the decomposition $\Prob(X^e,Y^e) = \Prob(Y^e|X^e)\cdot \Prob(X^e)$ in conjunction with Assumptions~\ref{assume:gen-model} and~\ref{assume:cov-shift} (see Appendix~\ref{sect:proof-of-mbdg}).  Note that this result allows us to implicitly absorb each of the domain distributions $\Prob(X^e,Y^e)$ into the domain transformation model.  Thus, the outer expectation in~\eqref{eq:pull-back-dg} is defined over the underlying distribution $\Prob(X,Y)$.  On the other hand, just as in~\eqref{eq:domain-gen}, this problem is still a challenging statistical min-max problem.  To this end, we next introduce a new notion of invariance with respect to domain transformation models, which allows us to reformulate the problem in \eqref{eq:pull-back-dg} as a semi-infinite constrained optimization problem.

\subsection{A new notion of model-based invariance}  Common to much of the domain generalization literature is the idea that predictors should be invariant to inter-domain changes.  For instance, in \cite{arjovsky2019invariant} the authors seek to learn an \emph{equipredictive representation} $\Phi:\calX\to\calZ$ \cite{koyama2020out}, i.e.\ an intermediate representation that satisfies
\begin{align}
    \Prob(Y^{e_1}|\Phi(X^{e_1})) = \Prob(Y^{e_2}|\Phi(X^{e_2})) \quad \forall e_1,e_2\in\Eall.
\end{align}
Despite compelling theoretical motivation for this approach, it has been shown that current algorithms which seek equipredictive representations do not significantly improve over ERM  \cite{rosenfeld2020risks,kamath2021does,nagarajan2020understanding,ahuja2020empirical}.  With this in mind and given the additional structure introduced in Assumptions~\ref{assume:gen-model} and~\ref{assume:cov-shift}, we introduce a new definition of invariance with respect to the variation captured by the underlying domain transformation model $G$.

\begin{definition}[$G$-invariance] \label{def:g-invar}
Given a domain transformation model $G$, we say a classifier $f$ is \textbf{$\mathbf{G}$-invariant} if it holds for all $e\in\Eall$ that $f(x) = f(G(x,e)) \text{ almost surely when }  x\sim \Prob(X)$.
\end{definition}

\noindent Concretely, this definition says that a predictor $f$ is $G$-invariant if environmental changes under $G(x,e)$ cannot change the prediction returned by $f$.  Intuitively, this notion of invariance couples with the definition of domain shift, in the sense that we expect that a prediction should return the same prediction for any realization of data under $G$.  Thus, whereas equipredictive representations are designed to enforce invariance of in an intermediate representation space $\calZ$, Definition \ref{def:g-invar} is designed to enforce invariance directly on the predictions made by $f$.  In this way, in the setting of Figure~\ref{fig:domain-gen-outline}, $G$-invariance would imply that the predictor $f$ would return the same label for a given cluster of cells regardless of the hospital at which these cells were imaged. 

\subsection{Formulating the MBDG optimization problem}

The $G$-invariance property described in the previous section is the key toward reformulating the min-max problem in~\eqref{eq:pull-back-dg}.  Indeed, the following proposition follows from Assumptions~\ref{assume:gen-model} and~\ref{assume:cov-shift} and from the definition of $G$-invariance.

\begin{proposition} \label{prop:mbdg}
Under Assumptions~\ref{assume:gen-model} and~\ref{assume:cov-shift}, if we restrict the domain $\calF$ of Problem \ref{prob:domain-gen} to the set of $G$-invariant predictors, then Problem \ref{prob:domain-gen} is equivalent to the following constrained optimization problem:
\begin{mdframed}[roundcorner=5pt, backgroundcolor=yellow!8]
\begin{alignat}{2}
    P^\star \triangleq &\minimize_{f\in\mathcal{F}} \: &&R(f) \triangleq \E_{\Prob(X,Y)} \ell(f(X),Y) \tag{MBDG} \label{eq:model-based-domain-gen} \\
    &\st  &&f(x) = f(G(x,e)) \quad \text{ a.e. } x\sim \Prob(X) \:\: \forall e\in\Eall. \notag
\end{alignat}
\end{mdframed}
\end{proposition}

\noindent Here a.e.\ stands for ``almost everywhere'' and $R(f)$ is the statistical risk of a predictor $f$ with respect to the underlying random variable pair $(X,Y)$.  Note that unlike~\eqref{eq:pull-back-dg}, \eqref{eq:model-based-domain-gen} is not a composite optimization problem, meaning that the inner maximization has been eliminated.  In essence, the proof of Proposition~\ref{prob:model-based-domain-gen} relies on the fact that $G$-invariance implies that predictions should not change across domains (see Appendix~\ref{sect:proof-of-mbdg}).

The optimization problem in \eqref{eq:model-based-domain-gen} forms the basis of our Model-Based Domain Generalization framework.  To explicitly contrast this problem to Problem \ref{prob:domain-gen}, we introduce the following concrete problem formulation for Model-Based Domain Generalization.

\begin{problem}[Model-Based Domain Generalization] \label{prob:model-based-domain-gen}
As in Problem \ref{prob:domain-gen}, let $\Etrain \subsetneq \Eall$ be a finite subset of training domains and assume that we have access to datasets $\mathcal{D}^e$ $\forall e\in\Etrain$.  Then under Assumptions~\ref{assume:gen-model} and~\ref{assume:cov-shift}, the goal of Model-Based Domain Generalization is to use the data from the training datasets to solve the semi-infinite constrained optimization problem in \eqref{eq:model-based-domain-gen}.
\end{problem}

\subsection{Challenges in solving Problem \ref{prob:model-based-domain-gen}}  Problem \ref{prob:model-based-domain-gen} offers a new, theoretically-principled perspective on Problem \ref{prob:domain-gen} when data varies from domain to domain with respect to an underlying domain transformation model $G$.  However, just as in general solving the min-max problem of Problem \ref{prob:domain-gen} is known to be difficult, the optimization problem in \eqref{eq:model-based-domain-gen} is also challenging to solve for several reasons:

\begin{enumerate}[itemsep=0.1em, leftmargin=4em]
    \item[(C1)] \textbf{Strictness of $G$-invariance.}  The $G$-invariance constraint in \eqref{eq:model-based-domain-gen} is a strict equality constraint and is thus difficult to enforce in practice.  Moreover, although we require that $f(G(x,e)) = f(x)$ holds for almost every $x\sim\Prob(X)$ and $\forall e\in\Eall$, in practice we only have access to samples from $\Prob(X^e)$ for a finite number of domains $\Etrain\subsetneq\Eall$.  Thus, for some problems it may be impossible to evaluate whether a predictor is $G$-invariant.
    
    \item[(C2)] \textbf{Constrained optimization.} Problem \ref{prob:model-based-domain-gen} is a constrained problem over an infinite dimensional functional space $\calF$.  While it is common to replace $\calF$ with a parameterized function class, this approach creates further complications.  Firstly, enforcing constraints on most modern, non-convex function classes such as the class of deep neural networks is known to be a challenging problem~\cite{chamon2020probably}.  Further, while a variety of heuristics exist for enforcing constraints on such classes (e.g.\ regularization), these approaches cannot guarantee constraint satisfaction for constrained problems~\cite{chamon2020empirical}.
    
    \item[(C3)] \textbf{Unavailable data.} We do have access to the set of all domains $\Eall$ or to the underlying distribution $\Prob(X,Y)$.  Not only does this limit our ability to enforce $G$-invariance (see (C1)), but it also complicates the task of evaluating the statistical risk $R(f)$ in~\eqref{eq:model-based-domain-gen}, since $R(f)$ is defined with respect to $\Prob(X,Y)$.
    
    \item[(C4)] \textbf{Unknown domain transformation model.}  In general, we do not have access to the underlying domain transformation model $G$.  While an analytic expression for $G$ may be known for simpler problems (e.g.\ rotations of the MNIST digits), analytic expressions for $G$ are most often difficult or impossible to obtain.  For instance, obtaining a simple equation that describes the variation in color and contrast in Figure~\ref{fig:model-diagram} would be challenging.
\end{enumerate}

\noindent In the ensuing sections, we explicitly address each of these challenges toward developing a tractable method for approximately solving Problem \ref{prob:model-based-domain-gen} with guarantees on optimality.  In particular, we discuss challenges (C1), (C2), and (C3) in Section~\ref{sect:approx-mbdg}.  We then discuss (C4) in Section~\ref{sect:learning-dtms}.

\section{Data-dependent duality gap for MBDG} \label{sect:approx-mbdg}

In this section, we offer a principled analysis of Problem \ref{prob:model-based-domain-gen}.  In particular, we first address (C1) by introducing a relaxation of the $G$-invariance constraint that is compatible with modern notions of constrained PAC learnability \cite{chamon2020probably}.  Next, to resolve the fundamental difficulty involved in solving constrained statistical problems highlighted in (C2), we follow~\cite{chamon2020empirical} by formulating the parameterized dual problem, which is unconstrained and thus more suitable for learning with deep neural networks.  Finally, to address (C3), we introduce an empirical version of the parameterized dual problem and explicitly characterize the data-dependent duality gap between this problem and Problem \ref{prob:model-based-domain-gen}.  At a high level, this analysis results in an \emph{unconstrained} optimization problem which is guaranteed to produce a solution that is close to the solution of Problem~\ref{prob:domain-gen} (see Theorem~\ref{thm:duality-gap}).

Throughout this section, we have chosen to present our results somewhat informally by  deferring preliminary results and regularity assumptions to the appendices.  Proofs of each of the results in this section are provided in Appendix~\ref{app:omitted-proofs}.

\subsection{Addressing (C1) by relaxing the $G$-invariance constraint}  Among the challenges inherent to solving Problem \ref{prob:model-based-domain-gen}, one of the most fundamental is the difficulty of enforcing the $G$-invariance equality constraint.  Indeed, it is not clear a priori how to enforce a hard invariance constraint on the class $\calF$ of predictors.  To alleviate some of this difficulty, we introduce the following relaxation of Problem \ref{prob:model-based-domain-gen}:
\begin{alignat}{2}
    P^\star(\gamma) \triangleq &\minimize_{f\in\mathcal{F}} &&R(f) \label{eq:relax-mbdg} \\
    &\st &&\mathcal{L}^e(f) \triangleq \E_{\Prob(X)}  d\big(f(X), f(G(X,e))\big)  \leq \gamma \quad\forall e\in\Eall  \notag
\end{alignat}
where $\gamma>0$ is a fixed margin the controls the extent to which we enforce $G$-invariance and $d:\calP(\mathcal{Y})\times\mathcal{P}(\mathcal{Y})\to\R_{\geq 0}$ is a distance metric over the space of probability distributions on $\calY$.  By relaxing the equality constraints in~\eqref{eq:model-based-domain-gen} to the inequality constraints in~\eqref{eq:relax-mbdg} and under suitable conditions on $\ell$ and $d$, \eqref{eq:relax-mbdg} can be characterized by the recently introduced constrained PAC learning framework, which can provide learnability guarantees on constrained statistical problems (see Appendix \ref{sect:pacc} for details).

While at first glance this problem may appear to be a significant relaxation of the MBDG optimization problem in~\eqref{eq:model-based-domain-gen}, when $\gamma = 0$ and under mild conditions on $d$, the two problems are equivalent in the sense that $P^\star(0) = P^\star$ (see Proposition~\ref{prop:relaxation}).  Indeed, we note that the conditions we require on $d$ are not restrictive, and include the KL-divergence and more generally the family of $f$-divergences.  
Moreover, when the margin $\gamma$ is strictly larger than zero, under the assumption that the perturbation function $P^\star(\gamma)$ is $L$-Lipschitz continuous, we show in Remark~\ref{rmk:gamma-remark} that $|P^\star - P^\star(\gamma)| \leq L\gamma$, meaning that the gap between the problems is relatively small when $\gamma$ is chosen to be small.  In particular, when strong duality holds for~\eqref{eq:model-based-domain-gen}, this Lipschitz constant $L$ is equal to the $L^1$ norm of the optimal dual variable $\norm{\nu^\star}_{L^1}$ for~\eqref{eq:model-based-domain-gen} (see Remark~\ref{rmk:opt-dual-var}).

\subsection{Addressing (C2) by formulating the parameterized dual problem}  As written, the relaxation in \eqref{eq:relax-mbdg} is an infinite-dimensional constrained optimization problem over a functional space $\calF$ (e.g.\ $L^2$ or the space of continuous functions).   Optimization in this infinite-dimensional function space is not tractable, and thus we follow the standard convention by leveraging a finite-dimensional parameterization of $\cal F$, such as the class of deep neural networks~\cite{hornik1991approximation,hornik1989multilayer}.  The approximation power of such a parameterization can be captured in the following definition: 
\begin{definition}[$\epsilon$-parameterization]\label{def:eps-param}
Let $\mathcal{H} \subseteq \R^p$ be a finite-dimensional parameter space.  For $\epsilon > 0$, a function $\varphi:\calH\times\calX \to \calY$ is said to be an $\pmb\epsilon$\textbf{-parameterization} of $\mathcal{F}$ if it holds that for each $f\in\mathcal{F}$, there exists a parameter $\theta\in\mathcal{H}$ such that $\E_{\Prob(X)} \norm{\varphi(\theta, x) - f(x)}_\infty \leq \epsilon$.
\end{definition}

\noindent The benefit of using such a parameterization is that optimization is generally more tractable in the parameterized space $\mathcal{A}_\epsilon:= \{\varphi(\theta, \cdot): \theta\in\mathcal{H}\} \subseteq \mathcal{F}$.  However, typical parameterizations often lead to nonconvex problems, wherein methods such as SGD cannot guarantee constraint satisfaction. And while several heuristic algorithms have been designed to enforce constraints over common parametric classes~\cite{pathak2015constrained,chen2018approximating,frerix2020homogeneous,amos2017optnet,ravi2018constrained,donti2021dc3}, these approaches cannot provide guarantees on the underlying statistical problem of interest \cite{chamon2020empirical}.  Thus, to provide guarantees on the underlying statistical problem in Problem \ref{prob:model-based-domain-gen}, given an $\epsilon$-parameterization $\varphi$ of $\mathcal{F}$, we consider the following saddle-point problem:
\begin{align}
    D_\epsilon^\star(\gamma) \triangleq \maximize_{\lambda\in\mathcal{P}(\Eall)} \: \min_{\theta\in\mathcal{H}} \:  R(\theta) + \int_{\Eall} \left[\mathcal{L}^e(\theta) - \gamma\right] \text{d}\lambda(e). \label{eq:param-dual}
\end{align}
where $\calP(\Eall)$ is the space of normalized probability distributions over $\Eall$ and $\lambda\in\calP(\Eall)$ is the (semi-infinite) dual variable.  Here we have slightly abused notation to write $R(\theta) = R(\varphi(\theta, \cdot))$ and $\calL^e(\theta) = \calL^e(\varphi(\theta, \cdot))$.  One can think of \eqref{eq:param-dual} as the dual problem to \eqref{eq:relax-mbdg} solved over the parametric space $\mathcal{A}_\epsilon$.  Notice that unlike Problem \ref{prob:model-based-domain-gen}, the problem in \eqref{eq:param-dual} is \emph{unconstrained}, making it much more amenable for optimization over the class of deep neural networks.  Moreover, under mild conditions, the optimality gap between \eqref{eq:relax-mbdg} and \eqref{eq:param-dual} can be explicitly bounded as follows:

\begin{proposition}[Parameterization gap] \label{prop:param-gap}
Let $\gamma > 0$ be given.  Under mild regularity assumptions (see Assumption~\ref{assume:lipschitz} in Appendix~\ref{app:proof-param-gap}) on $\ell$ and $d$, there exists a small universal constant $k$ such that
\begin{align}
    P^\star(\gamma) \leq D^\star_\epsilon(\gamma) \leq P^\star(\gamma) + \epsilon k \left(1 + \norm{\lambda_\text{pert}^\star}_{L^1} \right),
\end{align}
where $\lambda^\star_\text{pert}$ is the optimal dual variable for a perturbed version of~\eqref{eq:relax-mbdg} in which the constraints are tightened to hold with margin $\gamma - k\epsilon$.
\end{proposition}

\noindent In this way, solving the parameterized dual problem in \eqref{eq:param-dual} provides a solution that can be used to recover a close approximation of the primal problem in~\eqref{eq:relax-mbdg}.  To see this, observe that Prop.\ \ref{prop:param-gap} implies that $|D_\epsilon^\star(\gamma) - P^\star(\gamma)| \leq \epsilon k (1 + ||\lambda_\text{pert}^\star||_{L^1} )$.  This tells us that the gap between $P^\star(\gamma)$ and $D^\star_\epsilon(\gamma)$ is small when we use a tight $\epsilon$-parameterization of $\calF$.

\subsection{Addressing (C3) by bounding the empirical duality gap}  The parameterized dual problem in \eqref{eq:param-dual} gives us a principled way to address Problem \ref{prob:model-based-domain-gen} in the context of deep learning.  However, complicating matters is the fact that we do not have access to the full distribution $\Prob(X,Y)$ or to data from any of the domains in $\Eall\backslash\Etrain$.  In practice, it is ubiquitous to solve optimization problems such as \eqref{eq:param-dual} over a finite sample of $N$ data points drawn from $\Prob(X,Y)$\footnote{Indeed, in practice we do not have access to any samples from $\Prob(X,Y)$.  In Section~\ref{sect:alg}, we argue that the $N$ samples from $\Prob(X,Y)$ can be replaced by the $\sum_{e\in\Etrain} n_e$ samples drawn from the training datasets $\calD^e$.}.  More specifically, given $\{(x_j, y_j)\}_{j=1}^N$ drawn i.i.d.\ from the underlying random variables $(X,Y)$, we consider the empirical counterpart of \eqref{eq:param-dual}:
\begin{align}
    D^\star_{\epsilon, N, \Etrain}(\gamma) \triangleq \maximize_{\lambda(e)\geq 0, \: e\in\Etrain} \: \min_{\theta\in\mathcal{H}} \: \hat{\Lambda}(\theta, \lambda) \triangleq \hat{R}(\theta) + \frac{1}{|\Etrain|}\sum_{e\in\Etrain} \left[\hat{\mathcal{L}}^e(\theta) - \gamma\right] \lambda(e) \label{eq:param-empir-dual}
\end{align}
where $\hat R(\theta)$ and $\hat \calL^e(\theta)$ are the empirical counterparts of $R(\theta)$ and $\calL^e(\theta)$, i.e.
\begin{align}
    \hat{R}(\theta) := \frac{1}{N}\sum_{j=1}^N \ell(\varphi(\theta, x_j), y_j)  \quad\text{and}\quad \hat{\mathcal{L}}^e(\theta) := \frac{1}{N} \sum_{j=1}^N d(\varphi(\theta, x_j), \varphi(\theta, G(x_j, e))),
\end{align}
and $\hat{\Lambda}(\theta, \lambda)$ is the empirical Lagrangian.  Notably, the duality gap between the solution to \eqref{eq:param-empir-dual} and the original model-based problem in  \eqref{eq:model-based-domain-gen} can be explicitly bounded as follows.

\begin{theorem}[Data-dependent duality gap] \label{thm:duality-gap}
Let $\epsilon > 0$ be given, and let $\varphi$ be an $\epsilon$-parameterization of $\mathcal{F}$. Under mild regularity assumptions on $\ell$ and $d$ and assuming that $\mathcal{A}_\epsilon$ has finite VC-dimension, with probability $1-\delta$ over the $N$ samples from $\Prob(X,Y)$ we have that
\begin{align}
    |P^\star - D_{\epsilon,N,\Etrain}^\star(\gamma) | \leq L\gamma + \epsilon k \left(1 + \norm{\lambda_\text{pert}^\star}_{L^1} \right) + {\cal O}\left( \sqrt{\log(N)/N}\right)
\end{align}
where $L$ is the Lipschitz constant of $P^\star(\gamma)$ and $k$ and $\lambda^\star_\text{pert}$ are as defined in Proposition~\ref{prop:param-gap}.
\end{theorem}

\noindent The key message to take away from Theorem \ref{thm:duality-gap} is that given samples from $\Prob(X,Y)$, the duality gap incurred by solving the empirical problem in \eqref{eq:param-empir-dual} is small when (a) the $G$-invariance margin $\gamma$ is small, (b) the parametric space $\calA_\epsilon$ is a close approximation of $\calF$, and (c) we have access to sufficiently many samples.  Thus, assuming that Assumptions~\ref{assume:gen-model} and~\ref{assume:cov-shift} hold, the solution to the domain generalization problem in Problem~\ref{prob:domain-gen} is closely-approximated by the solution to the empirical, parameterized dual problem in~\eqref{eq:param-empir-dual}.
\section{Learning domain transformation models from data} \label{sect:learning-dtms}

\begin{figure}
    \centering
    \includegraphics[width=0.8\textwidth]{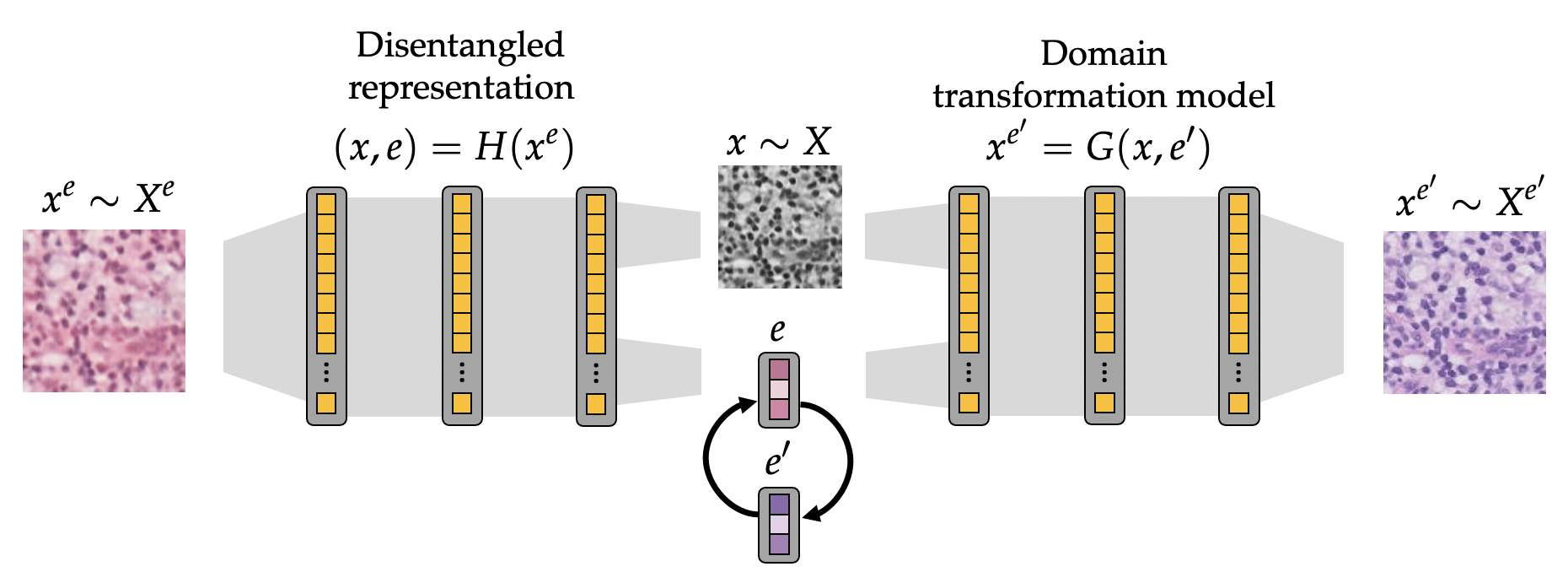}
    \caption{\textbf{Multi-modal image-to-image translation networks.}  In this paper, we parameterize domain transformation models via multi-modal image-to-image translation networks, which can be trained to map images from one domain so that they resemble images from different domains.}
    \label{fig:dtm-arch}
\end{figure}

Regarding challenge (C4), critical to our approach is having access to the underlying domain transformation model $G(x,e)$. For the vast majority of settings, the underlying function $G(x,e)$ is not known a priori and cannot be represented by a simple expression.  For example, obtaining a closed-form expression for a model that captures the variation in coloration, brightness, and contrast in the medical imaging dataset shown in Figure~\ref{fig:domain-gen-outline} would be challenging. 

\subsection{Multimodal image-to-image translation networks}

To address this challenge, we argue that a realistic \emph{approximation} of the underlying domain transformation model can be learned from the instances drawn from the training datasets $\calD^e$ for $e\in\Etrain$.  In this paper, to learn domain transformation models, we train multimodal image-to-image translation networks (MIITNs) on the instances drawn from the training domains. MIITNs are designed to transform samples from one dataset so that they resemble a diverse collection of images from another dataset.  That is, the constraints used to train these models enforce that a diverse array of samples is outputted for each input image.  This feature precludes the possibility of learning trivial maps between domains, such as the identity transformation. 

\begin{table}
    \centering
    \scalebox{0.85}{
    \begin{tabular}{|c|c|c|} \hline
        \thead{Dataset} & \thead{Original} & \thead{Samples from learned domain transformation models $G(x,e)$} \\ \hline
        \texttt{ColoredMNIST} 
        & \imgintable{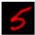}
        & 
        \imgintable{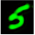}
        \imgintable{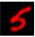}
        \imgintable{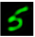}
        \imgintable{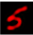}
        \\ \hline
        \texttt{\makecell{Camelyon17- \\ WILDS}}
        & \imgintable{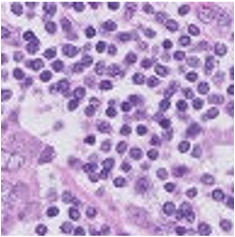} 
        & \imgintable{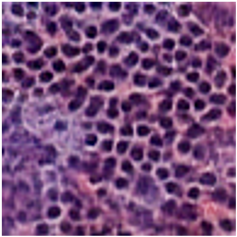}
        \imgintable{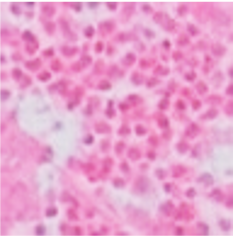}
        \imgintable{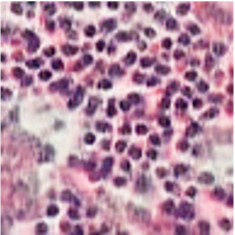}
        \imgintable{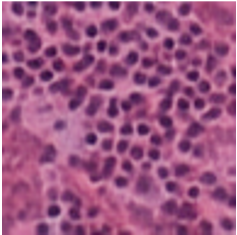}
        \\ \hline
        \texttt{\makecell{FMoW- \\ WILDS}} 
        & \imgintable{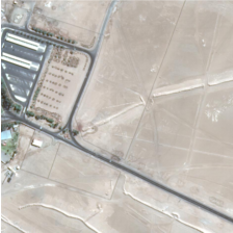}
        & \imgintable{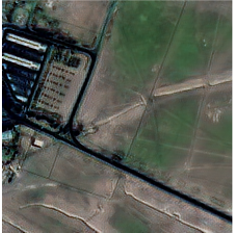}
        \imgintable{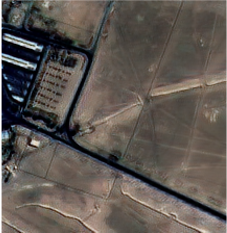}
        \imgintable{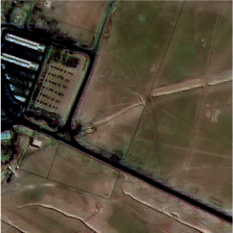}
        \imgintable{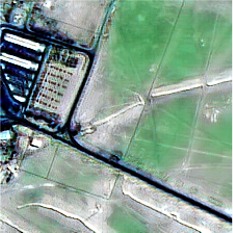}
        \\ \hline
        \texttt{PACS} 
        & \imgintable{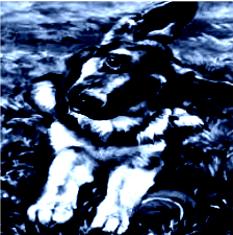} 
        & \imgintable{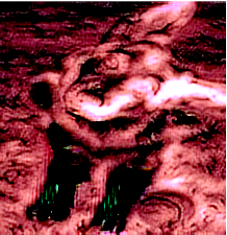} 
        \imgintable{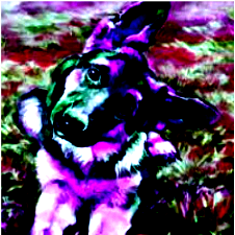} 
        \imgintable{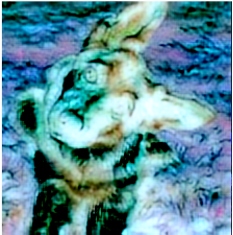} 
        \imgintable{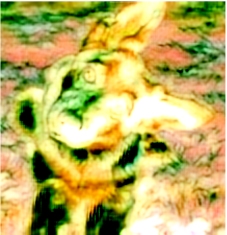} 
        \\ \hline
    \end{tabular}}
    \caption{We show samples from domain transformation models trained on images from the training datasets $\calD^e$ for $e\in\Etrain$ using the MUNIT architecture for the \texttt{Camelyon17-WILDS}, \texttt{FMOW-WILDS}, and \texttt{PACS} datasets.}
    \label{tab:model-based-samples}
\end{table}

As illustrated in Figure~\ref{fig:dtm-arch}, these architectures generally consist of two components: a disentangled representation~\cite{higgins2018towards} and a generative model.   The role of the disentangled representation is to recover a sample $x$ generated according to $X$ from a instance $x^e$ observed in a particular domain $e\in\Eall$.  In other words, for a fixed instance $x^e = G(x,e)$, the disentangled representation is designed to disentangle $x$ from $e$ via $(x,e) = H(x^e)$.  On the other hand, the role of the generative is to map each instance $x\sim X$ to a realization in a new environment $e'$.  Thus, given $x$ and $e$ at the output of the disentangled representation, we generate an instance from a new domain by replacing the environmental code $e$ with a different environmental parameter $e'\in\Eall$ to produce the instance $x^{e'} = G(x, e')$.  In this way, MIITNs are a natural framework for learning domain transformation models, as they facilitate 1) recovering samples from $X$ via the disentangled representation, and 2) generating instances from new domains in a multimodal fashion.

\paragraph{Samples from learned domain transformation models.}  In each of the experiments in Section \ref{sect:experiments}, we use the MUNIT architecture introduced in \cite{huang2018multimodal} to parameterize MIITNs.  As shown in Table \ref{tab:model-based-samples} and in Appendix \ref{sect:dtms}, models trained using the MUNIT architecture learn accurate and diverse transformations of the training data, which often generalize to generate images from new domains.  Notice that in this table, while the generated samples still retain the characteristic features of the input image (e.g.\ in the top row, the cell patterns are the same across the generated samples), there is clear variation between the generated samples.  Although these learned models cannot be expected to capture the full range of inter-domain generalization in the unseen test domains $\Eall\backslash\Etrain$, in our experiments, we show that these learned models are sufficient to significantly advance the state-of-the-art on several domain generalization benchmarks.

\section{A principled algorithm for Model-Based Domain Generalization} \label{sect:alg}

Motivated by the theoretical results in Section~\ref{sect:approx-mbdg} and the approach for learning domain transformation models in Section~\ref{sect:learning-dtms}, we now introduce a new domain generalization algorithm designed to solve the empirical, parameterized dual problem in~\eqref{eq:param-empir-dual}.  We emphasize that while our theory relies on the assumption that inter-domain variation is solely characterized by covariate shift, our algorithm is broadly applicable to problems with or without covariate shift (see the experimental results in Section~\ref{sect:experiments}).  In particular, assuming access to an appropriate learned domain transformation model $G$, we leverage $G$ toward solving the unconstrained dual optimization problem in~\eqref{eq:param-empir-dual} via a primal-dual iteration. 

\subsection{Primal-dual iteration}  Given a learned approximation $G(x,e)$ of the underlying domain transformation model, the next step in our approach is to use a primal-dual iteration \cite{bertsekas2015convex} toward solving \eqref{eq:param-empir-dual} using the training datasets $\calD^e$.  As we will show, the primal-dual iteration is a natural algorithmic choice for solving the empirical, parameterized dual problem in~\eqref{eq:param-empir-dual}.  Indeed, because the outer maximization in \eqref{eq:param-empir-dual} is a linear program in $\lambda$, the primal-dual iteration can be characterized by alternating between the following steps:

\begin{minipage}{\textwidth}
\begin{minipage}{.44\linewidth}
\begin{align}
  \theta^{(t+1)} \in \rho\mbox{-}\argmin_{\theta\in\calH} \: \hat{\Lambda}(\theta, \lambda^{(t)}) \label{eq:primal-step}
\end{align}
\end{minipage}%
\begin{minipage}{.52\linewidth}
\begin{align}
  \lambda^{(t+1)}(e) \gets \left[\lambda^{(t)}(e) + \eta \left(\hat{\calL}^e(\theta) - \gamma \right)\right]_+ \label{eq:dual-step}
\end{align}
\end{minipage}
\end{minipage}
\vspace{0.5em}

\noindent Here $[\cdot]_+ = \max\{0, \cdot\}$, $\eta > 0$ is the dual step size, and $\rho\mbox{-}\argmin$ denotes a solution that is $\rho$-close to being a minimizer, i.e.\ it holds that
\begin{align}
    \hat{\Lambda}(\theta^{(t+1)}, \lambda^{(t)}) \leq \min_{\theta\in\calH} \hat{\Lambda}(\theta, \lambda^{(t)}) + \rho.
\end{align}
\noindent For clarity, we refer to \eqref{eq:primal-step} as the primal step, and we call \eqref{eq:dual-step} the dual step.  

The utility of running this primal-dual scheme is as follows.  It can be shown that if this iteration is run for sufficiently many steps and with small enough step size, the iteration convergences with high probability to a solution which closely approximates the solution to Problem~\ref{prob:model-based-domain-gen}.  In particular, this result is captured in the following theorem\footnote{For clarity, we state this theorem informally in the main text; a full statement of the theorem and proof are provided in Appendix~\ref{sect:primal-dual-conv}.}:

\begin{theorem}[Primal-dual convergence] \label{thm:primal-dual}
Assuming that $\ell$ and $d$ are $[0,B]$-bounded, $\calH$ has finite VC-dimension, and under mild regularity conditions on \eqref{eq:param-empir-dual}, the primal-dual pair $(\theta^{(T)}, \lambda^{(T)})$ obtained after running the alternating primal-dual iteration in \eqref{eq:primal-step} and \eqref{eq:dual-step} for $T$ steps with step size $\eta$, where
\begin{align}
    T \triangleq \left\lceil \frac{1}{2\eta \kappa} \right\rceil + 1 \qquad\text{and}\qquad \eta\leq \frac{2\kappa}{|\Etrain|B^2}
\end{align}
satisfies the following inequality:
\begin{align}
    |P^\star - \hat{\Lambda}(\theta^{(T)}, \mu^{(T)})| \leq K(\rho, \kappa, \gamma) + \mathcal{O}\left(\sqrt{\log(N)/N}\right).
\end{align}
Here $\kappa = \kappa(\epsilon)$ is a constant that captures the regularity of the parametric space $\calH$ and $K(\rho, \kappa, \gamma)$ is a small constant depending linearly on $\rho$, $\kappa$, and $\gamma$.
\end{theorem}

\noindent This theorem means that by solving the empirical, parameterized dual problem in~\ref{eq:param-empir-dual} for sufficiently many steps with small enough step size, we can reach a solution that is close to solving the Model-Based Domain Generalization problem in Problem \ref{prob:model-based-domain-gen}.  In essence, the proof of this fact is a corollary of Theorem~\ref{thm:duality-gap} in conjunction with the recent literature concerning constrained PAC learning~\cite{chamon2021constrained} (see Appendix~\ref{sect:pacc}).

\begin{algorithm}[t]
   \caption{Model-Based Domain Generalization (MBDG)}
   \label{alg:mbst}
\begin{algorithmic}[1]
   \State {\bfseries Hyperparameters:} Primal step size $\eta_p > 0$,  dual step size $\eta_d \geq 0$, margin $\gamma > 0$
   \Repeat
   \For{minibatch $\{(x_j, y_j)\}_{j=1}^m$ in training dataset $\cup_{e\in\Etrain} \calD^e$}
   \State $\tilde{x}_j \gets \Call{GenerateImage}{x_j}$ $\forall j\in[m]$ \Comment{Generate model-based images}
   \State \text{distReg}$(\theta) \gets (1/m)\sum_{j=1}^m d(\varphi(\theta, x_j), \varphi(\theta, \tilde{x}_j))$ \Comment{Calculate distance regularizer}
   \State $\text{loss}(\theta) \gets (1/m) \sum_{j=1}^m  \ell\left(x_j, y_j; \varphi(\theta, \cdot)\right)$ \Comment{Calculate classification loss}
   \State $\theta \gets \theta - \eta_p \nabla_\theta [ \: \text{loss}(\theta) + \lambda \cdot  \text{distReg}(\theta) \: ]$ \Comment{Primal step for $\theta$}
   \State $\lambda \gets \left[ \lambda + \eta_d \left( \text{distReg}(\theta) - \gamma \right)\right]_+$ \Comment{Dual step for $\lambda$}
   \EndFor
   \Until{convergence} \\
\Procedure{GenerateImage}{$x^e$}
    \State $(x,e) \gets H(x^e)$ \Comment{Decompose $x^e$ into $x$ and $e$}
    \State Sample $e'\sim\mathcal{N}(0, I)$ \Comment{$e'$ is a latent code for MUNIT}
    \State \textbf{return} $G(x,e')$ \Comment{Return image produced by MUNIT}
\EndProcedure
\end{algorithmic}
\end{algorithm}

\subsection{Implementation of MBDG}  

In practice, we modify the primal-dual iteration in several ways to engender a more practical algorithmic scheme.  To begin, we remark that while our theory calls for data drawn from $\Prob(X,Y)$, in practice we only have access to finitely-many samples from $\Prob(X^e,Y^e)$ for $e\in\Etrain$.  However, note that the $G$-invariance condition implies that when \eqref{eq:param-empir-dual} is feasible, $\varphi(\theta, x) \approx \varphi(\theta, x^e)$ when $x^e\sim\Prob(X^e)$ and $x^e = G(x,e)$, where $x\sim\Prob(X)$.  Therefore, the data from $\cup_{e\in\Etrain} \calD^e$ is a useful proxy for data drawn from $\Prob(X,Y)$.   Furthermore, because (a) it may not be tractable to find a $\rho$-minimizer over $\calH$ at each iteration and (b) there may be a large number of domains in $\Etrain$, we propose two modifications of the primal-dual iteration in which we replace \eqref{eq:primal-step} with a stochastic gradient step and we use only one dual variable for all of the domains.  We call this algorithm MBDG; pseudocode is provided in Algorithm~\ref{alg:mbst}.

\paragraph{Walking through Algorithm~\ref{alg:mbst}.}  In Algorithm~\ref{alg:mbst}, we outline two main procedures.  In lines 12-15, we describe the \Call{GenerateImage}{$x^e$} procedure, which takes an image $x^e$ as input and returns an image that has been passed through a learned domain transformation model.  The MUNIT architecture uses a normally distributed latent code to vary the environment of a given image.  Thus, whenever \Call{GenerateImage}{} is called, an environmental latent code $e'\sim\mathcal{N}(0,I)$ is sampled and then passed through $G$ along with the disentangled input image.  

In lines 4-8 of Algorithm~\ref{alg:mbst}, we show the main training loop for MBDG.  In particular, after generating new images using the \Call{GenerateImage}{} procedure, we calculate the loss term $\text{loss}(\theta)$ and the regularization term $\text{distReg}(\theta)$, both of which are defined in the empirical, parameterized dual problem in~\eqref{eq:param-empir-dual}.  Note that we choose to enforce the constraints between $x^e = G(x,e)$ and $x^{e'} = G(x,e')$, so that $\text{distReg}(\theta) = (1/m)\sum_{j=1}^m d(\varphi(\theta, x^e), \varphi(\theta, x^{e'})$.  We emphasize that this is completely equivalent to enforcing the constraints between $x^e = G(x,e)$ and $x$, in which the regulizer would be $\text{distReg}(\theta) = (1/m)\sum_{j=1}^m d(\varphi(\theta, x^e), \varphi(\theta, x)$.  Next, in line 7, we perform the primal SGD step on $\theta$, and then in line 8, we perform the dual step on $\lambda$.  Throughout, we use the KL-divergence for the distance function $d$ in the $G$-invariance term $\text{distReg}(\theta)$.

\begin{figure}
    \centering
    \begin{subfigure}[b]{0.48\textwidth}
        \includegraphics[width=0.9\textwidth]{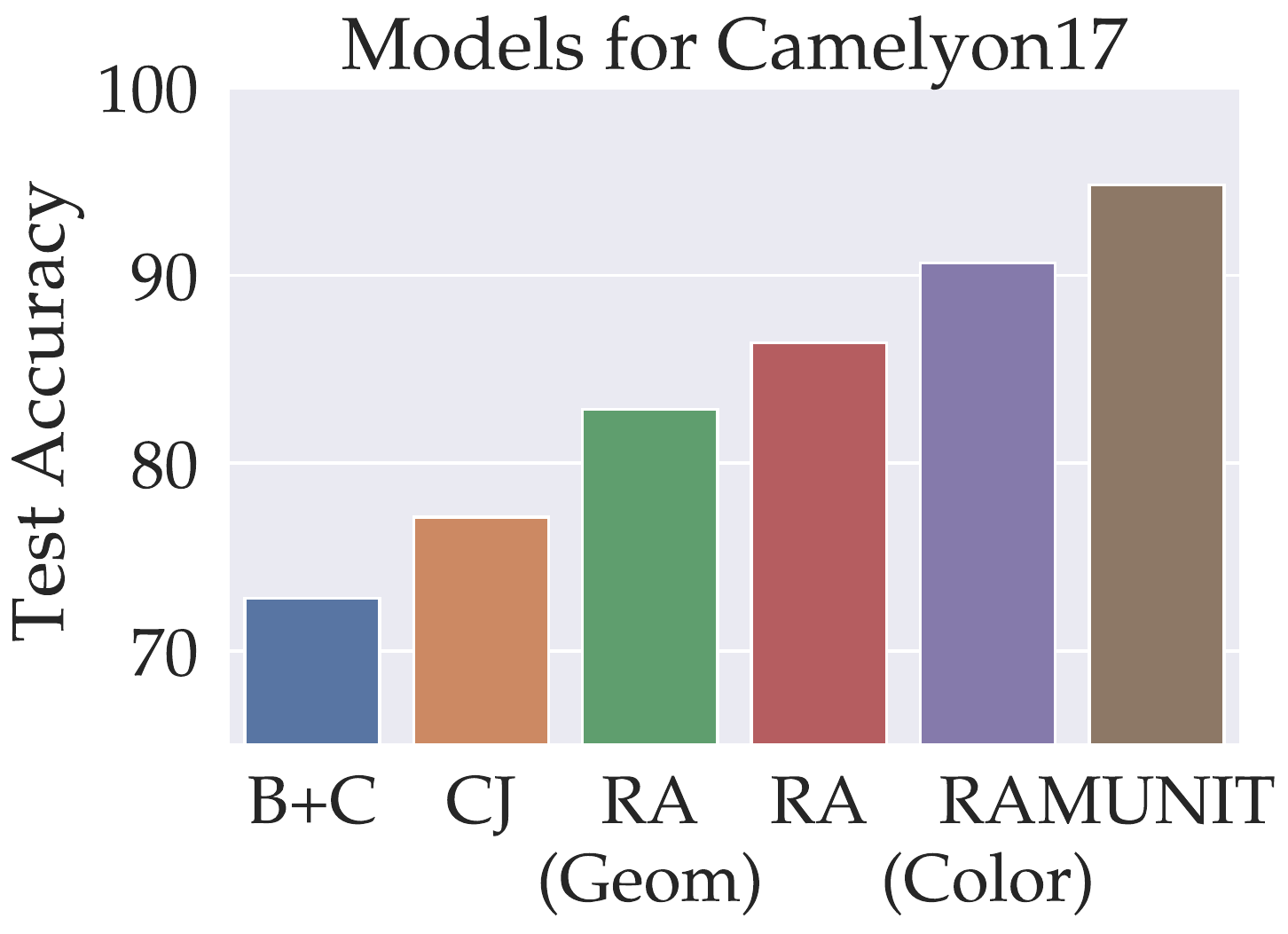}
        \captionof{figure}{\textbf{Known vs.\ learned models.}  We compare the performance of MBDG for known models (first five columns) against a model that was trained with the data from the training domains using MUNIT.}
        \label{fig:diff-models}
    \end{subfigure} \hfill
    \begin{subfigure}[b]{0.48\textwidth}
    \centering
    \includegraphics[width=0.87\textwidth]{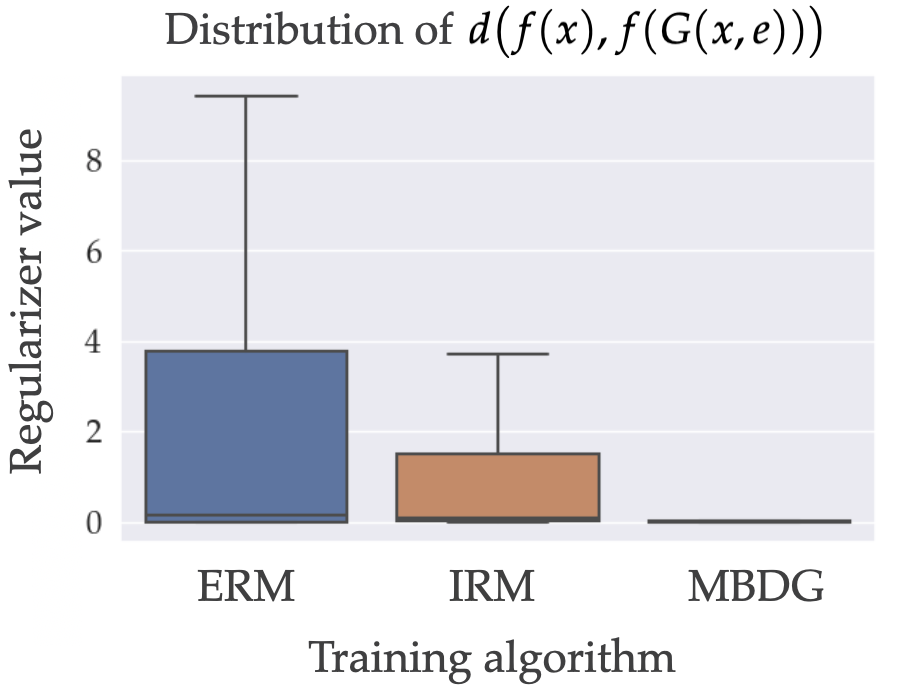}
    \captionof{figure}{\textbf{Measuring $G$-invariance.}  We show the distribution of $\text{distReg}(\theta)$ calculated in line~5 of Algorithm~\ref{alg:mbst} for classifiers trained using ERM, IRM, and MBDG on \texttt{Camelyon17-WILDS}.}
    \label{fig:measuring-invar}
    \end{subfigure}
    \caption{\textbf{Camelyon17-WILDS analysis.}  In (a) we show the benefit of learning $G$ from data as opposed to replacing $G$ with standard data-augmentation transformations; in (b) we measure $G$-invariance over the training data, showing that ERM and IRM are not invariant to $G$.}
\end{figure}

\section{Experiments}\label{sect:experiments}

We now evaluate the performance of MBDG on a range of standard domain generalization benchmarks.  In the main text, we present results for \texttt{ColoredMNIST}, \texttt{Camelyon17-WILDS}, \texttt{FMoW-WILDS}, and \texttt{PACS}; we defer results for \texttt{VLCS} to the supplemental.  For \texttt{ColoredMNIST}, \texttt{PACS}, and \texttt{VLCS}, we used the DomainBed\footnote{\url{https://github.com/facebookresearch/DomainBed}} package \cite{gulrajani2020search}, facilitating comparison to a range of baselines.  Model selection for each of these datasets was performed using hold-one-out cross-validation.  For \texttt{Camelyon17-WILDS} and \texttt{FMoW-WILDS}, we used the repository provided with the WILDS dataset suite\footnote{\url{https://github.com/p-lambda/wilds}}, and we performed model-selection using the out-of-distribution validation set provided in the WILDS repository.  Further details concerning hyperparameter tuning and model selection are deferred to Appendix~\ref{sect:further-exps}.

\subsection{Camelyon17-WILDS and FMoW-WILDS}

We first consider the \texttt{Camelyon17-WILDS} and \texttt{FMoW-WILDS} datasets from the WILDS family of domain generalization benchmarks \cite{koh2020wilds}.  \texttt{Camelyon17} contains roughly 400k $96\times96$ images of potentially cancerous cells taken at different hospitals, whereas \texttt{FMoW-WILDS} contains roughly 500k $224\times224$ images of aerial scenes characterized by different forms of land use.  Thus, both of these datasets are significantly larger than \texttt{ColoredMNIST} in both the number of images and the dimensionality of each image.  In Table~\ref{tab:wilds}, we report classification accuracies for MBDG and a range of baselines on both \texttt{Camelyon17-WILDS} and \texttt{FMOW-WILDS}.  Of particular interest is the fact that MBDG improves by more than 20 percentage points over the state-of-the-art baselines on \texttt{Camelyon17-WILDS}.  On \texttt{FMoW-WILDS}, we report a relatively modest improvement of around one percentage point.  

\newlength{\oldintextsep}
\setlength{\oldintextsep}{\intextsep}

\setlength\intextsep{-5pt}

\begin{wraptable}{r}{0.45\textwidth}
\begin{center}
\adjustbox{max width=\textwidth}{%
\scalebox{0.95}{
\begin{tabular}{lcc}
\toprule
\textbf{Algorithm}   & \textbf{Camelyon17} & \textbf{FMoW} \\
\midrule
ERM                  & 73.3 $\pm$ 9.9   &  51.3 (0.4)            \\
IRM                  & 60.9 $\pm$ 15.3  &  51.1 (0.4)            \\
ARM                  &  62.1 $\pm$ 6.4 &  47.9 (0.3)            \\
CORAL                  & 59.2 $\pm$ 15.1  &  49.6 (0.5)              \\
\midrule
MBDG                   & \textbf{94.8 $\pm$ 0.4}  &  \textbf{52.3 $\pm$ 0.5}            \\
\bottomrule
\end{tabular}}}
\end{center}
\caption{\textbf{WILDS accuracies.} We report classification accuracies for \texttt{Camelyon17-WILDS} and \texttt{FMoW-WILDS}.  For both datasets, we used the out-of-distribution validation set provided in the WILDS repository to perform model selection.}
\label{tab:wilds}
\end{wraptable}

In essence, the significant improvement we achieve on \texttt{Camelyon17-WILDS} is due to the ability of the learned model to vary the coloration and brightness in the images.  In the second row of Table~\ref{tab:model-based-samples}, observe that the input image is transformed so that it resembles images from the other domains shown in Figure~\ref{fig:domain-gen-outline}.  Thus, the ability of MBDG to enforce invariance to the changes captured by the learned domain transformation model is the key toward achieving strong domain generalization on this benchmark.  To further study the benefits of enforcing the $G$-invariance constraint, we consider two ablation studies on \texttt{Camelyon17-WILDS}.

\paragraph{Measuring the $G$-invariance of trained classifiers.}  In Section \ref{sect:mbdg}, we restricted our attention predictors satisfying the $G$-invariance condition.  To test whether our algorithm successfully enforces $G$-invariance when a domain transformation model $G$ is learned from data, we measure the distribution of distReg$(\theta)$ over all of the instances from the training domains of \texttt{Camelyon17-WILDS} for ERM, IRM, and MBDG.  In Figure~\ref{fig:measuring-invar}, observe that whereas MBDG is quite robust to changes under $G$, ERM and IRM are not nearly as robust.   This property is key to the ability of MBDG to learn invariant representations across domains.

\paragraph{Ablation on learning models vs.\ data augmentation.} 

As shown in Table \ref{tab:model-based-samples} and in Appendix \ref{sect:dtms}, accurate approximations of an underlying domain transformation model can often be learned from data drawn from the training domains.
However, rather than learning $G$ from data, a heuristic alternative is to replace the $\Call{GenerateImage}{}$ procedure in Algorithm \ref{alg:mbst} with standard data augmentation transformations.  In Figure \ref{fig:diff-models}, we investigate this approach with five different forms of data augmentation: B+C (brightness and contrast), CJ (color jitter), and three variants of RandAugment~\cite{cubuk2020randaugment} (RA, RA-Geom, and RA-Color).  More details regarding these data augmentation schemes are given in Appendix~\ref{sect:further-exps}.  The bars in Figure~\ref{fig:diff-models} show that although these schemes offer strong performance in our MBDG framework, the learned model trained using MUNIT offers the best OOD accuracy.

\subsection{ColoredMNIST}

\setlength{\columnsep}{20pt}%
\begin{wraptable}{r}{0.55\textwidth}
\centering
\scalebox{0.85}{
\begin{tabular}{lcccc}
\toprule
\textbf{Algorithm}   & \textbf{+90\%}       & \textbf{+80\%}       & \textbf{-90\%}       & \textbf{Avg}         \\
\midrule
ERM                  & 50.0 $\pm$ 0.2       & 50.1 $\pm$ 0.2       & 10.0 $\pm$ 0.0       & 36.7                 \\
IRM                  & 46.7 $\pm$ 2.4       & 51.2 $\pm$ 0.3       & 23.1 $\pm$ 10.7      & 40.3                 \\
GroupDRO             & 50.1 $\pm$ 0.5       & 50.0 $\pm$ 0.5       & 10.2 $\pm$ 0.1       & 36.8                 \\
Mixup                & 36.6 $\pm$ 10.9      & 53.4 $\pm$ 5.9       & 10.2 $\pm$ 0.1       & 33.4                 \\
MLDG                 & 50.1 $\pm$ 0.6       & 50.1 $\pm$ 0.3       & 10.0 $\pm$ 0.1       & 36.7                 \\
CORAL                & 49.5 $\pm$ 0.0       & 59.5 $\pm$ 8.2       & 10.2 $\pm$ 0.1       & 39.7                 \\
MMD                  & 50.3 $\pm$ 0.2       & 50.0 $\pm$ 0.4       & 9.9 $\pm$ 0.2        & 36.8                 \\
DANN                 & 49.9 $\pm$ 0.1       & 62.1 $\pm$ 7.0       & 10.0 $\pm$ 0.1       & 40.7                 \\
CDANN                & 63.2 $\pm$ 10.1      & 44.4 $\pm$ 4.5       & 9.9 $\pm$ 0.2        & 39.1                 \\
MTL                  & 44.3 $\pm$ 4.9       & 50.7 $\pm$ 0.0       & 10.1 $\pm$ 0.1       & 35.0                 \\
SagNet               & 49.9 $\pm$ 0.4       & 49.7 $\pm$ 0.3       & 10.0 $\pm$ 0.1       & 36.5                 \\
ARM                  & 50.0 $\pm$ 0.3       & 50.1 $\pm$ 0.3       & 10.2 $\pm$ 0.0       & 36.8                 \\
VREx                 & 50.2 $\pm$ 0.4       & 50.5 $\pm$ 0.5       & 10.1 $\pm$ 0.0       & 36.9                 \\
RSC                  & 49.6 $\pm$ 0.3       & 49.7 $\pm$ 0.4       & 10.1 $\pm$ 0.0       & 36.5                 \\
\midrule
MBDA                  & 72.0 $\pm$ 0.1       & 50.7 $\pm$ 0.1       & 22.5 $\pm$ 0.0      & 48.3                 \\
MBDG-DA                 & 72.7 $\pm$ 0.2       & 71.4 $\pm$ 0.1       & 33.2 $\pm$ 0.1      & 59.0                  \\

MBDG-Reg                  & 73.3 $\pm$ 0.0      & \textbf{73.7 $\pm$ 0.0}      & 27.2 $\pm$ 0.1     & 58.1              \\   
\midrule
MBDG                  & \textbf{73.7 $\pm$ 0.1}       & 68.4 $\pm$ 0.0       & \textbf{63.5 $\pm$ 0.0}      & \textbf{68.5}                 \\       
\bottomrule
\end{tabular}}
\caption{\textbf{ColoredMNIST accuracies.}  We report classification accuracies for \texttt{ColoredMNIST}.  Model-selection was performed via hold-one-out cross-validation.}
\label{tab:cmnist}
\end{wraptable}

\setlength\intextsep{-15pt}

We next consider the \texttt{ColoredMNIST} dataset \cite{arjovsky2019invariant}, which is a standard domain generalization benchmark created by colorizing subsets of the MNIST dataset~\cite{lecun2010mnist}.  This dataset contains three domains, each of which is characterized by a different level of correlation between the label and digit color.  The domains are constructed so that the colors are more strongly correlated with the labels than with the digits.  Thus, as was argued in~\cite{arjovsky2019invariant}, stronger domain generalization on \texttt{ColoredMNIST} can be obtained by eliminating color as a predictive feature.  

As shown in Table \ref{tab:cmnist}, despite the fact that the data generating procedure used to construct this dataset does not fulfill Assumptions~\ref{assume:gen-model} and~\ref{assume:cov-shift} (see Figure~\ref{fig:spur-corr-causal}), the MBDG algorithm still improves over each baseline by nearly thirty percentage points.  Indeed, due to way the \texttt{ColoredMNIST} dataset is constructed, the best possible result is an accuracy of 75\%.  Thus, the fact that MBDG achieves 68.5\% accuracy when averaged over the domains means that it is close to achieving perfect domain generalization.  
% In Appendix \textcolor{red}{X}, we discuss the relationship between the 

\begin{figure}[t]
    \centering
    \begin{subfigure}[b]{0.40\textwidth}
        \centering
        \includegraphics[width=0.9\textwidth]{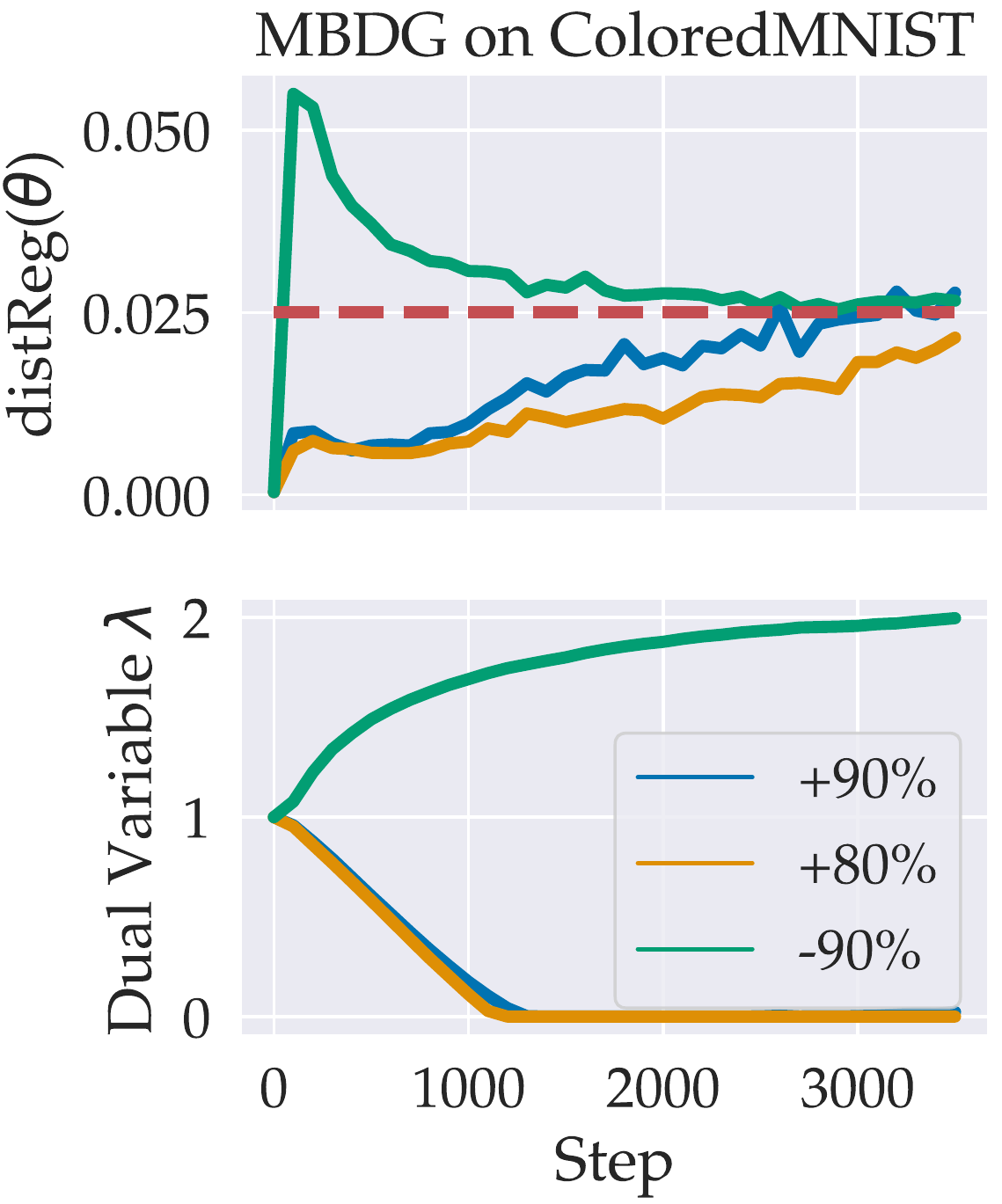}
        \caption{\textbf{Tracking the dual variables.}  We show the value of distReg$(\theta)$ and the dual variables $\lambda$ for each MBDG classifier in Table~\ref{tab:cmnist}.  The margin $\gamma = 0.025$ is shown in red.}
        \label{fig:cmnist-dual-var}
    \end{subfigure}\hfill
    \begin{subfigure}[b]{0.58\textwidth}
        \centering
        \includegraphics[width=0.9\textwidth]{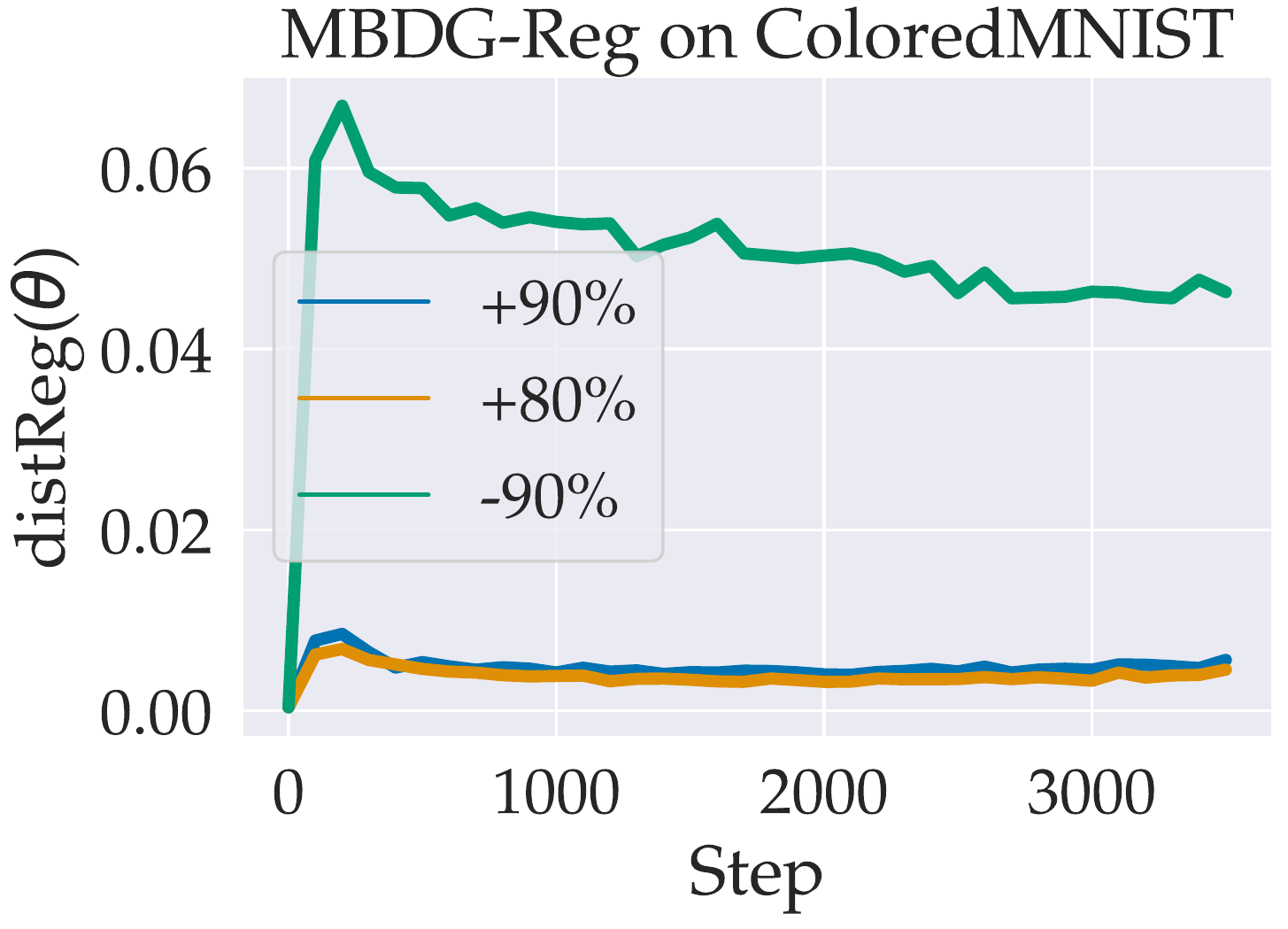}
        \caption{\textbf{Regularized MBDG.}  We show the value of the $\text{distReg}(\theta)$ term for each domain in \texttt{ColoredMNIST} for fixed dual variables $\lambda$.  This corresponds to the MBDG-Reg row in Table~\ref{tab:cmnist}.  Note that the +90\% constraint never reaches the margin $\gamma = 0.025$.}
      \label{fig:reg}
    \end{subfigure}
    \caption{\textbf{Primal-dual ascent vs.\ regularization on ColoredMNIST.}  We compare the constraint satisfaction of (a) the primal-dual ascent method described in Algorithm~\ref{alg:mbst} and (b) a regularized version of MBDG.  Notice that whereas the dual variable update step used in (a) pushes the value of $\text{distReg}(\theta)$ for the -90\% test domain (shown in green) down to the margin of $\gamma = 0.025$, the regularized version shown in (b) does not enforce constraint satisfaction.}
\end{figure}

To understand the reasons behind this improvement, consider the first row of Table~\ref{tab:model-based-samples}.  Notice that whereas the input image shows a red `5', samples from the learned domain transformation model show the same `5' colored green.  Thus, the $G$-invariance constraint calculated in line 5 of Algorithm~\ref{alg:mbst} forces the classifier $f$ to predict the same label for both the red `5' and the green `5'.  Therefore, in essence the $G$-invariance constraint explicitly eliminates color as a predictive feature, resulting in the strong performance shown in Table~\ref{tab:cmnist}.  To further evaluate the MBDG algorithm and its performance on \texttt{ColoredMNIST}, we consider three ablation studies.

\paragraph{Tracking the dual variables.}  For the three MBDG classifiers selected by cross-validation at the bottom of Table~\ref{tab:cmnist}, we plot the constraint term distReg$(\theta)$ and the corresponding dual variable at each training step in Figure~\ref{fig:cmnist-dual-var}.  Observe that for the +90\% and +80\% domains, the dual variables decay to zero, as the constraint is satisfied early on in training.  On the other hand, the constraint for the -90\% domain is not satisfied early on in training, and in response, the dual variable increases, gradually forcing constraint satisfaction.  As we show in the next subsection, without the dual update step, the constraints may never be satisfied (see Figure \ref{fig:reg}).  This underscores the message of Theorem~\ref{thm:primal-dual}, which is that the primal dual method can be used to enforce constraint satisfaction for Problem~\ref{prob:model-based-domain-gen}, resulting in stronger invariance across domains.

\paragraph{Regularization vs.\ dual asce gnt.}  A common trick for encouraging constraint satisfaction in deep learning is to introduce soft constraints by adding a regularizer multiplied by a fixed penalty weight to the objective.  While this approach yields a related problem to \eqref{eq:param-empir-dual} where the dual variables are fixed (see Appendix \ref{sect:reg-vs-primal-dual}), there are few formal guarantees for this approach and tuning the penalty weight can require expert or domain-specific knowledge.  

In Table~\ref{tab:cmnist}, we show the performance of a regularized version of MBDG (MBDG-Reg in Table~\ref{tab:cmnist}) where the dual variable is fixed during training (see Appendix~\ref{sect:reg-variant} for pseudocode).  Note that while the performance of MBDG-Reg improves significantly over the baselines, it lags more than ten percentage points behind MBDG.  Furthermore, consider that relative to Figure \ref{fig:cmnist-dual-var}, the value of distReg($\theta$) shown in~\ref{fig:reg} is much larger than the margin of $\gamma=0.025$ used in Figure~\ref{fig:cmnist-dual-var}, meaning that the constraint is not being satisfied when running MBDG-Reg.  Therefore, while regularization offers a heuristic alternative to MBDG, the primal-dual approach offers both stronger guarantees as well as superior performance.

\paragraph{Ablation on data augmentation.}  To study the efficacy of the primal-dual approach taken by the MBDG algorithm toward improving the OOD accuracy on the test domain, we consider two natural alternatives MBDG: (1) ERM with data augmentation through the learned model $G(x,e)$ (MBDA); and (2) MBDG with data augmentation through $G(x,e)$ on the training objective (MBDG-DA).  We provide psuedocode and further discussion of both of these methods in Appendix~\ref{sect:data-aug-algs}.  As shown at the bottom of Table \ref{tab:cmnist}, while these variants significantly outperform the baselines, they not perform nearly as well as MBDG.  Thus, while data augmentation can in some cases improve performance, the primal-dual iteration is a much more effective tool for enforcing invariance across domains.

\subsection{PACS}

\begin{table}
\centering
\adjustbox{max width=\textwidth}{%
\scalebox{0.9}{
\begin{tabular}{lccccc}
\toprule
\textbf{Algorithm}   & \textbf{A}           & \textbf{C}           & \textbf{P}           & \textbf{S}           & \textbf{Avg}         \\
\midrule
ERM                  & 83.2 $\pm$ 1.3       & 76.8 $\pm$ 1.7       & \textbf{97.2 $\pm$ 0.3}       & 74.8 $\pm$ 1.3       & 83.0                 \\
IRM                  & 81.7 $\pm$ 2.4       & 77.0 $\pm$ 1.3       & 96.3 $\pm$ 0.2       & 71.1 $\pm$ 2.2       & 81.5                 \\
GroupDRO             & 84.4 $\pm$ 0.7       & 77.3 $\pm$ 0.8       & 96.8 $\pm$ 0.8       & 75.6 $\pm$ 1.4       & 83.5                 \\
Mixup                & 85.2 $\pm$ 1.9       & 77.0 $\pm$ 1.7       & 96.8 $\pm$ 0.8       & 73.9 $\pm$ 1.6       & 83.2                 \\
MLDG                 & 81.4 $\pm$ 3.6       & 77.9 $\pm$ 2.3       & 96.2 $\pm$ 0.3       & 76.1 $\pm$ 2.1       & 82.9                 \\
CORAL                & 80.5 $\pm$ 2.8       & 74.5 $\pm$ 0.4       & 96.8 $\pm$ 0.3       & 78.6 $\pm$ 1.4       & 82.6                 \\
MMD                  & 84.9 $\pm$ 1.7       & 75.1 $\pm$ 2.0       & 96.1 $\pm$ 0.9       & 76.5 $\pm$ 1.5       & 83.2                 \\
DANN                 & 84.3 $\pm$ 2.8       & 72.4 $\pm$ 2.8       & 96.5 $\pm$ 0.8       & 70.8 $\pm$ 1.3       & 81.0                 \\
CDANN                & 78.3 $\pm$ 2.8       & 73.8 $\pm$ 1.6       & 96.4 $\pm$ 0.5       & 66.8 $\pm$ 5.5       & 78.8                 \\
MTL                  & \textbf{85.6 $\pm$ 1.5}       & 78.9 $\pm$ 0.6       & 97.1 $\pm$ 0.3       & 73.1 $\pm$ 2.7       & 83.7                 \\
SagNet               & 81.1 $\pm$ 1.9       & 75.4 $\pm$ 1.3       & 95.7 $\pm$ 0.9       & 77.2 $\pm$ 0.6       & 82.3                 \\
ARM                  & 85.9 $\pm$ 0.3       & 73.3 $\pm$ 1.9       & 95.6 $\pm$ 0.4       & 72.1 $\pm$ 2.4       & 81.7                 \\
VREx                 & 81.6 $\pm$ 4.0       & 74.1 $\pm$ 0.3       & 96.9 $\pm$ 0.4       & 72.8 $\pm$ 2.1       & 81.3                 \\
RSC                  & 83.7 $\pm$ 1.7       & \textbf{82.9 $\pm$ 1.1}       & 95.6 $\pm$ 0.7       & 68.1 $\pm$ 1.5       & 82.6                 \\
\midrule
MBDG                 & 80.6 $\pm$ 1.1       & 79.3 $\pm$ 0.2        & 97.0 $\pm$ 0.4       & \textbf{85.2 $\pm$ 0.2}        &  \textbf{85.6}                 \\
\bottomrule
\end{tabular}}}
\caption{\textbf{PACS.}  We report classification accuracies for \texttt{PACS}.  Model-selection was performed via hold-one-out cross-validation.}
\label{tab:pacs}
\end{table}

In this subsection, we provide results for the standard \texttt{PACS} benchmark.  This dataset contains four domains of $224\times224$ images; the domains are ``art/paining'' (A), ``cartoon'' (C), ``photo'' (P), and ``sketch'' (S).  In the fourth row of Table~\ref{tab:model-based-samples}, we show several samples for one of the domain transformation models used for the PACS dataset.  Further, Table~\ref{tab:pacs} shows that MBDG achieves 85.6\% classification accuracy (averaged across the domains), which is the best known result for \texttt{PACS}.  In particular, this result is nearly two percentage points higher than any of the baselines, which represents a significant advancement in the state-of-the-art for this benchmark.  In large part, this result is due to significant improvements on the ``Sketch'' (S) subset, wherein MBDG improves by nearly seven percentage points over all other baselines.
\section{Conclusion}

In this paper, we introduced a new framework for domain generalization called Model-Based Domain Generalization.  In this framework, we showed that under a natural model of data generation and a concomitant notion of invariance, the classical domain generalization problem is equivalent to a semi-infinite constrained statistical learning problem.  We then provide a theoretical, duality based perspective on problem, which results in a novel primal-dual style algorithm that improves by up to 30 percentage points over state-of-the-art baselines.

\newpage

\bibliography{bibliography}

\begin{thebibliography}{100}

\bibitem{lecun2015deep}
Yann LeCun, Yoshua Bengio, and Geoffrey Hinton.
\newblock Deep learning.
\newblock {\em nature}, 521(7553):436--444, 2015.

\bibitem{esteves2017polar}
Carlos Esteves, Christine Allen-Blanchette, Xiaowei Zhou, and Kostas
  Daniilidis.
\newblock Polar transformer networks.
\newblock {\em arXiv preprint arXiv:1709.01889}, 2017.

\bibitem{esteves2018learning}
Carlos Esteves, Christine Allen-Blanchette, Ameesh Makadia, and Kostas
  Daniilidis.
\newblock Learning so (3) equivariant representations with spherical cnns.
\newblock In {\em Proceedings of the European Conference on Computer Vision
  (ECCV)}, pages 52--68, 2018.

\bibitem{jaderberg2015spatial}
Max Jaderberg, Karen Simonyan, Andrew Zisserman, and Koray Kavukcuoglu.
\newblock Spatial transformer networks.
\newblock {\em arXiv preprint arXiv:1506.02025}, 2015.

\bibitem{hendrycks2019benchmarking}
Dan Hendrycks and Thomas Dietterich.
\newblock Benchmarking neural network robustness to common corruptions and
  perturbations.
\newblock {\em arXiv preprint arXiv:1903.12261}, 2019.

\bibitem{djolonga2020robustness}
Josip Djolonga, Jessica Yung, Michael Tschannen, Rob Romijnders, Lucas Beyer,
  Alexander Kolesnikov, Joan Puigcerver, Matthias Minderer, Alexander D'Amour,
  Dan Moldovan, et~al.
\newblock On robustness and transferability of convolutional neural networks.
\newblock {\em arXiv preprint arXiv:2007.08558}, 2020.

\bibitem{taori2020measuring}
Rohan Taori, Achal Dave, Vaishaal Shankar, Nicholas Carlini, Benjamin Recht,
  and Ludwig Schmidt.
\newblock Measuring robustness to natural distribution shifts in image
  classification.
\newblock {\em Advances in Neural Information Processing Systems}, 33, 2020.

\bibitem{hendrycks2020many}
Dan Hendrycks, Steven Basart, Norman Mu, Saurav Kadavath, Frank Wang, Evan
  Dorundo, Rahul Desai, Tyler Zhu, Samyak Parajuli, Mike Guo, et~al.
\newblock The many faces of robustness: A critical analysis of
  out-of-distribution generalization.
\newblock {\em arXiv preprint arXiv:2006.16241}, 2020.

\bibitem{torralba2011unbiased}
Antonio Torralba and Alexei~A Efros.
\newblock Unbiased look at dataset bias.
\newblock In {\em CVPR 2011}, pages 1521--1528. IEEE, 2011.

\bibitem{arjovsky2019invariant}
Martin Arjovsky, L{\'e}on Bottou, Ishaan Gulrajani, and David Lopez-Paz.
\newblock Invariant risk minimization.
\newblock {\em arXiv preprint arXiv:1907.02893}, 2019.

\bibitem{ahuja2020invariant}
Kartik Ahuja, Karthikeyan Shanmugam, Kush Varshney, and Amit Dhurandhar.
\newblock Invariant risk minimization games.
\newblock In {\em International Conference on Machine Learning}, pages
  145--155. PMLR, 2020.

\bibitem{lu2021nonlinear}
Chaochao Lu, Yuhuai Wu, Jo{\'s}e~Miguel Hern{\'a}ndez-Lobato, and Bernhard
  Sch{\"o}lkopf.
\newblock Nonlinear invariant risk minimization: A causal approach.
\newblock {\em arXiv preprint arXiv:2102.12353}, 2021.

\bibitem{biggio2013evasion}
Battista Biggio, Igino Corona, Davide Maiorca, Blaine Nelson, Nedim
  {\v{S}}rndi{\'c}, Pavel Laskov, Giorgio Giacinto, and Fabio Roli.
\newblock Evasion attacks against machine learning at test time.
\newblock In {\em Joint European conference on machine learning and knowledge
  discovery in databases}, pages 387--402. Springer, 2013.

\bibitem{goodfellow2014explaining}
Ian~J Goodfellow, Jonathon Shlens, and Christian Szegedy.
\newblock Explaining and harnessing adversarial examples.
\newblock {\em arXiv preprint arXiv:1412.6572}, 2014.

\bibitem{madry2017towards}
Aleksander Madry, Aleksandar Makelov, Ludwig Schmidt, Dimitris Tsipras, and
  Adrian Vladu.
\newblock Towards deep learning models resistant to adversarial attacks.
\newblock {\em arXiv preprint arXiv:1706.06083}, 2017.

\bibitem{wong2017provable}
Eric Wong and J~Zico Kolter.
\newblock Provable {D}efenses {A}gainst {A}dversarial {E}xamples {V}ia the
  {C}onvex {O}uter {A}dversarial {P}olytope.
\newblock {\em arXiv preprint arXiv:1711.00851}, 2017.

\bibitem{dobriban2020provable}
Edgar Dobriban, Hamed Hassani, David Hong, and Alexander Robey.
\newblock Provable tradeoffs in adversarially robust classification.
\newblock {\em arXiv preprint arXiv:2006.05161}, 2020.

\bibitem{santurkar2020breeds}
Shibani Santurkar, Dimitris Tsipras, and Aleksander Madry.
\newblock Breeds: Benchmarks for subpopulation shift.
\newblock {\em arXiv preprint arXiv:2008.04859}, 2020.

\bibitem{sohoni2020no}
Nimit~S Sohoni, Jared~A Dunnmon, Geoffrey Angus, Albert Gu, and Christopher
  R{\'e}.
\newblock No subclass left behind: Fine-grained robustness in coarse-grained
  classification problems.
\newblock {\em arXiv preprint arXiv:2011.12945}, 2020.

\bibitem{koh2020wilds}
Pang~Wei Koh, Shiori Sagawa, Henrik Marklund, Sang~Michael Xie, Marvin Zhang,
  Akshay Balsubramani, Weihua Hu, Michihiro Yasunaga, Richard~Lanas Phillips,
  Sara Beery, et~al.
\newblock Wilds: A benchmark of in-the-wild distribution shifts.
\newblock {\em arXiv preprint arXiv:2012.07421}, 2020.

\bibitem{xiao2020noise}
Kai Xiao, Logan Engstrom, Andrew Ilyas, and Aleksander Madry.
\newblock Noise or signal: The role of image backgrounds in object recognition.
\newblock {\em arXiv preprint arXiv:2006.09994}, 2020.

\bibitem{robey2020model}
Alexander Robey, Hamed Hassani, and George~J Pappas.
\newblock Model-based robust deep learning.
\newblock {\em arXiv preprint arXiv:2005.10247}, 2020.

\bibitem{wong2020learning}
Eric Wong and J~Zico Kolter.
\newblock Learning perturbation sets for robust machine learning.
\newblock {\em arXiv preprint arXiv:2007.08450}, 2020.

\bibitem{gowal2020achieving}
Sven Gowal, Chongli Qin, Po-Sen Huang, Taylan Cemgil, Krishnamurthy Dvijotham,
  Timothy Mann, and Pushmeet Kohli.
\newblock Achieving robustness in the wild via adversarial mixing with
  disentangled representations.
\newblock In {\em Proceedings of the IEEE/CVF Conference on Computer Vision and
  Pattern Recognition}, pages 1211--1220, 2020.

\bibitem{laidlaw2020perceptual}
Cassidy Laidlaw, Sahil Singla, and Soheil Feizi.
\newblock Perceptual adversarial robustness: Defense against unseen threat
  models.
\newblock {\em arXiv preprint arXiv:2006.12655}, 2020.

\bibitem{esteva2019guide}
Andre Esteva, Alexandre Robicquet, Bharath Ramsundar, Volodymyr Kuleshov, Mark
  DePristo, Katherine Chou, Claire Cui, Greg Corrado, Sebastian Thrun, and Jeff
  Dean.
\newblock A guide to deep learning in healthcare.
\newblock {\em Nature medicine}, 25(1):24--29, 2019.

\bibitem{yao2019strong}
Li~Yao, Jordan Prosky, Ben Covington, and Kevin Lyman.
\newblock A strong baseline for domain adaptation and generalization in medical
  imaging.
\newblock {\em arXiv preprint arXiv:1904.01638}, 2019.

\bibitem{li2020domain}
Haoliang Li, YuFei Wang, Renjie Wan, Shiqi Wang, Tie-Qiang Li, and Alex~C Kot.
\newblock Domain generalization for medical imaging classification with
  linear-dependency regularization.
\newblock {\em arXiv preprint arXiv:2009.12829}, 2020.

\bibitem{bashyam2020medical}
Vishnu~M Bashyam, Jimit Doshi, Guray Erus, Dhivya Srinivasan, Ahmed Abdulkadir,
  Mohamad Habes, Yong Fan, Colin~L Masters, Paul Maruff, Chuanjun Zhuo, et~al.
\newblock Medical image harmonization using deep learning based canonical
  mapping: Toward robust and generalizable learning in imaging.
\newblock {\em arXiv preprint arXiv:2010.05355}, 2020.

\bibitem{zhang2020learning}
Amy Zhang, Rowan McAllister, Roberto Calandra, Yarin Gal, and Sergey Levine.
\newblock Learning invariant representations for reinforcement learning without
  reconstruction.
\newblock {\em arXiv preprint arXiv:2006.10742}, 2020.

\bibitem{yang2018real}
Luona Yang, Xiaodan Liang, Tairui Wang, and Eric Xing.
\newblock Real-to-virtual domain unification for end-to-end autonomous driving.
\newblock In {\em Proceedings of the European conference on computer vision
  (ECCV)}, pages 530--545, 2018.

\bibitem{zhang2017curriculum}
Yang Zhang, Philip David, and Boqing Gong.
\newblock Curriculum domain adaptation for semantic segmentation of urban
  scenes.
\newblock In {\em Proceedings of the IEEE International Conference on Computer
  Vision}, pages 2020--2030, 2017.

\bibitem{julian2020never}
Ryan Julian, Benjamin Swanson, Gaurav~S Sukhatme, Sergey Levine, Chelsea Finn,
  and Karol Hausman.
\newblock Never stop learning: The effectiveness of fine-tuning in robotic
  reinforcement learning.
\newblock {\em arXiv e-prints}, pages arXiv--2004, 2020.

\bibitem{sonar2020invariant}
Anoopkumar Sonar, Vincent Pacelli, and Anirudha Majumdar.
\newblock Invariant policy optimization: Towards stronger generalization in
  reinforcement learning.
\newblock {\em arXiv preprint arXiv:2006.01096}, 2020.

\bibitem{vinitsky2020robust}
Eugene Vinitsky, Yuqing Du, Kanaad Parvate, Kathy Jang, Pieter Abbeel, and
  Alexandre Bayen.
\newblock Robust reinforcement learning using adversarial populations.
\newblock {\em arXiv preprint arXiv:2008.01825}, 2020.

\bibitem{ribeiro2016should}
Marco~Tulio Ribeiro, Sameer Singh, and Carlos Guestrin.
\newblock " why should i trust you?" explaining the predictions of any
  classifier.
\newblock In {\em Proceedings of the 22nd ACM SIGKDD international conference
  on knowledge discovery and data mining}, pages 1135--1144, 2016.

\bibitem{biggio2018wild}
Battista Biggio and Fabio Roli.
\newblock Wild patterns: Ten years after the rise of adversarial machine
  learning.
\newblock {\em Pattern Recognition}, 84:317--331, 2018.

\bibitem{blanchard2011generalizing}
Gilles Blanchard, Gyemin Lee, and Clayton Scott.
\newblock Generalizing from several related classification tasks to a new
  unlabeled sample.
\newblock {\em Advances in neural information processing systems},
  24:2178--2186, 2011.

\bibitem{muandet2013domain}
Krikamol Muandet, David Balduzzi, and Bernhard Sch{\"o}lkopf.
\newblock Domain generalization via invariant feature representation.
\newblock In {\em International Conference on Machine Learning}, pages 10--18,
  2013.

\bibitem{blanchard2017domain}
Gilles Blanchard, Aniket~Anand Deshmukh, Urun Dogan, Gyemin Lee, and Clayton
  Scott.
\newblock Domain generalization by marginal transfer learning.
\newblock {\em arXiv preprint arXiv:1711.07910}, 2017.

\bibitem{huang2020self}
Zeyi Huang, Haohan Wang, Eric~P Xing, and Dong Huang.
\newblock Self-challenging improves cross-domain generalization.
\newblock {\em arXiv preprint arXiv:2007.02454}, 2, 2020.

\bibitem{sun2016deep}
Baochen Sun and Kate Saenko.
\newblock Deep coral: Correlation alignment for deep domain adaptation.
\newblock In {\em European conference on computer vision}, pages 443--450.
  Springer, 2016.

\bibitem{li2018learning}
Da~Li, Yongxin Yang, Yi-Zhe Song, and Timothy Hospedales.
\newblock Learning to generalize: Meta-learning for domain generalization.
\newblock In {\em Proceedings of the AAAI Conference on Artificial
  Intelligence}, volume~32, 2018.

\bibitem{vapnik1999overview}
Vladimir~N Vapnik.
\newblock An overview of statistical learning theory.
\newblock {\em IEEE transactions on neural networks}, 10(5):988--999, 1999.

\bibitem{he2016deep}
Kaiming He, Xiangyu Zhang, Shaoqing Ren, and Jian Sun.
\newblock Deep residual learning for image recognition.
\newblock In {\em Proceedings of the IEEE conference on computer vision and
  pattern recognition}, pages 770--778, 2016.

\bibitem{huang2017densely}
Gao Huang, Zhuang Liu, Laurens Van Der~Maaten, and Kilian~Q Weinberger.
\newblock Densely connected convolutional networks.
\newblock In {\em Proceedings of the IEEE conference on computer vision and
  pattern recognition}, pages 4700--4708, 2017.

\bibitem{gulrajani2020search}
Ishaan Gulrajani and David Lopez-Paz.
\newblock In search of lost domain generalization.
\newblock {\em arXiv preprint arXiv:2007.01434}, 2020.

\bibitem{zhou2021domain}
Kaiyang Zhou, Ziwei Liu, Yu~Qiao, Tao Xiang, and Chen~Change Loy.
\newblock Domain generalization: A survey.
\newblock {\em arXiv preprint arXiv:2103.02503}, 2021.

\bibitem{wang2021generalizing}
Jindong Wang, Cuiling Lan, Chang Liu, Yidong Ouyang, and Tao Qin.
\newblock Generalizing to unseen domains: A survey on domain generalization.
\newblock {\em arXiv preprint arXiv:2103.03097}, 2021.

\bibitem{shi2021gradient}
Yuge Shi, Jeffrey Seely, Philip~HS Torr, N~Siddharth, Awni Hannun, Nicolas
  Usunier, and Gabriel Synnaeve.
\newblock Gradient matching for domain generalization.
\newblock {\em arXiv preprint arXiv:2104.09937}, 2021.

\bibitem{bellot2020accounting}
Alexis Bellot and Mihaela van~der Schaar.
\newblock Accounting for unobserved confounding in domain generalization.
\newblock {\em arXiv preprint arXiv:2007.10653}, 2020.

\bibitem{ganin2016domain}
Yaroslav Ganin, Evgeniya Ustinova, Hana Ajakan, Pascal Germain, Hugo
  Larochelle, Fran{\c{c}}ois Laviolette, Mario Marchand, and Victor Lempitsky.
\newblock Domain-adversarial training of neural networks.
\newblock {\em The journal of machine learning research}, 17(1):2096--2030,
  2016.

\bibitem{albuquerque2019adversarial}
Isabela Albuquerque, João Monteiro, Mohammad Darvishi, Tiago~H. Falk, and
  Ioannis Mitliagkas.
\newblock Adversarial target-invariant representation learning for domain
  generalization.
\newblock {\em arXiv preprint arXiv:1911.00804}, 2019.

\bibitem{li2018domain}
Haoliang Li, Sinno~Jialin Pan, Shiqi Wang, and Alex~C Kot.
\newblock Domain generalization with adversarial feature learning.
\newblock In {\em Proceedings of the IEEE Conference on Computer Vision and
  Pattern Recognition}, pages 5400--5409, 2018.

\bibitem{motiian2017unified}
Saeid Motiian, Marco Piccirilli, Donald~A Adjeroh, and Gianfranco Doretto.
\newblock Unified deep supervised domain adaptation and generalization.
\newblock In {\em Proceedings of the IEEE international conference on computer
  vision}, pages 5715--5725, 2017.

\bibitem{ghifary2016scatter}
Muhammad Ghifary, David Balduzzi, W~Bastiaan Kleijn, and Mengjie Zhang.
\newblock Scatter component analysis: A unified framework for domain adaptation
  and domain generalization.
\newblock {\em IEEE transactions on pattern analysis and machine intelligence},
  39(7):1414--1430, 2016.

\bibitem{hu2020domain}
Shoubo Hu, Kun Zhang, Zhitang Chen, and Laiwan Chan.
\newblock Domain generalization via multidomain discriminant analysis.
\newblock In {\em Uncertainty in Artificial Intelligence}, pages 292--302.
  PMLR, 2020.

\bibitem{ilse2020diva}
Maximilian Ilse, Jakub~M Tomczak, Christos Louizos, and Max Welling.
\newblock Diva: Domain invariant variational autoencoders.
\newblock In {\em Medical Imaging with Deep Learning}, pages 322--348. PMLR,
  2020.

\bibitem{akuzawa2019adversarial}
Kei Akuzawa, Yusuke Iwasawa, and Yutaka Matsuo.
\newblock Adversarial invariant feature learning with accuracy constraint for
  domain generalization.
\newblock In {\em Joint European Conference on Machine Learning and Knowledge
  Discovery in Databases}, pages 315--331. Springer, 2019.

\bibitem{chattopadhyay2020learning}
Prithvijit Chattopadhyay, Yogesh Balaji, and Judy Hoffman.
\newblock Learning to balance specificity and invariance for in and out of
  domain generalization.
\newblock In {\em European Conference on Computer Vision}, pages 301--318.
  Springer, 2020.

\bibitem{piratla2020efficient}
Vihari Piratla, Praneeth Netrapalli, and Sunita Sarawagi.
\newblock Efficient domain generalization via common-specific low-rank
  decomposition.
\newblock In {\em International Conference on Machine Learning}, pages
  7728--7738. PMLR, 2020.

\bibitem{shankar2018generalizing}
Shiv Shankar, Vihari Piratla, Soumen Chakrabarti, Siddhartha Chaudhuri, Preethi
  Jyothi, and Sunita Sarawagi.
\newblock Generalizing across domains via cross-gradient training.
\newblock {\em arXiv preprint arXiv:1804.10745}, 2018.

\bibitem{li2018deep}
Ya~Li, Xinmei Tian, Mingming Gong, Yajing Liu, Tongliang Liu, Kun Zhang, and
  Dacheng Tao.
\newblock Deep domain generalization via conditional invariant adversarial
  networks.
\newblock In {\em Proceedings of the European Conference on Computer Vision
  (ECCV)}, pages 624--639, 2018.

\bibitem{ben2007analysis}
Shai Ben-David, John Blitzer, Koby Crammer, Fernando Pereira, et~al.
\newblock Analysis of representations for domain adaptation.
\newblock {\em Advances in neural information processing systems}, 19:137,
  2007.

\bibitem{daume2009frustratingly}
Hal Daum{\'e}~III.
\newblock Frustratingly easy domain adaptation.
\newblock {\em arXiv preprint arXiv:0907.1815}, 2009.

\bibitem{pan2010domain}
Sinno~Jialin Pan, Ivor~W Tsang, James~T Kwok, and Qiang Yang.
\newblock Domain adaptation via transfer component analysis.
\newblock {\em IEEE Transactions on Neural Networks}, 22(2):199--210, 2010.

\bibitem{tzeng2017adversarial}
Eric Tzeng, Judy Hoffman, Kate Saenko, and Trevor Darrell.
\newblock Adversarial discriminative domain adaptation.
\newblock In {\em Proceedings of the IEEE conference on computer vision and
  pattern recognition}, pages 7167--7176, 2017.

\bibitem{fu2020learning}
Bo~Fu, Zhangjie Cao, Mingsheng Long, and Jianmin Wang.
\newblock Learning to detect open classes for universal domain adaptation.
\newblock In {\em European Conference on Computer Vision}, pages 567--583.
  Springer, 2020.

\bibitem{patel2015visual}
Vishal~M Patel, Raghuraman Gopalan, Ruonan Li, and Rama Chellappa.
\newblock Visual domain adaptation: A survey of recent advances.
\newblock {\em IEEE signal processing magazine}, 32(3):53--69, 2015.

\bibitem{csurka2017domain}
Gabriela Csurka.
\newblock Domain adaptation for visual applications: A comprehensive survey.
\newblock {\em arXiv preprint arXiv:1702.05374}, 2017.

\bibitem{wang2018deep}
Mei Wang and Weihong Deng.
\newblock Deep visual domain adaptation: A survey.
\newblock {\em Neurocomputing}, 312:135--153, 2018.

\bibitem{dubey2021adaptive}
Abhimanyu Dubey, Vignesh Ramanathan, Alex Pentland, and Dhruv Mahajan.
\newblock Adaptive methods for real-world domain generalization.
\newblock {\em arXiv preprint arXiv:2103.15796}, 2021.

\bibitem{deshmukh2019generalization}
Aniket~Anand Deshmukh, Yunwen Lei, Srinagesh Sharma, Urun Dogan, James~W
  Cutler, and Clayton Scott.
\newblock A generalization error bound for multi-class domain generalization.
\newblock {\em arXiv preprint arXiv:1905.10392}, 2019.

\bibitem{finn2017model}
Chelsea Finn, Pieter Abbeel, and Sergey Levine.
\newblock Model-agnostic meta-learning for fast adaptation of deep networks.
\newblock In {\em International Conference on Machine Learning}, pages
  1126--1135. PMLR, 2017.

\bibitem{balaji2018metareg}
Yogesh Balaji, Swami Sankaranarayanan, and Rama Chellappa.
\newblock Metareg: Towards domain generalization using meta-regularization.
\newblock {\em Advances in Neural Information Processing Systems},
  31:998--1008, 2018.

\bibitem{dou2019domain}
Qi~Dou, Daniel~C Castro, Konstantinos Kamnitsas, and Ben Glocker.
\newblock Domain generalization via model-agnostic learning of semantic
  features.
\newblock {\em arXiv preprint arXiv:1910.13580}, 2019.

\bibitem{li2019episodic}
Da~Li, Jianshu Zhang, Yongxin Yang, Cong Liu, Yi-Zhe Song, and Timothy~M
  Hospedales.
\newblock Episodic training for domain generalization.
\newblock In {\em Proceedings of the IEEE/CVF International Conference on
  Computer Vision}, pages 1446--1455, 2019.

\bibitem{shu2021open}
Yang Shu, Zhangjie Cao, Chenyu Wang, Jianmin Wang, and Mingsheng Long.
\newblock Open domain generalization with domain-augmented meta-learning.
\newblock {\em arXiv preprint arXiv:2104.03620}, 2021.

\bibitem{li2019feature}
Yiying Li, Yongxin Yang, Wei Zhou, and Timothy Hospedales.
\newblock Feature-critic networks for heterogeneous domain generalization.
\newblock In {\em International Conference on Machine Learning}, pages
  3915--3924. PMLR, 2019.

\bibitem{wang2020heterogeneous}
Yufei Wang, Haoliang Li, and Alex~C Kot.
\newblock Heterogeneous domain generalization via domain mixup.
\newblock In {\em ICASSP 2020-2020 IEEE International Conference on Acoustics,
  Speech and Signal Processing (ICASSP)}, pages 3622--3626. IEEE, 2020.

\bibitem{qiao2020learning}
Fengchun Qiao, Long Zhao, and Xi~Peng.
\newblock Learning to learn single domain generalization.
\newblock In {\em Proceedings of the IEEE/CVF Conference on Computer Vision and
  Pattern Recognition}, pages 12556--12565, 2020.

\bibitem{zhang2020adaptive}
Marvin Zhang, Henrik Marklund, Abhishek Gupta, Sergey Levine, and Chelsea Finn.
\newblock Adaptive risk minimization: A meta-learning approach for tackling
  group shift.
\newblock {\em arXiv preprint arXiv:2007.02931}, 2020.

\bibitem{mancini2018robust}
Massimiliano Mancini, Samuel~Rota Bulo, Barbara Caputo, and Elisa Ricci.
\newblock Robust place categorization with deep domain generalization.
\newblock {\em IEEE Robotics and Automation Letters}, 3(3):2093--2100, 2018.

\bibitem{mancini2018best}
Massimiliano Mancini, Samuel~Rota Bul{\`o}, Barbara Caputo, and Elisa Ricci.
\newblock Best sources forward: domain generalization through source-specific
  nets.
\newblock In {\em 2018 25th IEEE international conference on image processing
  (ICIP)}, pages 1353--1357. IEEE, 2018.

\bibitem{li2017deeper}
Da~Li, Yongxin Yang, Yi-Zhe Song, and Timothy~M Hospedales.
\newblock Deeper, broader and artier domain generalization.
\newblock In {\em Proceedings of the IEEE international conference on computer
  vision}, pages 5542--5550, 2017.

\bibitem{ding2017deep}
Zhengming Ding and Yun Fu.
\newblock Deep domain generalization with structured low-rank constraint.
\newblock {\em IEEE Transactions on Image Processing}, 27(1):304--313, 2017.

\bibitem{sagawa2019distributionally}
Shiori Sagawa, Pang~Wei Koh, Tatsunori~B Hashimoto, and Percy Liang.
\newblock Distributionally robust neural networks for group shifts: On the
  importance of regularization for worst-case generalization.
\newblock {\em arXiv preprint arXiv:1911.08731}, 2019.

\bibitem{hu2018does}
Weihua Hu, Gang Niu, Issei Sato, and Masashi Sugiyama.
\newblock Does distributionally robust supervised learning give robust
  classifiers?
\newblock In {\em International Conference on Machine Learning}, pages
  2029--2037. PMLR, 2018.

\bibitem{johansson2018learning}
Fredrik~D Johansson, Nathan Kallus, Uri Shalit, and David Sontag.
\newblock Learning weighted representations for generalization across designs.
\newblock {\em arXiv preprint arXiv:1802.08598}, 2018.

\bibitem{krizhevsky2012imagenet}
Alex Krizhevsky, Ilya Sutskever, and Geoffrey~E Hinton.
\newblock Imagenet classification with deep convolutional neural networks.
\newblock {\em Advances in neural information processing systems},
  25:1097--1105, 2012.

\bibitem{hendrycks2019augmix}
Dan Hendrycks, Norman Mu, Ekin~D Cubuk, Barret Zoph, Justin Gilmer, and Balaji
  Lakshminarayanan.
\newblock Augmix: A simple data processing method to improve robustness and
  uncertainty.
\newblock {\em arXiv preprint arXiv:1912.02781}, 2019.

\bibitem{chen2019invariance}
Shuxiao Chen, Edgar Dobriban, and Jane~H Lee.
\newblock Invariance reduces variance: Understanding data augmentation in deep
  learning and beyond.
\newblock {\em arXiv preprint arXiv:1907.10905}, 2019.

\bibitem{zhang2019unseen}
Ling Zhang, Xiaosong Wang, Dong Yang, Thomas Sanford, Stephanie Harmon, Baris
  Turkbey, Holger Roth, Andriy Myronenko, Daguang Xu, and Ziyue Xu.
\newblock When unseen domain generalization is unnecessary? rethinking data
  augmentation.
\newblock {\em arXiv preprint arXiv:1906.03347}, 2019.

\bibitem{volpi2018generalizing}
Riccardo Volpi, Hongseok Namkoong, Ozan Sener, John Duchi, Vittorio Murino, and
  Silvio Savarese.
\newblock Generalizing to unseen domains via adversarial data augmentation.
\newblock {\em arXiv preprint arXiv:1805.12018}, 2018.

\bibitem{xu2020adversarial}
Minghao Xu, Jian Zhang, Bingbing Ni, Teng Li, Chengjie Wang, Qi~Tian, and
  Wenjun Zhang.
\newblock Adversarial domain adaptation with domain mixup.
\newblock In {\em Proceedings of the AAAI Conference on Artificial
  Intelligence}, volume~34, pages 6502--6509, 2020.

\bibitem{yan2020improve}
Shen Yan, Huan Song, Nanxiang Li, Lincan Zou, and Liu Ren.
\newblock Improve unsupervised domain adaptation with mixup training.
\newblock {\em arXiv preprint arXiv:2001.00677}, 2020.

\bibitem{zhang2017mixup}
Hongyi Zhang, Moustapha Cisse, Yann~N Dauphin, and David Lopez-Paz.
\newblock mixup: Beyond empirical risk minimization.
\newblock {\em arXiv preprint arXiv:1710.09412}, 2017.

\bibitem{goodfellow2014generative}
Ian~J Goodfellow, Jean Pouget-Abadie, Mehdi Mirza, Bing Xu, David Warde-Farley,
  Sherjil Ozair, Aaron Courville, and Yoshua Bengio.
\newblock Generative adversarial networks.
\newblock {\em arXiv preprint arXiv:1406.2661}, 2014.

\bibitem{karras2019style}
Tero Karras, Samuli Laine, and Timo Aila.
\newblock A style-based generator architecture for generative adversarial
  networks.
\newblock In {\em Proceedings of the IEEE/CVF Conference on Computer Vision and
  Pattern Recognition}, pages 4401--4410, 2019.

\bibitem{brock2018large}
Andrew Brock, Jeff Donahue, and Karen Simonyan.
\newblock Large scale gan training for high fidelity natural image synthesis.
\newblock {\em arXiv preprint arXiv:1809.11096}, 2018.

\bibitem{gatys2016image}
Leon~A Gatys, Alexander~S Ecker, and Matthias Bethge.
\newblock Image style transfer using convolutional neural networks.
\newblock In {\em Proceedings of the IEEE conference on computer vision and
  pattern recognition}, pages 2414--2423, 2016.

\bibitem{zhu2017unpaired}
Jun-Yan Zhu, Taesung Park, Phillip Isola, and Alexei~A Efros.
\newblock Unpaired image-to-image translation using cycle-consistent
  adversarial networks.
\newblock In {\em Proceedings of the IEEE international conference on computer
  vision}, pages 2223--2232, 2017.

\bibitem{huang2018multimodal}
Xun Huang, Ming-Yu Liu, Serge Belongie, and Jan Kautz.
\newblock Multimodal unsupervised image-to-image translation.
\newblock In {\em Proceedings of the European conference on computer vision
  (ECCV)}, pages 172--189, 2018.

\bibitem{almahairi2018augmented}
Amjad Almahairi, Sai Rajeswar, Alessandro Sordoni, Philip Bachman, and Aaron
  Courville.
\newblock Augmented cyclegan: Learning many-to-many mappings from unpaired
  data.
\newblock {\em arXiv preprint arXiv:1802.10151}, 2018.

\bibitem{russo2018source}
Paolo Russo, Fabio~M Carlucci, Tatiana Tommasi, and Barbara Caputo.
\newblock From source to target and back: symmetric bi-directional adaptive
  gan.
\newblock In {\em Proceedings of the IEEE Conference on Computer Vision and
  Pattern Recognition}, pages 8099--8108, 2018.

\bibitem{zhou2020deep}
Kaiyang Zhou, Yongxin Yang, Timothy Hospedales, and Tao Xiang.
\newblock Deep domain-adversarial image generation for domain generalisation.
\newblock In {\em Proceedings of the AAAI Conference on Artificial
  Intelligence}, volume~34, pages 13025--13032, 2020.

\bibitem{carlucci2019domain}
Fabio~M Carlucci, Antonio D'Innocente, Silvia Bucci, Barbara Caputo, and
  Tatiana Tommasi.
\newblock Domain generalization by solving jigsaw puzzles.
\newblock In {\em Proceedings of the IEEE/CVF Conference on Computer Vision and
  Pattern Recognition}, pages 2229--2238, 2019.

\bibitem{vandenhende2019three}
Simon Vandenhende, Bert De~Brabandere, Davy Neven, and Luc Van~Gool.
\newblock A three-player gan: generating hard samples to improve classification
  networks.
\newblock In {\em 2019 16th International Conference on Machine Vision
  Applications (MVA)}, pages 1--6. IEEE, 2019.

\bibitem{arruda2019cross}
Vinicius~F Arruda, Thiago~M Paix{\~a}o, Rodrigo~F Berriel, Alberto~F De~Souza,
  Claudine Badue, Nicu Sebe, and Thiago Oliveira-Santos.
\newblock Cross-domain car detection using unsupervised image-to-image
  translation: From day to night.
\newblock In {\em 2019 International Joint Conference on Neural Networks
  (IJCNN)}, pages 1--8. IEEE, 2019.

\bibitem{rahman2019multi}
Mohammad~Mahfujur Rahman, Clinton Fookes, Mahsa Baktashmotlagh, and Sridha
  Sridharan.
\newblock Multi-component image translation for deep domain generalization.
\newblock In {\em 2019 IEEE Winter Conference on Applications of Computer
  Vision (WACV)}, pages 579--588. IEEE, 2019.

\bibitem{yue2019domain}
Xiangyu Yue, Yang Zhang, Sicheng Zhao, Alberto Sangiovanni-Vincentelli, Kurt
  Keutzer, and Boqing Gong.
\newblock Domain randomization and pyramid consistency: Simulation-to-real
  generalization without accessing target domain data.
\newblock In {\em Proceedings of the IEEE/CVF International Conference on
  Computer Vision}, pages 2100--2110, 2019.

\bibitem{murez2018image}
Zak Murez, Soheil Kolouri, David Kriegman, Ravi Ramamoorthi, and Kyungnam Kim.
\newblock Image to image translation for domain adaptation.
\newblock In {\em Proceedings of the IEEE Conference on Computer Vision and
  Pattern Recognition}, pages 4500--4509, 2018.

\bibitem{li2021semantic}
Daiqing Li, Junlin Yang, Karsten Kreis, Antonio Torralba, and Sanja Fidler.
\newblock Semantic segmentation with generative models: Semi-supervised
  learning and strong out-of-domain generalization.
\newblock {\em arXiv preprint arXiv:2104.05833}, 2021.

\bibitem{wang2019learning}
Haohan Wang, Zexue He, Zachary~C Lipton, and Eric~P Xing.
\newblock Learning robust representations by projecting superficial statistics
  out.
\newblock {\em arXiv preprint arXiv:1903.06256}, 2019.

\bibitem{nam2019reducing}
Hyeonseob Nam, HyunJae Lee, Jongchan Park, Wonjun Yoon, and Donggeun Yoo.
\newblock Reducing domain gap via style-agnostic networks.
\newblock {\em arXiv preprint arXiv:1910.11645}, 2019.

\bibitem{asadi2019towards}
Nader Asadi, Amir~M Sarfi, Mehrdad Hosseinzadeh, Zahra Karimpour, and Mahdi
  Eftekhari.
\newblock Towards shape biased unsupervised representation learning for domain
  generalization.
\newblock {\em arXiv preprint arXiv:1909.08245}, 2019.

\bibitem{krueger2020out}
David Krueger, Ethan Caballero, Joern-Henrik Jacobsen, Amy Zhang, Jonathan
  Binas, Dinghuai Zhang, Remi~Le Priol, and Aaron Courville.
\newblock Out-of-distribution generalization via risk extrapolation (rex).
\newblock {\em arXiv preprint arXiv:2003.00688}, 2020.

\bibitem{gan2016learning}
Chuang Gan, Tianbao Yang, and Boqing Gong.
\newblock Learning attributes equals multi-source domain generalization.
\newblock In {\em Proceedings of the IEEE conference on computer vision and
  pattern recognition}, pages 87--97, 2016.

\bibitem{matsuura2020domain}
Toshihiko Matsuura and Tatsuya Harada.
\newblock Domain generalization using a mixture of multiple latent domains.
\newblock In {\em Proceedings of the AAAI Conference on Artificial
  Intelligence}, volume~34, pages 11749--11756, 2020.

\bibitem{niu2015visual}
Li~Niu, Wen Li, and Dong Xu.
\newblock Visual recognition by learning from web data: A weakly supervised
  domain generalization approach.
\newblock In {\em Proceedings of the IEEE conference on computer vision and
  pattern recognition}, pages 2774--2783, 2015.

\bibitem{ben2009robust}
Aharon Ben-Tal, Laurent El~Ghaoui, and Arkadi Nemirovski.
\newblock {\em Robust optimization}.
\newblock Princeton university press, 2009.

\bibitem{ben2010theory}
Shai Ben-David, John Blitzer, Koby Crammer, Alex Kulesza, Fernando Pereira, and
  Jennifer~Wortman Vaughan.
\newblock A theory of learning from different domains.
\newblock {\em Machine learning}, 79(1):151--175, 2010.

\bibitem{david2010impossibility}
Shai~Ben David, Tyler Lu, Teresa Luu, and D{\'a}vid P{\'a}l.
\newblock Impossibility theorems for domain adaptation.
\newblock In {\em Proceedings of the Thirteenth International Conference on
  Artificial Intelligence and Statistics}, pages 129--136. JMLR Workshop and
  Conference Proceedings, 2010.

\bibitem{bagnell2005robust}
J~Andrew Bagnell.
\newblock Robust supervised learning.
\newblock In {\em AAAI}, pages 714--719, 2005.

\bibitem{scholkopf2012causal}
Bernhard Sch{\"o}lkopf, Dominik Janzing, Jonas Peters, Eleni Sgouritsa, Kun
  Zhang, and Joris Mooij.
\newblock On causal and anticausal learning.
\newblock {\em arXiv preprint arXiv:1206.6471}, 2012.

\bibitem{lipton2018detecting}
Zachary Lipton, Yu-Xiang Wang, and Alexander Smola.
\newblock Detecting and correcting for label shift with black box predictors.
\newblock In {\em International conference on machine learning}, pages
  3122--3130. PMLR, 2018.

\bibitem{ye2021ood}
Nanyang Ye, Kaican Li, Lanqing Hong, Haoyue Bai, Yiting Chen, Fengwei Zhou, and
  Zhenguo Li.
\newblock Ood-bench: Benchmarking and understanding out-of-distribution
  generalization datasets and algorithms.
\newblock {\em arXiv preprint arXiv:2106.03721}, 2021.

\bibitem{rosenfeld2018elephant}
Amir Rosenfeld, Richard Zemel, and John~K Tsotsos.
\newblock The elephant in the room.
\newblock {\em arXiv preprint arXiv:1808.03305}, 2018.

\bibitem{quinonero2009dataset}
Joaquin Qui{\~n}onero-Candela, Masashi Sugiyama, Neil~D Lawrence, and Anton
  Schwaighofer.
\newblock {\em Dataset shift in machine learning}.
\newblock Mit Press, 2009.

\bibitem{kattakinda2021focus}
Priyatham Kattakinda and Soheil Feizi.
\newblock Focus: Familiar objects in common and uncommon settings.
\newblock {\em arXiv preprint arXiv:2110.03804}, 2021.

\bibitem{singla2021causal}
Sahil Singla and Soheil Feizi.
\newblock Causal imagenet: How to discover spurious features in deep learning?
\newblock {\em arXiv preprint arXiv:2110.04301}, 2021.

\bibitem{koyama2020out}
Masanori Koyama and Shoichiro Yamaguchi.
\newblock Out-of-distribution generalization with maximal invariant predictor.
\newblock {\em arXiv preprint arXiv:2008.01883}, 2020.

\bibitem{rosenfeld2020risks}
Elan Rosenfeld, Pradeep Ravikumar, and Andrej Risteski.
\newblock The risks of invariant risk minimization.
\newblock {\em arXiv preprint arXiv:2010.05761}, 2020.

\bibitem{kamath2021does}
Pritish Kamath, Akilesh Tangella, Danica~J Sutherland, and Nathan Srebro.
\newblock Does invariant risk minimization capture invariance?
\newblock {\em arXiv preprint arXiv:2101.01134}, 2021.

\bibitem{nagarajan2020understanding}
Vaishnavh Nagarajan, Anders Andreassen, and Behnam Neyshabur.
\newblock Understanding the failure modes of out-of-distribution
  generalization.
\newblock {\em arXiv preprint arXiv:2010.15775}, 2020.

\bibitem{ahuja2020empirical}
Kartik Ahuja, Jun Wang, Amit Dhurandhar, Karthikeyan Shanmugam, and Kush~R
  Varshney.
\newblock Empirical or invariant risk minimization? a sample complexity
  perspective.
\newblock {\em arXiv preprint arXiv:2010.16412}, 2020.

\bibitem{chamon2020probably}
Luiz Chamon and Alejandro Ribeiro.
\newblock Probably approximately correct constrained learning.
\newblock {\em Advances in Neural Information Processing Systems}, 33, 2020.

\bibitem{chamon2020empirical}
Luiz~FO Chamon, Santiago Paternain, Miguel Calvo-Fullana, and Alejandro
  Ribeiro.
\newblock The empirical duality gap of constrained statistical learning.
\newblock In {\em ICASSP 2020-2020 IEEE International Conference on Acoustics,
  Speech and Signal Processing (ICASSP)}, pages 8374--8378. IEEE, 2020.

\bibitem{hornik1991approximation}
Kurt Hornik.
\newblock Approximation capabilities of multilayer feedforward networks.
\newblock {\em Neural networks}, 4(2):251--257, 1991.

\bibitem{hornik1989multilayer}
Kurt Hornik, Maxwell Stinchcombe, and Halbert White.
\newblock Multilayer feedforward networks are universal approximators.
\newblock {\em Neural networks}, 2(5):359--366, 1989.

\bibitem{pathak2015constrained}
Deepak Pathak, Philipp Krahenbuhl, and Trevor Darrell.
\newblock Constrained convolutional neural networks for weakly supervised
  segmentation.
\newblock In {\em Proceedings of the IEEE international conference on computer
  vision}, pages 1796--1804, 2015.

\bibitem{chen2018approximating}
Steven Chen, Kelsey Saulnier, Nikolay Atanasov, Daniel~D Lee, Vijay Kumar,
  George~J Pappas, and Manfred Morari.
\newblock Approximating explicit model predictive control using constrained
  neural networks.
\newblock In {\em 2018 Annual American control conference (ACC)}, pages
  1520--1527. IEEE, 2018.

\bibitem{frerix2020homogeneous}
Thomas Frerix, Matthias Nie{\ss}ner, and Daniel Cremers.
\newblock Homogeneous linear inequality constraints for neural network
  activations.
\newblock In {\em Proceedings of the IEEE/CVF Conference on Computer Vision and
  Pattern Recognition Workshops}, pages 748--749, 2020.

\bibitem{amos2017optnet}
Brandon Amos and J~Zico Kolter.
\newblock Optnet: Differentiable optimization as a layer in neural networks.
\newblock In {\em International Conference on Machine Learning}, pages
  136--145. PMLR, 2017.

\bibitem{ravi2018constrained}
Sathya~N Ravi, Tuan Dinh, Vishnu Lokhande, and Vikas Singh.
\newblock Constrained deep learning using conditional gradient and applications
  in computer vision.
\newblock {\em arXiv preprint arXiv:1803.06453}, 2018.

\bibitem{donti2021dc3}
Priya~L Donti, David Rolnick, and J~Zico Kolter.
\newblock Dc3: A learning method for optimization with hard constraints.
\newblock {\em arXiv preprint arXiv:2104.12225}, 2021.

\bibitem{higgins2018towards}
Irina Higgins, David Amos, David Pfau, Sebastien Racaniere, Loic Matthey,
  Danilo Rezende, and Alexander Lerchner.
\newblock Towards a definition of disentangled representations.
\newblock {\em arXiv preprint arXiv:1812.02230}, 2018.

\bibitem{bertsekas2015convex}
Dimitri~P Bertsekas and Athena Scientific.
\newblock {\em Convex optimization algorithms}.
\newblock Athena Scientific Belmont, 2015.

\bibitem{chamon2021constrained}
Luiz~FO Chamon, Santiago Paternain, Miguel Calvo-Fullana, and Alejandro
  Ribeiro.
\newblock Constrained learning with non-convex losses.
\newblock {\em arXiv preprint arXiv:2103.05134}, 2021.

\bibitem{cubuk2020randaugment}
Ekin~D Cubuk, Barret Zoph, Jonathon Shlens, and Quoc~V Le.
\newblock Randaugment: Practical automated data augmentation with a reduced
  search space.
\newblock In {\em Proceedings of the IEEE/CVF Conference on Computer Vision and
  Pattern Recognition Workshops}, pages 702--703, 2020.

\bibitem{lecun2010mnist}
Yann LeCun, Corinna Cortes, and CJ~Burges.
\newblock Mnist handwritten digit database.
\newblock {\em ATT Labs [Online]. Available: http://yann. lecun.
  com/exdb/mnist}, 2, 2010.

\bibitem{cover1999elements}
Thomas~M Cover.
\newblock {\em Elements of information theory}.
\newblock John Wiley \& Sons, 1999.

\bibitem{bass2013real}
Richard~F Bass.
\newblock {\em Real analysis for graduate students}.
\newblock Createspace Ind Pub, 2013.

\bibitem{boyd2004convex}
Stephen Boyd, Stephen~P Boyd, and Lieven Vandenberghe.
\newblock {\em Convex optimization}.
\newblock Cambridge university press, 2004.

\bibitem{goberna2017recent}
Miguel~A Goberna and MA~L{\'o}pez.
\newblock Recent contributions to linear semi-infinite optimization.
\newblock {\em 4OR}, 15(3):221--264, 2017.

\bibitem{stein2011functional}
Elias~M Stein and Rami Shakarchi.
\newblock {\em Functional analysis: introduction to further topics in
  analysis}, volume~4.
\newblock Princeton University Press, 2011.

\bibitem{kingma2014adam}
Diederik~P Kingma and Jimmy Ba.
\newblock Adam: A method for stochastic optimization.
\newblock {\em arXiv preprint arXiv:1412.6980}, 2014.

\bibitem{ghifary2015domain}
Muhammad Ghifary, W~Bastiaan Kleijn, Mengjie Zhang, and David Balduzzi.
\newblock Domain generalization for object recognition with multi-task
  autoencoders.
\newblock In {\em Proceedings of the IEEE international conference on computer
  vision}, pages 2551--2559, 2015.

\bibitem{anoosheh2018combogan}
Asha Anoosheh, Eirikur Agustsson, Radu Timofte, and Luc Van~Gool.
\newblock Combogan: Unrestrained scalability for image domain translation.
\newblock In {\em Proceedings of the IEEE conference on computer vision and
  pattern recognition workshops}, pages 783--790, 2018.

\bibitem{choi2020stargan}
Yunjey Choi, Youngjung Uh, Jaejun Yoo, and Jung-Woo Ha.
\newblock Stargan v2: Diverse image synthesis for multiple domains.
\newblock In {\em Proceedings of the IEEE/CVF Conference on Computer Vision and
  Pattern Recognition}, pages 8188--8197, 2020.

\end{thebibliography}
\bibliographystyle{unsrt}

\newpage
\appendix
\section{Further theoretical results and discussion}

\subsection{On the optimality of relaxation of Problem \texorpdfstring{\ref{prob:model-based-domain-gen}}{Lg} in \texorpdfstring{\eqref{eq:relax-mbdg}}{Lg}}

In Section \ref{sect:approx-mbdg} of the main text, we claimed that the relaxation introduced in \eqref{eq:relax-mbdg} was tight under certain conditions.  In this section of the appendix, we formally enumerate the conditions under which the relaxation is tight.  Further, we show that the tightness of the relaxation can be characterized by the margin parameter $\gamma$.

\subsubsection{The case when \texorpdfstring{$\gamma = 0$}{Lg}}
\label{sect:relaxing-invar-prop}

In Section \ref{sect:approx-mbdg}, we claimed that the relaxation of the Model-Based Domain Generalization problem given in \eqref{eq:relax-mbdg} was tight when $\gamma = 0$ under mild conditions on the distance metric $d$.  In particular, we simply require that $d(\Prob, \mathbb{T}) = 0$ if and only if $\Prob = \mathbb{T}$ almost surely.  We emphasize that this condition is not overly restrictive.  Indeed, a variety of distance metrics, including the KL-divergence and more generally the family of $f$-divergences, satisfy this property (c.f.\ \cite[Theorem~8.6.1]{cover1999elements}).  In what follows, we formally state and prove this result.

\begin{proposition} \label{prop:relaxation}
Let $d$ be a distance metric between probability measures for which it holds that $d(\Prob, \mathbb{T}) = 0$ for two distributions $\Prob$ and $\mathbb{T}$ if and only if $\Prob = \mathbb{T}$ almost surely.  Then $P^\star(0) = P^\star$.
\end{proposition}

\begin{proof}
The idea in this proof is simply to leverage the fact a non-negative random variable has expectation zero if and only if it is zero almost everywhere.  For ease of exposition, we remind the reader of the definition of the relaxed constraints: $\calL^e(f) := \E_{\Prob(X)} d(f(X), f(G(X,e)) )$.

First, observe that because $d(\cdot, \cdot)$ is a metric, it is non-negative-valued.  Then the following statement is trivial
\begin{align}
    \calL^e(f) \leq 0 \iff \calL^e(f) = 0.
\end{align}
Next, we claim that under the assumptions given in the statement of the proposition, $\calL^e(f) = 0$ is equivalent to the $G$-invariance condition.  To verify this claim, for simplicity we start by defining the random variable
\begin{align}
    Z_e \triangleq d\big(f(X), f(G(X, e)) \big)
\end{align}
and note that by construction $Z_e\geq 0$ a.e.\ and $\calL^e(f) = \E_{\Prob(X)} Z_e$.  Now consider that because $Z_e$ is non-negative and has an expectation of zero, we have that $\E_{\Prob(X)} Z_e = 0$ if and only if $Z_e = 0$ almost surely (c.f.\ Prop.\ 8.1 in \cite{bass2013real}).  In other words, we have shown that
\begin{align}
    \calL^e(f) = 0 \iff d\big(f(x), f(G(x,e)) \big) = 0 \quad \text{a.e. } \: x\sim\Prob(X) \label{eq:ae-iff-zero-exp}
\end{align}
holds for each $e\in\Eall$.  Now by assumption, we have that for any two distributions $\Prob$ and $\mathbb{T}$ sharing the same support that $d(\Prob, \mathbb{T}) = 0$ holds if and only if $\Prob = \mathbb{T}$ almost surely.  Applying this to \eqref{eq:ae-iff-zero-exp}, we have that
\begin{align}
    \calL^e(f) = 0 \iff f(x) = f(G(x,e)) \quad \text{a.e. } \: x\sim\Prob(X).
\end{align}
Altogether we have shown that $\calL^e(f) \leq 0$ if and only if $f$ is $G$-invariant.  Thus, when $\gamma = 0$, the optimization problems in \eqref{eq:model-based-domain-gen} and \eqref{eq:relax-mbdg} are equivalent, which implies that $P^\star(0) = P^\star$. 
\end{proof}

\subsection{The case when \texorpdfstring{$\gamma > 0$}{Lg}} \label{sect:gamma-greater-than-zero}

When $\gamma > 0$, the relaxation is no longer tight.  However, if the perturbation function $P^\star(\gamma)$ is assumed to be Lipschitz continuous, we can directly characterize the tightness of the bound.

\begin{remark} \label{rmk:gamma-remark}
Let us assume that the perturbation function $P^\star(\gamma)$ is $L$-Lipschitz continuous in $\gamma$.  Then given Proposition \ref{prop:relaxation}, it follows that $|P^\star - P^\star(\gamma)| \leq L\gamma$.  
\end{remark}

\begin{proof}
Observe that by Proposition \ref{prop:relaxation}, we have that $P^\star = P^\star(0)$.  It follows that
\begin{align}
    |P^\star - P^\star(\gamma)| &= |P^\star(0) - P^\star(\gamma)| \\
    &\leq L|0 - \gamma| \label{eq:apply-lipschitz} \\
    &= L\gamma
\end{align}
where the inequality in \eqref{eq:apply-lipschitz} follows by the definition of Lipschitz continuity.
\end{proof}

\noindent We note that in general the perturbation function $P^\star(\gamma)$ cannot be guaranteed to be Lipschitz.  However, as we will show in Remark~\ref{rmk:opt-dual-var}, when strong duality holds for~\eqref{eq:model-based-domain-gen}, $P^\star(\gamma)$ turns out to be Lipschitz continuous with a Lipschitz constant equal to the $L^1$ norm of optimal dual variable for the dual problem to~\eqref{eq:model-based-domain-gen}.  Before proving this result, we state a preliminary lemma from~\cite{boyd2004convex}.
\begin{lemma}[\S 5.6.2 in~\cite{boyd2004convex}] \label{lemma:pert}
Consider a generic optimization problem
\begin{align}
    p^\star \triangleq \min_{x\in\R^d} \: f_0(x) \quad\text{subject to} \quad  f_i(x) \leq 0 \quad\forall i\{1, \dots, m\}.
\end{align}
Assume that strong duality holds for this problem, and let $\lambda^\star$ denote an optimal dual variable.  Define the perturbation function as follows:
\begin{align}
    p^\star(u) \triangleq \min_{x\in\R^d} \: f_0(x) \quad\text{subject to}\quad f_i(x) \leq u_i \quad\forall i\in\{1,\dots,m\}
\end{align}
where $u\in\R^m$.  Then it holds that $p^\star(u) \geq p^\star - u^\top \lambda^\star$.
\end{lemma}

\noindent This useful result, which follows from a simple one-line proof in \S 5.6.2 of~\cite{boyd2004convex}, shows that the perturbation function $p^\star(u)$ can be related to the optimal value of the unperturbed problem via the optimal dual variable.  We can readily use a semi-infinite version of this lemma to prove the following remark:

\begin{remark} \label{rmk:opt-dual-var}
Consider the dual problem to~\eqref{eq:model-based-domain-gen}:
\begin{align}
    D^\star \triangleq \max_{\lambda\in\calB(\Eall)} \: \min_{f\in\calF} \: R(f) + \int_{\Eall} \left[ L^e(f) - \gamma\right] \text{d}\nu(e)
\end{align}
where $\calB(\cdot)$ denotes the cone of non-regular, non-negative Borel measures supported on its argument~\cite{goberna2017recent}.  Assume that strong duality holds, and let $\nu^\star$ denote an optimal dual variable for this problem.  Then it holds that 
\begin{align}
    |P^\star - P^\star(\gamma)| \leq \gamma \norm{\nu^\star}_{L^1}.
\end{align}
\end{remark}

\begin{proof}
The idea here is to apply Lemma~\ref{lemma:pert} for the constant function defined by $u = u(e) = \gamma$ $\forall e\in\Eall$.  To begin, let $\langle\cdot, \cdot\rangle$ denote the standard inner product on $L^2$; i.e. $\langle f,g\rangle = \int_{\Eall} f(e)g(e)\text{d}e$ for $f,g\in L^2(\Eall)$.  In this way, we find that
\begin{align}
    P^\star - \langle u, \nu^\star\rangle \leq P^\star(\gamma) \leq P^\star \label{eq:pert-mb}
\end{align}
where the second inequality holds because for $\gamma$ strictly larger than zero, the relaxation in~\eqref{eq:relax-mbdg} corresponds to an expansion of the feasible set of relative to~\eqref{eq:model-based-domain-gen}.  In this case, since $u$ is constant, a simple calculation shows that
\begin{align}
    \langle u, \nu^\star\rangle = \int_{\Eall} \nu^\star(e) u(e)\text{d}e = \gamma \int_{\Eall} \nu^\star(e)\text{d}e = \gamma \cdot \norm{\nu^\star}_{L^1}
\end{align}
where in the last step we have used the fact that the optimal dual variable $\nu^\star \succeq 0$.  Now if we apply this result to~\eqref{eq:pert-mb}, we find that
\begin{align}
    P^\star - \gamma\norm{\nu^\star} \leq P^\star(\gamma) \leq P^\star,
\end{align}
which directly implies the desired result.
\end{proof}

\subsection{Relationship to constrained PAC learning} \label{sect:pacc}

Recently, the authors of \cite{chamon2020probably} introduced the Probably Approximately Correct Constrained (PACC) framework, which extends the classical PAC framework to constrained problems.  In particular, recall the following definition of agnostic PAC learnability:
\begin{definition}[PAC learnability]
A hypothesis class $\calH$ is said to be (agnostic) PAC learnable if for every $\epsilon,\delta\in(0,1)$ and every distribution $\Prob_0$, there exists a $\theta^\star\in\calH$ which can be obtained from $N\geq N_{\calH}(\epsilon,\delta)$ samples from $\Prob_0$ such that $\E \ell(\varphi(\theta, X),Y) \leq U^\star + \epsilon$ with probability $1-\delta$, where
\begin{align}
    U^\star \triangleq \minimize_{\theta\in\calH} \: \E_{\Prob_0(X,Y)} \ell(\varphi(\theta, X), Y)
\end{align}
\end{definition}
\noindent The authors of \cite{chamon2020probably} extended this definition toward studying the learning theoretic properties of constrained optimization problems of the form
\begin{alignat}{2}
    C^\star \triangleq &\minimize_{\theta\in\calH} \: &&\E_{\Prob_0(X,Y)} \ell_0(\varphi(\theta, X), Y) \label{eq:pacc} \\
    &\st &&\E_{\Prob_i(X,Y)} \ell_i(\varphi(\theta, X), Y) \leq c_i \quad\text{for } i\in\{1, \dots, m\} \\
    & &&\ell_j(\varphi(\theta, X), Y) \leq c_j \:\: \Prob_j-\text{a.e.} \quad \text{for } j\in\{m+1, \dots m+q\}
\end{alignat}
via the following definition:
\begin{definition}[PACC learnability]
A hypothesis class $\calH$ is said to be PACC learnable if for every $\epsilon,\delta\in(0,1)$ and every distribution $\calP_i$ for $i\in\{0, \dots, m+q\}$, there exists a $\theta^\star\in\calH$ which can be obtained from $N\geq N_{\calH}(\epsilon,\delta)$ samples from each of the distributions $\Prob_i$ such that, with probability $1-
\delta$, $\theta^\star$ is:
\begin{enumerate}
    \item[(1)] approximately optimal, meaning that
    \begin{align}
        \E_{\Prob_0} \ell_0(\varphi(\theta^\star, X),Y) \leq C^\star + \epsilon
    \end{align}
    \item[(2)] approximately feasible, meaning that
    \begin{align}
        &\E_{\Prob_i(X,Y)} \ell_i(\varphi(\theta, X), Y) \leq c_i + \epsilon \quad\text{for } i\in\{1, \dots, m\} \\
        &\ell_j(\varphi(X), Y) \leq c_j \:\: \forall (x,y)\in\mathcal{K}_j \quad\text{for } j \in \{m+1, \dots, m+q\}
    \end{align}
    where $\mathcal{K}_j\subseteq\calX\times\calY$ are sets of $\Prob_j$ measure at least $1-\epsilon$.
\end{enumerate}
\end{definition}
\noindent One of the main results in \cite{chamon2020probably} is that a hypothesis class $\calH$ is PAC learnable if and only if it is PACC learnable.

Now if we consider the optimization problem in \eqref{eq:pacc}, we see that the admissible constraints are both inequality constraints.  In contrast, the optimization problem in Problem \ref{prob:model-based-domain-gen} contains a family of equality constraints.  Thus, in addition to easing the burden of enforcing hard $G$-invariance, the relaxation in \eqref{eq:relax-mbdg} serves to manipulate the Model-Based Domain Generalization problem into a form compatible with \eqref{eq:pacc}.  This is one of the key steps that sets the stage for deriving the learning theoretic guarantees for Model-Based Domain Generalization (e.g.\ Theorems \ref{thm:duality-gap} and \ref{thm:primal-dual}).

\subsection{Regularization vs.\ dual ascent} \label{sect:reg-vs-primal-dual}

A common trick for encouraging constraint satisfaction is to introduce soft constraints by adding a regularizer multiplied by a fixed penalty weight to the objective.  As noted in Section \ref{sect:experiments}, this approach yields a similar optimization problem to \eqref{eq:param-empir-dual}.  In particular, the regularized version of \eqref{eq:param-empir-dual} is the following:
\begin{align}
    \hat{D}_{\epsilon,N,\Etrain}^\star \triangleq \minimize_{\theta\in\calH} \hat{R}(\theta) + \frac{1}{|\Etrain|}\sum\nolimits_{e\in\Etrain} \left[\hat{\mathcal{L}}^e(\theta) - \gamma\right] w(e) \label{eq:regularized-mbdg}
\end{align}
where $w(e) \geq 0$ $e\in\Etrain$ are weights that are chosen as hyperparameters.  From an optimization perspective, the benefit of such an objective is that gradient-based algorithms are known to converge to local minima given small enough step sizes~\eqref{eq:model-based-domain-gen}.  However, classical results in learning theory can only provide generalization guarantees on the aggregated objective, rather than on each term individually.  Furthermore, the choice of the penalty weights $w(e)$ is non-trivial and often requires significant domain knowledge, limiting the applicability of this approach.

In contrast, in primal-dual style algorithms, the weights $\lambda(e)$ are not fixed beforehand.  Rather, the $\lambda(e)$ are updated iteratively via the dual ascent step described in line 8 of Algorithm \ref{alg:mbst}.  Furthermore, as we showed in the main text, the optimal value of the primal problem $P^\star$ can be directly related to the solution of the empirical dual problem in \eqref{eq:param-empir-dual} via Theorem \ref{thm:duality-gap}.  Such guarantees are not possible in the regularization case, which underscores the benefits of the primal-dual iteration over the more standard regularization approach. 

\newpage
\section{Omitted proofs} \label{app:omitted-proofs}

In this appendix, we provide the proofs that were omitted in the main text.  For ease of exposition, we restate each result before proving it so that the reader can avoid scrolling back and forth between the main text and the appendices.

\subsection{Proof of Proposition\ \ref{prop:pull-back-domain-gen}} \label{sect:proof-of-pull-back}

\textbf{Proposition~\ref{prop:pull-back-domain-gen}.}  Under Assumptions~\ref{assume:gen-model} and~\ref{assume:cov-shift}, Problem \ref{prob:domain-gen} is equivalent to
\begin{align}
    \minimize_{f\in\cal F} \: \max_{e\in\Eall} \: \E_{\Prob(X,Y)}  \ell(f(G(X,e)), Y).
\end{align}

\begin{proof}
The main idea in this proof is the following.  First, we decompose the joint distribution $\Prob(X^e,Y^e)$ into $\Prob(Y^e|X^e)\cdot \Prob(X^e)$ to expand the risk term in the objective of \eqref{eq:domain-gen}.  Next, we leverage Assumptions~\ref{assume:gen-model} and~\ref{assume:cov-shift} to rewrite the outer and inner expectations engendered by the tower property.  Finally, we undo our expansion to arrive at at the statement of the proposition.

To begin, observe that by the our decomposition $\Prob(X^e,Y^e)=\Prob(Y^e|X^e)\cdot \Prob(X^e)$ of the joint expectation, we can rewrite the objective of \eqref{eq:domain-gen} in the following way:
\begin{align}
    \E_{\Prob(X^e,Y^e)} \ell(f(X^e), Y^e) = \E_{\Prob(X^e)} \left[ \E_{\Prob(Y^e|X^e)} \ell(f(X^e), Y^e) \right]. \label{eq:tower}
\end{align}
Then, recall that by Assumption \ref{assume:cov-shift}, we have that $\Prob(Y^e|X^e) = \Prob(Y|X)$ $\forall e \in\Eall$, i.e.\ the conditional distribution of labels given instances is the same across domains.  Thus, if we consider the inner expectation in \eqref{eq:tower}, it follows that
\begin{align}
     \E_{\Prob(Y^e|X^e)} \ell(f(X^e), Y^e) = \E_{\Prob(Y|X)} \ell(f(X), Y). \label{eq:use-cov-assump}
\end{align}
Now observe that under Assumption \ref{assume:gen-model}, we have that $\Prob(X^e) \overset{d}{=} G \: \# \: (\Prob(X), \delta_e)$.  Therefore, a simple manipulation reveals that
\begin{align}
    \E_{\Prob(X^e)} \left[\E_{\Prob(Y^e|X^e)} \ell(f(X), Y) \right] &= \E_{G\: \# \: (\Prob(X), \:\delta_e)} \left[\E_{\Prob(Y|X)} \ell(f(X), Y) \right] \\
    &= \E_{\Prob(X)} \left[ \E_{\Prob(Y|X)} \ell(f(G(X, e)), Y) \right] \\
    &= \E_{\Prob(X,Y)} \ell(f(G(X,e)), Y), \label{eq:reapply-tower}
\end{align}
where the final step again follows from the tower property of expectation.  Therefore, by combining \eqref{eq:tower} and \eqref{eq:reapply-tower}, we conclude that
\begin{align}
    \E_{\Prob(X^e,Y^e)} \ell(f(X^e), Y^e) = \E_{\Prob(X,Y)} \ell(f(G(X,e)), Y),
\end{align}
which directly implies the statement of the proposition.
\end{proof}

\subsection{Proof of Proposition \ref{prop:mbdg}} \label{sect:proof-of-mbdg}

\textbf{Proposition~\ref{prop:mbdg}.}  Under Assumptions~\ref{assume:gen-model} and~\ref{assume:cov-shift}, if we restrict the feasible set to the set of $G$-invariant predictors, then Problem \ref{prob:domain-gen} is equivalent to the following semi-infinite constrained problem:
\begin{alignat}{2}
    P^\star \triangleq &\minimize_{f\in\mathcal{F}} \: &&R(f) \triangleq \E_{\Prob(X,Y)} \ell(f(X),Y) \\
    &\st  &&f(x) = f(G(x,e)) \quad \text{ a.e. } x\sim \Prob(X) \:\: \forall e\in\Eall. \notag
\end{alignat}

\begin{proof}
The main idea in this proof is simply to leverage the definition of $G$-invariance and the result of Prop.\ \ref{prop:pull-back-domain-gen}.  Starting from Prop.\ \ref{prop:pull-back-domain-gen}, we see that by restricting the feasible set to the set of $G$ invariant predictors, the optimization problem in \eqref{eq:pull-back-dg} can be written as
\begin{alignat}{2}
    P^\star = &\minimize_{f\in\calF} \: &&\max_{e\in\Eall} \:\E_{\Prob(X,Y)} \ell(f(G(X,e)), Y) \\
    &\st && f(x) = f(G(x,e)) \quad \text{a.e.} x\sim \Prob(X), \: \forall e\in\Eall
\end{alignat}
Now observe that due to the constraint, we can replace the $f(G(X,e))$ term in the objective with $f(X)$.  Thus, the above problem is equivalent to
\begin{alignat}{2}
P^\star = &\minimize_{f\in\calF} \: &&\max_{e\in\Eall} \:\E_{\Prob(X,Y)} \ell(f(X), Y) \label{eq:max-without-e} \\
    &\st && f(x) = f(G(x,e)) \quad \text{a.e. } \: x\sim \Prob(X), \: \forall e\in\Eall
\end{alignat}
Now observe that the objective in \eqref{eq:max-without-e} is free of the optimization variable $e\in\Eall$.  Therefore, we can eliminate the inner maximization step in \eqref{eq:max-without-e}, which verifies the claim of the proposition.
\end{proof}

\subsection{Proof of Proposition \ref{prop:param-gap}} \label{app:proof-param-gap}

Before proving Proposition~\ref{prop:param-gap}, we formally state the assumptions we require on $\ell$ and $d$.  These assumptions are enumerated in the following Assumption:

\begin{assumption}\label{assume:lipschitz}
We make the following assumptions:
\begin{enumerate}
    \item The loss function $\ell$ is non-negative, convex, and $L_\ell$-Lipschitz continuous in it's first argument,~i.e.
    \begin{align}
        |\ell(f_1(x), y) - \ell(f_2(x),y)| \leq \norm{f_1(x) - f_2(x)}_\infty
    \end{align}
    \item The distance metric $d$ is non-negative, convex, and satisfies the following uniform Lipschitz-like inequality for some constant $L_d>0$:
    \begin{align}
        |d(f_1(x), f_1(G(x,e))) - d(f_2(x), f_2(G(x,e)))| \leq L_d \norm{f_1(x) - f_2(x)}_\infty \quad\forall e\in\Eall.
    \end{align}
    \item There exists a predictor $f\in\calF$ such that $\calL^e(f) < \gamma - \epsilon \cdot \max\{L_\ell, L_d\}$ $\forall e\in\Eall$.
\end{enumerate}
\end{assumption}

\noindent At a high level, these assumptions necessitate that $\ell$ and $d$ are sufficiently regular and that the problem is strictly feasible with a particular margin $\epsilon\cdot \max\{L_\ell, L_d\}$.  In particular, this final assumption is essential as it implies that strong duality holds for~\eqref{eq:relax-mbdg}, which is a key technical element of the proof.  Given these assumptions, we restate Proposition~\ref{prop:param-gap} below:\\

\noindent\textbf{Proposition~\ref{prop:param-gap}.}  Let $\gamma > 0$ be given.  Then under Assumption~\ref{assume:lipschitz}, it holds that
\begin{align}
    P^\star(\gamma) \leq D^\star_\epsilon(\gamma) \leq P^\star(\gamma) + \epsilon \left(1 + \norm{\lambda_\text{pert}^\star}_{L^1} \right) \cdot\max\{L_\ell, L_d\} \label{eq:upper-and-lower}
\end{align}
where $\lambda_\text{pert}^\star$ is the optimal dual variable for a perturbed version of~\eqref{eq:relax-mbdg} in which the constraints are tightened to hold with margin $\gamma - \epsilon\cdot\max\{L_\ell, L_d\}$.  In particular, this result implies that
\begin{align}
    |P^\star(\gamma) - D_\epsilon^\star(\gamma)| \leq \epsilon \left(1 + \norm{\lambda_\text{pert}^\star}_{L^1} \right) \cdot\max\{L_\ell, L_d\}
\end{align}

\begin{proof}
In this proof, we extend the results of \cite{chamon2020empirical} to optimization problems with an infinite number of constraints.  The key insight toward deriving the lower bound is to use the fact that maximizing over the $\epsilon$-parameterization of $\calF$ yields a sub-optimal result vis-a-vis maximizing over $\calF$.  On the other hand, the upper bound, which requires slightly more machinery, leverages Jensen's and H\"older's inequalities along with the definition of the $\epsilon$-parameterization to over-approximate the parameter space via a Lipschitz $\epsilon$-ball covering argument.

\textbf{Step 1.}  In the first step, we prove the lower bound in~\eqref{eq:upper-and-lower}. To begin, we define the dual problem to the relaxed Model-Based Domain Generalization problem in \eqref{eq:relax-mbdg} in the following way:
\begin{align}
    D^\star(\gamma) \triangleq \maximize_{\lambda\in\calB(\Eall)} \: \min_{f\in\calF } \: \Lambda(f, \lambda) \triangleq R(f) + \int_{\Eall} \left[\mathcal{L}^e(\varphi(\theta, \cdot)) - \gamma\right] \text{d}\lambda(e). \label{eq:true-dual}
\end{align}
where with a slight abuse of notation, we redefine the Lagrangian $\Lambda$ from \eqref{eq:param-dual} in its first argument.  Now recall that by assumption, there exists a predictor $f\in\calF$ such that $\calL(f) < \gamma$ $\forall e\in\Eall$.  Thus, Slater's condition holds \cite{boyd2004convex}, and therefore so too does strong duality.  Now let $f^\star$ be optimal for the primal problem \eqref{eq:relax-mbdg}, and let $\lambda^\star \in\calB(\Eall)$ be dual optimal for the dual problem \eqref{eq:true-dual}; that is,
\begin{align}
    f^\star \in\argmin_{f\in\calF} \: \max_{\lambda\in\calB(\Eall)} \: R(f) + \int_{\Eall} \left[\mathcal{L}^e(\varphi(\theta, \cdot)) - \gamma\right] \text{d}\lambda(e) \label{eq:def-of-primal-opt}
\end{align}
and 
\begin{align}
    \lambda^\star \in \argmax_{\lambda\in\calB(\Eall)} \: \min_{f\in\calF} \: R(f) + \int_{\Eall} \left[\mathcal{L}^e(\varphi(\theta, \cdot)) - \gamma\right] \text{d}\lambda(e) \label{def-of-dual-opt}
\end{align}
At this early stage, it will be useful to state the following saddle-point relation, which is a direct result of strong duality:
\begin{align}
    \Lambda(f^\star, \lambda') \leq \Lambda(f^\star, \lambda^\star) \leq \Lambda(f', \lambda^\star) \label{eq:orig-saddle-point}
\end{align}
which holds for all $f'\in\calF$ and for all $\lambda'\in\calB(\Eall)$.  Now consider that by the definition of the optimization problem in \eqref{eq:param-dual}, we have that
\begin{align}
    D^\star_\epsilon(\gamma) = \max_{\lambda\in\calB(\Eall)} \: \min_{\theta\in\calH} \: \Lambda(\theta, \lambda) \geq \min_{\theta\in\calH} \Lambda(\theta, \lambda') \quad\forall \lambda'\in\calB(\Eall).
\end{align}
Therefore, by choosing $\lambda' = \lambda^\star$ in the above expression, and since $\calA_\epsilon = \{\varphi(\theta, \cdot) : \theta\in\calH\} \subseteq \calF$ by the definition of an $\epsilon$-parametric approximation, we have that
\begin{align}
    D^\star_\epsilon(\gamma) \geq \min_{\theta\in\calH} \: \Lambda(\theta, \lambda^\star) \geq \min_{f\in\calF} \: \Lambda(f,\lambda^\star) = P^\star(\gamma).
\end{align} 
This concludes the proof of the lower bound: $P^\star(\gamma) \leq D^\star_\epsilon(\gamma)$.

\textbf{Step 2.} Next, we show that $D_\epsilon^\star(\gamma)$ is upper bounded by the optimal value of a perturbed version of the empirical dual problem.  To begin, we add and subtract $\min_{f\in\calF} \: \Lambda(f, \lambda)$ from the parameterized dual problem in \eqref{eq:param-dual}.
\begin{align}
    D_\epsilon^\star(\gamma) &=  \max_{\lambda\in\calB(\Eall)} \: \min_{\theta\in\mathcal{H}} \: \left[\Lambda(\theta, \lambda) + \min_{f\in\calF} \: \Lambda(f,\lambda) - \min_{f\in\calF} \: \Lambda(f,\lambda) \right] \\
    &= \max_{\lambda\in\calB(\Eall)} \: \min_{\substack{\theta\in\mathcal{H} \\ f\in\calF}} \: \Lambda(f, \lambda) + \big[ R(\varphi(\theta, \cdot)) - R(f) \big] + \int_{\Eall} \big[ \calL^e(\varphi(\theta, \cdot)) - \calL^e(f)\big] \text{d}\lambda(e) \label{eq:add-and-subtract}
\end{align}
Now let $\mu(e)$ denote any probability measure with support over $\Eall$.  Consider the latter two terms in the above problem, and observe that we can write 
\begin{align}
    &\big[R(\varphi(\theta, \cdot)) - R(f) \big] + \int_{\Eall} \big[ \calL^e(\varphi(\theta, \cdot)) - \calL^e(f)\big] \text{d}\lambda(e) \\
    &\qquad = \int_{\Eall} \big[R(\varphi(\theta, \cdot)) - R(f) \big] \mu(e)\text{d}e + \int_{\Eall} \big[ \calL^e(\varphi(\theta, \cdot)) - \calL^e(f)\big] \lambda(e) \text{d}e \\
    &\qquad = \int_{\Eall} \begin{bmatrix}R(\varphi(\theta, \cdot)) - R(f)  \\ \calL^e(\varphi(\theta, \cdot)) - \calL^e(f) \end{bmatrix}^\top \begin{bmatrix} \mu(e) \\ \lambda(e)\end{bmatrix} \text{d}e \\
    &\qquad \overset{(*)}{\leq} \int_{\Eall} \norm{\begin{bmatrix} \mu(e) \\ \lambda(e) \end{bmatrix}}_{1} \cdot \norm{\begin{bmatrix}R(\varphi(\theta, \cdot)) - R(f)  \\ \calL^e(\varphi(\theta, \cdot)) - \calL^e(f) \end{bmatrix}}_{\infty} \text{d}e \\
    &\qquad = \int_{\Eall} \big(\mu(e) + \lambda(e)\big) \cdot \max\left\{R(\varphi(\theta, \cdot)) - R(f), \calL^e(\varphi(\theta, \cdot)) - \calL^e(f)\right\}\text{d}e \\
    &\qquad\overset{(**)}{\leq} \norm{\mu + \lambda}_{L^1} \cdot \norm{\max\left\{R(\varphi(\theta, \cdot)) - R(f), \calL^e(\varphi(\theta, \cdot)) - \calL^e(f)\right\}}_{L^\infty} \\
    &\qquad \overset{(\square)}{\leq} (1 + \norm{\lambda}_{L^1}) \cdot \norm{\max\left\{R(\varphi(\theta, \cdot)) - R(f), \calL^e(\varphi(\theta, \cdot)) - \calL^e(f)\right\}}_{L^\infty}. \label{eq:bound-plus-one}
\end{align}
where $(*)$ and $(**)$ follows from separate applications of H\"older's ineqaulity \cite{stein2011functional}, and $(\square)$ follows from an application of Minkowski's inequality and from the fact that $\mu$ is a (normalized) probability distribution.  Let us now consider the second term in the above product:
\begin{align}
    &\norm{\max\left\{R(\varphi(\theta, \cdot)) - R(f), \calL^e(\varphi(\theta, \cdot)) - \calL^e(f)\right\}}_{L^\infty} \\
    &\quad = \norm{\max\{ \mathbb{E}[\ell(\varphi(\theta, X), Y) - \ell(f(X),Y)], \mathbb{E}[d(\varphi(\theta, X), \varphi(\theta, G(X,e))) - d(f(X), f(G(X,e)))] \}}_{L^\infty} \\
    &\quad \overset{(\circ)}{\leq} \norm{ \mathbb{E}\left[ \max\{|\ell(\varphi(\theta, X), Y) - \ell(f(X),Y)|, |d(\varphi(\theta, X), \varphi(\theta, G(X,e))) - d(f(X), f(G(X,e)))| \right] }_{L^\infty} \\
    &\quad \overset{(\triangle)}{\leq} \E\norm{  \max\{|\ell(\varphi(\theta, X), Y) - \ell(f(X),Y)|, |d(\varphi(\theta, X), \varphi(\theta, G(X,e))) - d(f(X), f(G(X,e)))|  }_{L^\infty} \\
    &\quad \leq \mathbb{E} \left[ \max\{L_\ell \norm{\varphi(\theta, X) - f(X)}_\infty, L_d \norm{\varphi(\theta, X) - f(X) }_\infty\} \right] \\
    &\quad = \max\{L_\ell, L_d\} \cdot \E \norm{\varphi(\theta, X) - f(X) }_\infty. \label{eq:bound-max}
\end{align}
where $(\circ)$ and $(\triangle)$ both follow from Jensen's inequality, and the final inequality follows from our Lipschitzness assumptions on $\ell$ and $d$.  For simplicity, let $c = \max\{L_\ell, L_d\}$.  Now returning to~\eqref{eq:add-and-subtract}, we can combine~\eqref{eq:bound-plus-one} and~\eqref{eq:bound-max} to obtain
\begin{align}
    D^\star_\epsilon(\gamma) &\leq \max_{\lambda\in\calB(\Eall)} \min_{\substack{\theta\in\calH \\ f\in\calF}} \Lambda(f,\lambda) + c(1+\norm{\lambda}_{L^1}) \cdot \E \norm{\varphi(\theta, X) - f(X) }_\infty \\
    &= \max_{\lambda\in\calB(\Eall)} \min_{f\in\calF} \Lambda(f,\lambda) + c(1 + \norm{\lambda}_{L^1}) \cdot  \min_{\theta\in\calH} \E \norm{\varphi(\theta, X) - f(X) }_\infty \\
    &\leq \max_{\lambda\in\calB(\Eall)} \min_{f\in\calF} \Lambda(f,\lambda) + c\epsilon(1 + \norm{\lambda}_{L^1}). \label{eq:upper-bound-pert}
\end{align}
Now let $D^\star_\text{pert}(\gamma)$ denote the optimal value of the above problem; that is,
\begin{align}
    D^\star_\text{pert}(\gamma) &\triangleq \max_{\lambda\in\calB(\Eall)} \min_{f\in\calF} \Lambda(f,\lambda) + c\epsilon(1 + \norm{\lambda}_{L^1}) \\
    &\qquad = \max_{\lambda\in\calB(\Eall)} \min_{f\in\calF} R(f) + ce + \int_{\Eall} \left[\calL^e(f) - \gamma + c\epsilon\right]\text{d}\lambda(e) \label{eq:perturbed-problem}
\end{align}

\textbf{Step 3.}  In the final step, we prove the theorem.  We begin with the perhaps unintuitive fact that the perturbed problem defined above is the dual problem to a perturbed version of the optimization problem in~\eqref{eq:relax-mbdg}.  More specifically, the perturbed problem in~\eqref{eq:perturbed-problem} is the dual of
\begin{alignat}{2}
    P^\star_\text{pert}(\gamma) \triangleq &\minimize_{f\in\calF} &&R(f) + c\epsilon \\
    &\st &&\calL^e(f) \leq \gamma - c\epsilon \quad\forall e\in\Eall.
\end{alignat}
Note that as this primal perturbed optimization problem is convex since~\eqref{eq:relax-mbdg} is convex, and by assumption strong duality also holds for this perturbed problem.  Let $(f_\text{pert}^\star, \lambda_\text{pert}^\star)$ be primal-dual optimal for the perturbed problems we have defined above.  The following saddle-point relation is evident from the fact that strong duality holds:
\begin{align}
    \Lambda(f_\text{pert}^\star, \lambda') + c\epsilon\left(1 + \norm{\lambda'}_{L^1}\right) \leq D_\text{pert}^\star(\gamma) = P_\text{pert}^\star(\gamma) \leq \Lambda(f', \lambda_\text{pert}^\star) + c\epsilon\left(1 + \norm{\lambda^\star_\text{pert}}_{L^1}\right)
\end{align}
where the inequalities hold for all $f'\in\calF$ and for all $\lambda'\in\calB(\Eall)$.  Using this result for the choice of $f' = f^\star$, where we recall that $f^\star$ is defined in~\eqref{eq:def-of-primal-opt} as the primal optimal solution to~\eqref{eq:relax-mbdg}, it follows from~\eqref{eq:upper-bound-pert} that
\begin{align}
    D_\epsilon^\star(\gamma) \leq D^\star_\text{pert}(\gamma) \leq \Lambda(f^\star, \lambda_\text{pert}^\star) + c\epsilon\left(1 + \norm{\lambda_\text{pert}^\star}_{L^1} \right) \label{eq:almost-done}
\end{align}
Now, recalling the original saddle-point relation in~\eqref{eq:upper-bound-pert}, it holds that $\Lambda(f^\star, \lambda_\text{pert}^\star) \leq \Lambda(f^\star, \lambda^\star)$.  Using this fact along with~\eqref{eq:almost-done} yields the following result:
\begin{align}
    D_\epsilon^\star(\gamma) \leq \Lambda(f^\star, \lambda^\star) + c\epsilon\left(1 + \norm{\lambda_\text{pert}^\star}_{L^1} \right) = P^\star(\gamma) + c\epsilon\left(1 + \norm{\lambda_\text{pert}^\star}_{L^1} \right)
\end{align}
This completes the proof.
\end{proof}

\subsection{Characterizing the empirical gap (used in Theorem~\ref{thm:duality-gap})}

\begin{proposition}[Empirical gap]\label{prop:empir-gap}
Assume $\ell$ and $d$ are non-negative and bounded in $[-B,B]$ and let $\vcdim$ denote the VC-dimension of the hypothesis class $\mathcal{A}_\epsilon$.  Then it holds with probability $1-\delta$ over the $N$ samples from each domain that
\begin{align}
    |D_\epsilon^\star(\gamma) - D^\star_{\epsilon,N,\Etrain}(\gamma)| \leq 2B \sqrt{\frac{1}{N} \left[1 + \log\left(\frac{4(2N)^{\vcdim}}{\delta}\right)\right]} \label{eq:restate-empirical-gap}
\end{align}
\end{proposition}

\begin{proof}
In this proof, we use a similar approach as in \cite[Prop.\ 2]{chamon2020empirical} to derive the generalization bound.  Notably, we extend the ideas given in this proof to accommodate two problems with different constraints, wherein the constraints of one problem are a strict subset of the other problem.

To begin, let $(\theta_\epsilon^\star, \lambda^\star_\epsilon)$ and $(\theta_{\epsilon, N, \Etrain}^\star, \lambda_{\epsilon, N, \Etrain}^\star)$ be primal-dual optimal pairs for \eqref{eq:param-dual} and \eqref{eq:param-empir-dual} that achieve $D^\star_\epsilon(\gamma)$ and $D^\star_{\epsilon, N, \Etrain}(\gamma)$ respectively; that is,
\begin{align}
    (\theta_\epsilon^\star, \lambda^\star_\epsilon)\in \argmax_{\lambda\in\mathcal{P}(\Eall)} \: \min_{\theta\in\mathcal{H}} \: R(\varphi(\theta, \cdot)) + \int_{\Eall} \left[\mathcal{L}^e(\varphi(\theta, \cdot)) - \gamma\right] \text{d}\lambda(e). \label{eq:opt-param-dual}
\end{align}
and
\begin{align}
    (\theta_{\epsilon, N, \Etrain}^\star, \lambda_{\epsilon, N, \Etrain}^\star) \in \argmax_{\lambda(e)\geq 0, \: e\in\Etrain} \: \min_{\theta\in\mathcal{H}} \: \hat{R}(\varphi(\theta, \cdot)) + \frac{1}{|\Etrain|}\sum_{e\in\Etrain} \left[\hat{\mathcal{L}}^e(\varphi(\theta, \cdot)) - \gamma\right] \lambda(e) \label{eq:opt-param-dual-empir}
\end{align}
are satisfied.  Due to the optimality of these primal-dual pairs, both primal-dual pairs satisfy the KKT conditions \cite{boyd2004convex}.  In particular, the complementary slackness condition implies that
\begin{align}
    \int_{\Eall} \left[\mathcal{L}^e(\varphi(\theta^\star_\epsilon, \cdot)) - \gamma\right] \text{d}\lambda^\star_\epsilon(e) = 0 \label{eq:eall-dual-var-goes-away}
\end{align}
and that
\begin{align}
    \frac{1}{|\Etrain|}\sum_{e\in\Etrain} \left[\hat{\mathcal{L}}^e(\varphi(\theta^\star_{\epsilon, N, \Etrain}, \cdot)) - \gamma\right] \lambda^\star_{\epsilon, N, \Etrain}(e) = 0. \label{eq:etrain-dual-var-goes-away}
\end{align}
Thus, as \eqref{eq:eall-dual-var-goes-away} indicates that the second term in the objective of \eqref{eq:opt-param-dual} is zero, we can recharacterize the optimal value $D_\epsilon^\star(\gamma)$ via
\begin{align}
    D_\epsilon^\star(\gamma) = R(\varphi(\theta_\epsilon^\star, \cdot)) = \E_{\Prob(X,Y)} \ell(\varphi(\theta^\star_\epsilon, X), Y) \label{eq:elim-consts-eall}
\end{align}
and similarly from \eqref{eq:etrain-dual-var-goes-away}, can recharacterize the optimal value $D^\star_{\epsilon, N, \Etrain}(\gamma)$ as
\begin{align}
    D^\star_{\epsilon, N, \Etrain}(\gamma) = \hat{R}(\varphi(\theta_{\epsilon, N, \Etrain}^\star, \cdot)) = \frac{1}{N}\sum_{i=1}^N \ell(\varphi(\theta_{\epsilon, N, \Etrain}^\star, x_i), y_i). \label{eq:elim-consts-etr}
\end{align}
Ultimately, our goal is to bound the gap between $|D_\epsilon^\star(\gamma) - D^\star_{\epsilon, N, \Etrain}(\gamma)|$.  Combining \eqref{eq:elim-consts-eall} and \eqref{eq:elim-consts-etr}, we see that this gap can be characterized in the following way
\begin{align}
    |D_\epsilon^\star(\gamma) - D^\star_{\epsilon, N, \Etrain}(\gamma)| = |R(\varphi(\theta_\epsilon^\star, \cdot)) - \hat{R}(\varphi(\theta_{\epsilon, N, \Etrain}^\star, \cdot))|. \label{eq:gap-in-terms-of-R}
\end{align}
Now due to the optimality of the primal-optimal variables $\theta_\epsilon^\star$ and $\theta_{\epsilon, N, \Etrain}^\star$, observe that
\begin{align}
    &R(\varphi(\theta_\epsilon^\star, \cdot)) - \hat{R}(\varphi(\theta_\epsilon^\star, \cdot)) &\\
    &\qquad\qquad \leq R(\varphi(\theta_\epsilon^\star, \cdot)) - \hat{R}(\varphi(\theta_{\epsilon, N, \Etrain}^\star, \cdot)) \\
    &\qquad\qquad\qquad\qquad \leq R(\varphi(\theta_{\epsilon, N, \Etrain}^\star, \cdot)) - \hat{R}(\varphi(\theta_{\epsilon, N, \Etrain}^\star, \cdot))
\end{align}
which, when combined with \eqref{eq:gap-in-terms-of-R}, implies that
\begin{align}
    &|D_\epsilon^\star(\gamma) - D^\star_{\epsilon, N, \Etrain}(\gamma)| \\
    &\qquad \leq \max\left\{\left| R(\varphi(\theta_\epsilon^\star, \cdot)) - \hat{R}(\varphi(\theta_\epsilon^\star, \cdot)) \right|, \left| R(\varphi(\theta_{\epsilon, N, \Etrain}^\star, \cdot)) - \hat{R}(\varphi(\theta_{\epsilon, N, \Etrain}^\star, \cdot))\right| \right\}. \label{eq:pre-apply-vc}
\end{align}
To wrap up the proof, we simply leverage the classical VC-dimension bounds for both of the terms in \eqref{eq:pre-apply-vc}.  That is, following \cite{vapnik1999overview}, it holds for all $\theta$ that with probability $1-\delta$, 
\begin{align}
    |R(\varphi(\theta, \cdot)) - \hat{R}(\varphi(\theta), \cdot)| \leq 2B \sqrt{\frac{1}{N} \left[ 1 + \log\left(\frac{4(2N)^{\vcdim}}{\delta}\right)\right]}. \label{eq:vc-dim-bound}
\end{align}
As the bound in \eqref{eq:vc-dim-bound} holds $\forall\theta\in\calH$, in particular it holds for $\theta_\epsilon^\star$ and $\theta_{\epsilon, N, \Etrain}^\star$.  This directly implies the bound in \eqref{eq:restate-empirical-gap}.
\end{proof}

\subsection{Proof of Theorem \ref{thm:duality-gap}}
The 
\noindent\textbf{Theorem~\ref{thm:duality-gap}.}  Let $\epsilon > 0$ be given, and let $\varphi$ be an $\epsilon$-parameterization of $\mathcal{F}$. Let Assumption~\ref{assume:lipschitz} hold, and further assume that $\ell$ and $d$ are $[0,B]$-bounded and that $d(\Prob,\mathbb{T}) = 0$ if and only if $\Prob = \mathbb{T}$ almost surely, and that $P^\star(\gamma)$ is $L$-Lipschitz.  Then assuming that $\mathcal{A}_\epsilon$ has finite VC-dimension, it holds with probability $1-\delta$ over the $N$ samples from $\Prob$ that
\begin{align}
    |P^\star - D_{\epsilon,N,\Etrain}^\star(\gamma) | \leq L\gamma + (L_\ell + 2L_d)\epsilon + {\cal O}\left( \sqrt{\log(N)/N}\right)
\end{align}

\begin{proof}
The proof of this theorem is a simple consequence of the triangle inequality.  Indeed, by combining Remark \ref{rmk:gamma-remark}, Proposition \ref{prop:param-gap}, and Proposition \ref{prop:empir-gap}, we find that
\begin{align}
    &|P^\star - D_{\epsilon,N,\Etrain}^\star(\gamma) | \\
    &\qquad = |P^\star + P^\star(\gamma) - P^\star(\gamma) + D_\epsilon^\star(\gamma) - D^\star_\epsilon(\gamma) - D_{\epsilon,N,\Etrain}^\star(\gamma) | \\
    &\qquad \leq |P^\star - P^\star(\gamma)| + |P^\star(\gamma) - D^\star_\epsilon(\gamma)| + |D^\star(\gamma) - D_{\epsilon,N,\Etrain}^\star(\gamma) | \\
    &\qquad\leq L\gamma + \epsilon k \left(1 + \norm{\lambda_\text{pert}^\star}_{L^1} \right) + 2B \sqrt{\frac{1}{N} \left[1 + \log\left(\frac{4(2N)^{\vcdim}}{\delta}\right)\right]}.
\end{align}
This completes the proof.
\end{proof}

\subsection{Proof of Theorem \ref{thm:primal-dual}} \label{sect:primal-dual-conv}

\textbf{Theorem~\ref{thm:primal-dual}.}  Assume that $\ell$ and $d$ are $[0,B]$-bounded, convex, and $M$-Lipschitz continuous (i.e.\ $M = \max\{L_\ell, L_d\}$.  Further, assume that $\calH$ has finite VC-dimension $\vcdim$ and that for each $\theta_1, \theta_2 \in \calH$ and for each $\beta\in[0,1]$, there exists a parameter $\theta\in\calH$ and a constant $\nu>0$ such that
\begin{align}
    \E_{\Prob(X,Y)} \left| \beta \varphi(\theta_1, X) + (1-\beta)\varphi(\theta_2, X) - \varphi(\theta, X)\right| \leq \nu. \label{eq:func-class-regularity}
\end{align}
Finally, assume that there exists a parameter $\theta\in\calH$ such that $\varphi(\theta, \cdot)$ is strictly feasible for \eqref{eq:relax-mbdg}, i.e. that
\begin{align}
    \calL^e(\varphi(\theta, \cdot)) \leq \gamma - M\nu \quad\forall e\in\Eall
\end{align}
where $\nu$ is the constant from \eqref{eq:func-class-regularity}.  Then it follows that the primal-dual pair $(\theta^{(T)}, \lambda^{(T)})$ obtained after running the alternating primal-dual iteration in \eqref{eq:primal-step} and \eqref{eq:dual-step} for $T$ steps with step size $\eta$, where
\begin{align}
    T \triangleq \left\lceil \frac{\norm{\lambda^\star}}{2\eta M\nu } \right\rceil + 1 \qquad\text{and}\qquad \eta\leq \frac{2M\nu }{|\Etrain|B^2}
\end{align}
satisfies 
\begin{align}
    |P^\star - \hat{\Lambda}(\theta^{(T)}, \mu^{(T)})| \leq \rho + M\nu + L\gamma + \mathcal{O}(\sqrt{\log(N)/N})
\end{align}
where $\norm{\lambda^
\star}$ is the optimal dual variable for \eqref{eq:param-dual}.

\begin{proof}
Observe that by the triangle inequality, we have
\begin{align}
    |P^\star - \hat{\Lambda}(\theta^{(T)}, \mu^{(T)})| &= |P^\star - P^\star(\gamma) + P^\star(\gamma) - \hat{\Lambda}(\theta^{(T)}, \mu^{(T)})| \\
    &\leq |P^\star - P^\star(\gamma)| + |P^\star(\gamma) - \hat{\Lambda}(\theta^{(T)}, \mu^{(T)})| \\
    &\leq L\gamma + |P^\star(\gamma) - \hat{\Lambda}(\theta^{(T)}, \mu^{(T)})| \label{eq:apply-rmk}
\end{align}
where the last step follows from Remark \ref{rmk:gamma-remark}.  Then, from \cite[Theorem 2]{chamon2021constrained}, it directly follows that
\begin{align}
    |P^\star(\gamma) - \hat{\Lambda}(\theta^{(T)}, \mu^{(T)})| \leq \rho + M\nu + \mathcal{O}\sqrt{\log(N)/N}.
\end{align}
Combining this with \eqref{eq:apply-rmk} completes the proof.
\end{proof}

\newpage
\section{Algorithmic variants for MBDG}

In Section \ref{sect:experiments}, we considered several algorithmic variants of MBDG.  Each variant offers a natural point of comparison to the MBDG algorithm, and for completeness, in this section we fully characterize these variants.

\subsection{Data augmentation} \label{sect:data-aug-algs}

In Section \ref{sect:experiments}, we did an ablation study concerning various data-augmentation alternatives to MBDG.  In particular, in the experiments performed on \texttt{ColoredMNIST}, we compared results obtained with MBDG to two algorithms we called MBDA and MBDG-DA.  For clarity, in what follows we describe each of them in more detail. 

\begin{algorithm}[t]
   \caption{ERM with model-based data augmentation (MBDA)}
   \label{alg:mbda}
\begin{algorithmic}[1]
   \State {\bfseries Hyperparameters:} Step size $\eta > 0$
   \Repeat
   \For{minibatch $\{(x_j, y_j)\}_{j=1}^m$ in training dataset}
   \State $\tilde{x}_j \gets \Call{GenerateImage}{x_j}$ $\forall j\in[m]$ \Comment{Generate model-based images}
   \State $\text{loss}(\theta) \gets (1/m) \sum_{j=1}^m  [\ell\left(x_j, y_j; \varphi(\theta, \cdot)\right) + \ell(\tilde{x}_j, y_j; \varphi(\theta, \cdot))]$
   \State $\theta \gets \theta - \eta \nabla_\theta \text{loss}(\theta)$
   \EndFor
   \Until{convergence}
\end{algorithmic}
\end{algorithm}

\paragraph{MBDA.}  In the MDBA variant, we train using ERM with data augmentation through the learned domain transformation model $G(x,e)$.  This procedure is summarized in Algorithm \ref{alg:mbda}.  Notice that in this algorithm, we do not consider the constraints engendered by the assumption of $G$-invariance.  Rather, we simply seek to use follow the recent empirical evidence that suggests that ERM with proper tuning and data augmentation yields state-of-the-art performance in domain generalization~\cite{gulrajani2020search}.  Note that in Table \ref{tab:cmnist}, the MBDA algorithm improves significantly over the baselines, but that it lags more than 20 percentage points behind results obtained using MBDG.  This highlights the utility of enforcing constraints rather than performing data augmentation on the training objective.  

\begin{algorithm}[t]
   \caption{MBDG with data augmentation (MBDG-DA)}
   \label{alg:mbdg-da}
\begin{algorithmic}[1]
   \State {\bfseries Hyperparameters:} Primal step size $\eta_p > 0$,  dual step size $\eta_d \geq 0$, margin $\gamma > 0$
   \Repeat
   \For{minibatch $\{(x_j, y_j)\}_{j=1}^m$ in training dataset}
   \State $\tilde{x}_j \gets \Call{GenerateImage}{x_j}$ $\forall j\in[m]$ \Comment{Generate images for constraints}
   \State $\bar{x}_j \gets \Call{GenerateImage}{x_j}$ $\forall j\in[m]$ \Comment{Generate images for objective}
   \State $\text{loss}(\theta) \gets (1/m) \sum_{j=1}^m  [\ell\left(x_j, y_j; \varphi(\theta, \cdot)\right) + \ell(\bar{x}_j, y_j; \varphi(\theta, \cdot)) + \ell(\tilde{x}_j, y_j; \varphi(\theta, \cdot))]$
   \State \text{distReg}$(\theta) \gets (1/m)\sum_{j=1}^m d(\varphi(\theta, x_j), \varphi(\theta, \tilde{x}_j))$ 
   \State $\theta \gets \theta - \eta_p \nabla_\theta [ \: \text{loss}(\theta) + \lambda \cdot  \text{distReg}(\theta) \: ]$
   \State $\lambda \gets \left[ \lambda + \eta_d \left( \text{distReg}(\theta) - \gamma \right)\right]_+$ 
   \EndFor
   \Until{convergence}
\end{algorithmic}
\end{algorithm}

\paragraph{MBDG-DA.}  In the MBDG-DA variant, we follow a similar procedure to the MBDG algorithm.  The only modification is that we perform data augmentation through the learned model $G(x,e)$ on the training objective in addition to enforcing the $G$-invariance constraints.  This procedure is summarized in Algorithm \ref{alg:mbdg-da}.  As shown in Table \ref{tab:cmnist}, this procedure performs rather well on \texttt{ColoredMNIST}, beating all baselines by nearly 20 percentage points.  However, this algorithm still does not reach the performance level of MBDG when the -90\% domain is taken to be the test domain.  

\subsection{Regularization}\label{sect:reg-variant}

\begin{algorithm}[t]
   \caption{Regularized MBDG (MBDG-Reg)}
   \label{alg:mbdg-reg}
\begin{algorithmic}[1]
   \State {\bfseries Hyperparameters:} Step size $\eta > 0$, weight $w > 0$
   \Repeat
   \For{minibatch $\{(x_j, y_j)\}_{j=1}^m$ in training dataset}
   \State $\tilde{x}_j \gets \Call{GenerateImage}{x_j}$ $\forall j\in[m]$ \Comment{Generate model-based images}
   \State $\text{loss}(\theta) \gets (1/m) \sum_{j=1}^m  [\ell\left(x_j, y_j; \varphi(\theta, \cdot)\right) + \ell(\tilde{x}_j, y_j; \varphi(\theta, \cdot))]$
   \State \text{distReg}$(\theta) \gets (1/m)\sum_{j=1}^m d(\varphi(\theta, x_j), \varphi(\theta, \tilde{x}_j))$ 
   \State $\theta \gets \theta - \eta \nabla_\theta [\text{loss}(\theta) + w\cdot\text{distReg}(\theta)]$
   \EndFor
   \Until{convergence}
\end{algorithmic}
\end{algorithm}

In Section \ref{sect:experiments}, we also compared the performance of MBDG to a regularized version of MBDG.  In this regularized version, we sought to solve \eqref{eq:regularized-mbdg} using the algorithm described in Algorithm \ref{alg:mbdg-reg}.  In particular, in this algorithm we fix the weight $w>0$ as a hyperparameter, and we perform SGD on the regularized loss function $\text{loss}(\theta) + w\cdot \text{distReg}(\theta)$.  Note that while this method performs well in practice (see Table \ref{tab:cmnist}), it is generally not possible to provide generalization guarantees for the regularized version of the problem.
\newpage
\section{Additional experiments and experimental details}  \label{sect:further-exps}

\begin{table}
    \centering
    \caption{\textbf{DomainBed hyperparameters for MBDG and its variants.}  We record the additional hyperparameters and their selection criteria for MBDG and its variants.  Each of these hyperparameters was selected via randomly in the ranges defined in the third column in the DomainBed package.}
    \label{tab:domainbed-hparams}
    \begin{tabular}{cccc} 
    \toprule
         \textbf{Algorithm} & \textbf{Hyperparameter} & \textbf{Randomness} & \textbf{Default}  \\
    \midrule
         \multirow{2}{*}{MBDG} & Dual step size $\eta_d$ & $\text{Unif}(0.001, 0.1)$ & 0.05 \\
         & Constraint margin $\gamma$ & $\text{Unif}(0.0001, 0.01)$ & 0.025 \\
         \midrule 
         \multirow{2}{*}{MBDG-DA} & Dual step size $\eta_d$ & $\text{Unif}(0.001, 0.1)$ & 0.05 \\
         & Constraint margin $\gamma$ & $\text{Unif}(0.0001, 0.01)$ & 0.025 \\
         \midrule
         MBDG-Reg & Weight $w$ & \text{Unif}(0.5, 10.0) & 1.0 \\ 
    \bottomrule
    \end{tabular}
    
\end{table}

In this appendix, we record further experimental details beyond the results presented in Section \ref{sect:experiments}.  The experiments performed on \texttt{ColoredMNIST}, \texttt{PACS}, and \texttt{VLCS} were all performed using the DomainBed package.  All of the default hyperparameters (e.g.\ learning rate, weight decay, etc.) were left unchanged from the standard DomainBed implementation. In Table \ref{tab:domainbed-hparams}, we record the additional hyperparameters used for MBDG and its variants as well as the random criteria by which hyperparameters were generated.  For each of these DomainBed datasets, model-selection was performed via hold-one-out cross-validation, and the baseline accuracies were taken from commit 7df6f06 of the DomainBed repository.  The experiments on the \texttt{WILDS} datasets used the hyperparameters recorded by the authors of~\cite{koh2020wilds}; these hyperparameters are recorded in Sections~\ref{sec:camelyon-appendix} and~\ref{sec:fmow-appendix}.  Throughout the experiments, we use the KL-divergence as the distance metric $d$.

% \subsection{ColoredMNIST}

% In Section \ref{sect:experiments}, we achieved state-of-the-art performance on the \texttt{ColoredMNIST} dataset.

\subsection{Camelyon17-WILDS} \label{sec:camelyon-appendix}

For the \texttt{Camelyon17-WILDS} dataset, we used the out-of-distribution validation set provided in the \texttt{Camelyon17-WILDS} dataset to tune the hyperparameters for each classifier.  This validation set contains images from a hospital that is not represented in any of the training domains or the test domain.  Following \cite{koh2020wilds}, we used the DenseNet-121 architecture \cite{huang2017densely} and the Adam optimizer \cite{kingma2014adam} with a batch size of 200.  We also used the same hyperparameter sweep as was described in Appendix B.4 of \cite{koh2020wilds}.  In particular, when training using our algorithm, we used the the following grid for the (primal) learning rate: $\eta_p \in \{0.01, 0.001, 0.0001\}$.  Because we use the same hyperparameter sweep, architecture, and optimizer, we report the classification accuracies recorded in Table 9 of \cite{koh2020wilds} to provide a fair comparison to past work.  After selecting the hyperparameters based on the accuracy on the validation set, we trained classifiers using MBDG for 10 independent runs and reported the average accuracy and standard deviation across these trials in Table~\ref{tab:wilds}. 

In Section \ref{sect:experiments}, we performed an ablation study on \texttt{Camelyon17-WILDS} wherein the model $G$ was replaced by standard data augmentation transforms.  For completeness, we describe each of the methods used in this plot below.  For each method, invariance was enforced between a clean images drawn from the training domains and corresponding data that was varied according to a particular fixed transformation.

\augimage{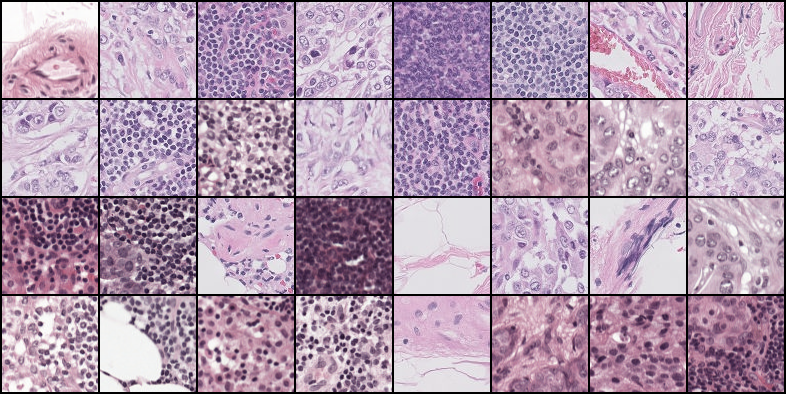}{\textbf{Samples before and after CJ transformations.}}{aug-color}

\paragraph{CJ (Color Jitter).}  The PIL color transformation\footnote{\url{https://pillow.readthedocs.io/en/stable/reference/ImageEnhance.html\#PIL.ImageEnhance.Color}}.  See Figure \ref{fig:aug-color} for samples.

\augimage{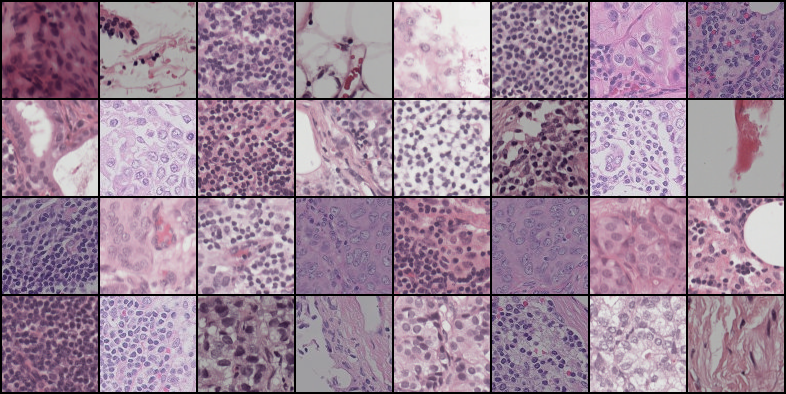}{\textbf{Samples before and after B+C transformations.}}{aug-bc}

\paragraph{B+C (Brightness and contrast).}  PIL \texttt{Brightness}\footnote{\url{https://pillow.readthedocs.io/en/stable/reference/ImageEnhance.html\#PIL.ImageEnhance.Brightness}} and \texttt{Contrast}\footnote{\url{https://pillow.readthedocs.io/en/stable/reference/ImageEnhance.html\#PIL.ImageEnhance.Contrast}} transformations.  See Figure \ref{fig:aug-bc} for samples.

\augimage{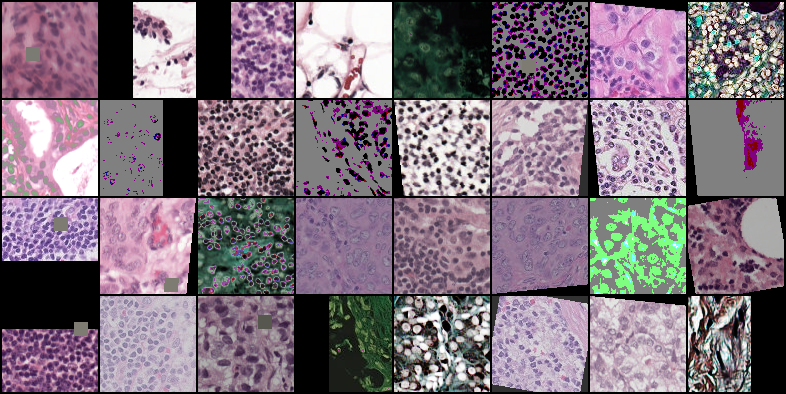}{\textbf{Samples before and after RandAugment transformations.}}{rand-aug}

\paragraph{RA (RandAugment).}  We use the data augmentation technique RandAugment \cite{cubuk2020randaugment}, which randomly samples random transformations to be applied at training time.  In particular, the following transformations are randomly sampled: \texttt{AutoContrast}, \texttt{Equalize}, \texttt{Invert}, \texttt{Rotate}, \texttt{Posterize}, \texttt{Solarize}, \texttt{SolarizeAdd}, \texttt{Color}, \texttt{Constrast}, \texttt{Brightness}, \texttt{Sharpness}, \texttt{ShearX}, \texttt{ShearY}, \texttt{CutoutAbs}, \texttt{TranslateXabs}, and \texttt{TranslateYabs}.  We used an open-source implementation of RandAugment for this experiment\footnote{\url{https://github.com/ildoonet/pytorch-randaugment}}.  See Figure \ref{fig:rand-aug} for samples.

\augimage{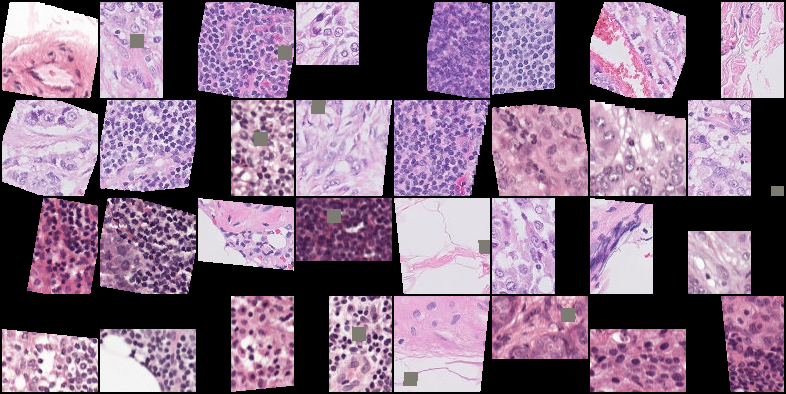}{\textbf{Samples before and after RA-Geom transformations.}}{aug-geom}

\paragraph{RA-Geom (RandAugment with geometric transformations).}  We use the RandAugment scheme with a subset of the transformations mentioned in the previous paragraph.  In particular, we use the following geometric transformations: \texttt{Rotate}, \texttt{ShearX}, \texttt{ShearY}, \texttt{CutoutAbs}, \texttt{TranslateXabs}, and \texttt{TranslateYabs}.  See Figure \ref{fig:aug-geom} for samples.

\augimage{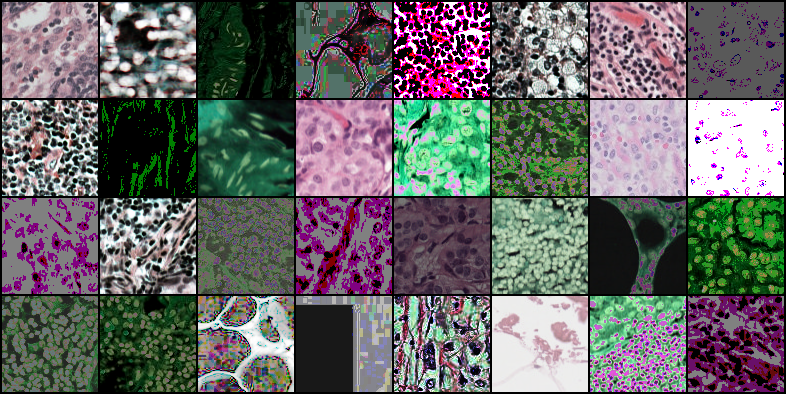}{\textbf{Samples before and after RA-Color transformations.}}{ra-color}

\paragraph{RA-Color (RandAugment with color-based transformations).}  We use the RandAugment scheme with a subset of transformations mentioned in the RandAugment paragraph.  In particular, we use the following color-based transformations: \texttt{AutoContrast}, \texttt{Equalize}, \texttt{Invert},
\texttt{Posterize},
\texttt{Solarize}, \texttt{SolarizeAdd}, \texttt{Color}, \texttt{Constrast}, \texttt{Brightness}, \texttt{Sharpness}.  See Figure \ref{fig:ra-color} for samples.

\augimage{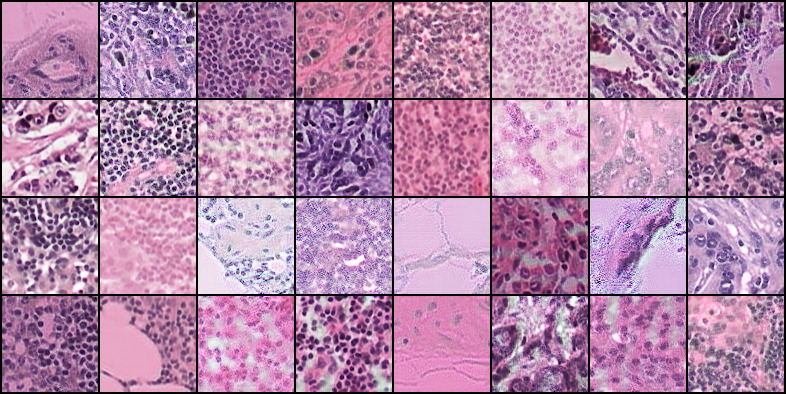}{\textbf{Samples before and after (learned) MUNIT transformations.}}{aug-munit}

\paragraph{MUNIT.}  We use an MUNIT model trained on the images from the training datasets; this is the procedure advocated for in the main text, i.e.\ in the \Call{GenerateImage}{x} procedure.  See Figure~\ref{fig:aug-munit} for samples.

\subsection{FMoW-WILDS} \label{sec:fmow-appendix}

As with the \texttt{Camelyon17-WILDS} dataset, to facilitate a fair comparison, we again use the out-of-distribution validation set provided in \cite{koh2020wilds}.  While the authors report the architecture, optimizer, and final hyperparameter choices used for the \texttt{FMoW-WILDS} dataset, they not report the grid used for hyperparameter search.  For this reason, we rerun all baselines along with our algorithm over a grid of hyperparameters using the same architecture and optimizer as in \cite{koh2020wilds}.  In particular, we follow \cite{koh2020wilds} by training a DenseNet-121 with the Adam optimizer with a batch size of 64.  We selected the (primal) learning rate from $\eta_p \in \{0.05, 0.01, 0.005, 0.001\}$.  We selected the trade-off parameter $\lambda_{\text{IRM}}$ for IRM from the grid $\lambda_{\text{IRM}}\in\{0.1, 0.5, 1.0, 10.0\}$.  As before, the results in Table \ref{tab:wilds} list the average accuracy and standard deviation over ten independent runs attained by our algorithm as well as ERM, IRM, and ARM.

\subsection{VLCS}

In Table \ref{tab:vlcs}, we provide a full set of results for the \texttt{VLCS} dataset.  As shown in this Table, MBDG offers competitive performance to other state-of-the-art method.  Indeed, MBDG achieves the best results on the ``LabelMe'' (L) subset by nearly two percentage points.  

\begin{table}[h]
\centering
\caption{\textbf{Full results for \texttt{VLCS}.}  In this table, we present results for all baselines on the \texttt{VLCS} dataset.}
\label{tab:vlcs}
\adjustbox{max width=\textwidth}{%
\begin{tabular}{lccccc}
\toprule
\textbf{Algorithm}   & \textbf{C}           & \textbf{L}           & \textbf{S}           & \textbf{V}           & \textbf{Avg}         \\
\midrule
ERM                  & 98.0 $\pm$ 0.4       & 62.6 $\pm$ 0.9       & 70.8 $\pm$ 1.9       & 77.5 $\pm$ 1.9       & 77.2                 \\
IRM                  & \textbf{98.6 $\pm$ 0.3}       & 66.0 $\pm$ 1.1       & 69.3 $\pm$ 0.9       & 71.5 $\pm$ 1.9       & 76.3                 \\
GroupDRO             & 98.1 $\pm$ 0.3       & 66.4 $\pm$ 0.9       & 71.0 $\pm$ 0.3       & 76.1 $\pm$ 1.4       & 77.9                 \\
Mixup                & 98.4 $\pm$ 0.3       & 63.4 $\pm$ 0.7       & 72.9 $\pm$ 0.8       & 76.1 $\pm$ 1.2       & 77.7                 \\
MLDG                 & 98.5 $\pm$ 0.3       & 61.7 $\pm$ 1.2       & \textbf{73.6 $\pm$ 1.8}       & 75.0 $\pm$ 0.8       & 77.2                 \\
CORAL                & 96.9 $\pm$ 0.9       & 65.7 $\pm$ 1.2       & 73.3 $\pm$ 0.7       & \textbf{78.7 $\pm$ 0.8}       & \textbf{78.7 }                \\
MMD                  & 98.3 $\pm$ 0.1       & 65.6 $\pm$ 0.7       & 69.7 $\pm$ 1.0       & 75.7 $\pm$ 0.9       & 77.3                 \\
DANN                 & 97.3 $\pm$ 1.3       & 63.7 $\pm$ 1.3       & 72.6 $\pm$ 1.4       & 74.2 $\pm$ 1.7       & 76.9                 \\
CDANN                & 97.6 $\pm$ 0.6       & 63.4 $\pm$ 0.8       & 70.5 $\pm$ 1.4       & 78.6 $\pm$ 0.5       & 77.5                 \\
MTL                  & 97.6 $\pm$ 0.6       & 60.6 $\pm$ 1.3       & 71.0 $\pm$ 1.2       & 77.2 $\pm$ 0.7       & 76.6                 \\
SagNet               & 97.3 $\pm$ 0.4       & 61.6 $\pm$ 0.8       & 73.4 $\pm$ 1.9       & 77.6 $\pm$ 0.4       & 77.5                 \\
ARM                  & 97.2 $\pm$ 0.5       & 62.7 $\pm$ 1.5       & 70.6 $\pm$ 0.6       & 75.8 $\pm$ 0.9       & 76.6                 \\
VREx                 & 96.9 $\pm$ 0.3       & 64.8 $\pm$ 2.0       & 69.7 $\pm$ 1.8       & 75.5 $\pm$ 1.7       & 76.7                 \\
RSC                  & 97.5 $\pm$ 0.6       & 63.1 $\pm$ 1.2       & 73.0 $\pm$ 1.3       & 76.2 $\pm$ 0.5       & 77.5                 \\
\midrule
MBDG                  & 98.3 $\pm$ 1.2       & \textbf{68.1 $\pm$ 0.5}       & 68.8 $\pm$ 1.1       & 76.3 $\pm$ 1.3       &     77.9            \\
\bottomrule
\end{tabular}}
\end{table}

\newpage
\section{Further discussion of domain transformation models}\label{sect:dtms}

In some applications, domain transformation models in the spirit of Assumption \ref{assume:gen-model} are known a priori.  To illustrate this, consider the classic domain generalization task in which the domains correspond to different fixed rotations of the data \cite{ghifary2015domain,ilse2020diva}.  In this setting, the underlying generative model is given by 
\begin{align}
    G(x,e) := R(e)x \quad\text{for } e\in[0, 2\pi)
\end{align}
where $R(e)$ is a one-dimensional rotation matrix parameterized by an angle $e$.  In this way, each angle $e$ is identified with a different domain in $\Eall$.  However, unlike in this simple example, for the vast majority of settings encountered in practice,  the underlying domain transformation model is not known a priori and cannot be represented by concise mathematical expressions.  For example, obtaining a closed-form expression for a generative model that captures the variation in coloration, brightness, and contrast in the \texttt{Camelyon17-WILDS} cancer cell dataset shown in Figure \ref{fig:domain-gen} would be very challenging.  

In this appendix, we provide an extensive discussion concerning the means by which we used unlabeled data to learn domain transformation models using instances drawn from the training domains $\Etrain$.  In particular, we argue that it is not necessary to have access to the true underlying domain transformation model $G$ to achieve state-of-the-art results in domain generalization.  We then give further details concerning how we used the MUNIT architecture to train domain transformation models for \texttt{ColoredMNIST}, \texttt{Camelyon17-WILDS}, \texttt{FMoW-WILDS}, \texttt{PACS}, and \texttt{VLCS}.  Finally, we show further samples from these learned domain transformation models to demonstrate that high-quality samples can be obtained on this diverse array of datasets. 

\subsection{Is it necessary to learn a perfect domain transformation model?}

We emphasize that while our theoretical results rely on having access to the underlying domain transformation model, our algorithm and empirical results do not rely on having access to the true $G$.  Indeed, although we did not have access to the true model in any of the experiments in Section \ref{sect:experiments}, our empirical results show that we were able to achieve state-of-the-art results on several datasets.  

\subsection{Learning domain transformation models with MUNIT}

In practice, to learn a domain transformation model, a number of methods from the deep generative modeling literature have been recently been proposed \cite{huang2018multimodal,anoosheh2018combogan,choi2020stargan}.  In particular, throughout the remainder of this paper we will use the MUNIT architecture introduced in \cite{huang2018multimodal} to parameterize learned domain transformation models.  This architecture comprises two GANs and two autoencoding networks.  In particular, the MUNIT architecture -- along with many related works in the image-to-image translation literature -- was designed to map images between two datasets $A$ and $B$.  In this paper, rather than separating data we simply use $\calD_X$ for both $A$ and $B$, meaning that we train MUNIT to map the training data back to itself.  In this way, since $\calD_X$ contains data from different domains $e\in\Etrain$, the architecture is exposed to different environments during training, and thus seeks to map data between domains.

\subsection{On the utility of multi-modal image-to-image translation networks.}

In this paper, we chose the MUNIT framework because it is designed to learn a multimodal transformation that maps an image $x$ to a family of images with different levels of variation.  Unlike methods that seek deterministic mappings, e.g.\ CycleGAN and its variants~\cite{zhu2017unpaired}, this method will learn to generate diverse images, which allows us to more effectively enforce invariance over a wider class of images.   In Figures \ref{fig:multi-camelyon}, \ref{fig:multi-fmow}, and \ref{fig:multi-pacs}, we plot samples generated by sampling different style codes $e\sim\mathcal{N}(0, I)$ for MUNIT.  Note that while the results for \texttt{Camelyon17-WILDS} and \texttt{FMoW-WILDS} are sampled using the model $G(x,e)$, the samples from \texttt{PACS} are all sampled from \emph{different} models.

\multiimage{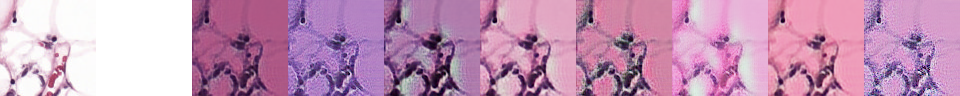}{\textbf{Multimodal \texttt{Camelyon17-WILDS} samples.}  Images from \texttt{Camelyon17-WILDS} (left) and images generated by sampling different style codes $e\sim\mathcal{N}(0, I)$ (right).}{multi-camelyon}

\multiimage{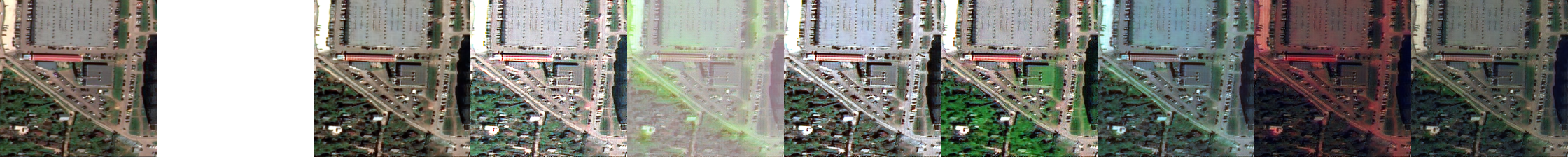}{\textbf{Multimodal \texttt{FMoW-WILDS} samples.}  Images from \texttt{FMoW-WILDS} (left) and images generated by sampling different style codes $e\sim\mathcal{N}(0, I)$ (right).}{multi-fmow}

\multiimage{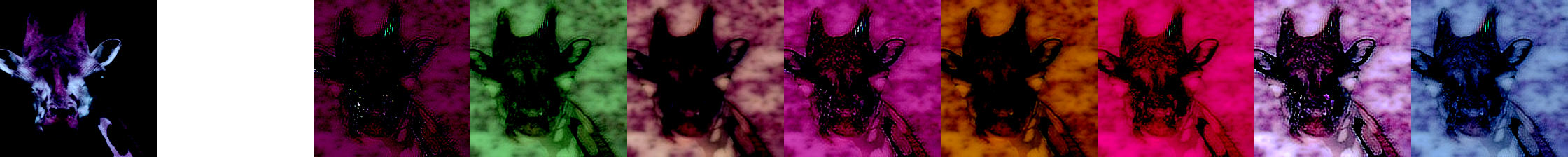}{\textbf{Multimodal \texttt{PACS} samples.}  Images from \texttt{PACS} (left) and images generated by sampling different style codes $e\sim\mathcal{N}(0, I)$ (right).}{multi-pacs}

\end{document}